\documentclass{article}

  \usepackage[main, final]{neurips_2026}

\usepackage[utf8]{inputenc} %
\usepackage[T1]{fontenc}    %
\usepackage[colorlinks,
            linkcolor=red,
            anchorcolor=blue,
            citecolor=blue, 
            pagebackref=true
           ]{hyperref}
\usepackage{url}            %
\usepackage{booktabs}       %
\usepackage{amsfonts}       %
\usepackage{nicefrac}       %
\usepackage{microtype}      %
\usepackage{xcolor}         %

\usepackage{mmll}
\usepackage{enumitem}
\usepackage{wrapfig}
\usepackage{subcaption}
\usepackage[table,xcdraw]{xcolor}

\definecolor{second}{RGB}{135,206,250}   %
\colorlet{best}{green!80!red}

\def \citet {\cite}
\def \citep {\cite}
\def\algname{LoDADA }

\usepackage{tcolorbox}
\tcbuselibrary{listings, skins}

\lstdefinestyle{xmlstyle}{
  language=HTML,
  basicstyle=\ttfamily\small,
  keywordstyle=\color{blue!70!black},
  stringstyle=\color{red!70!black},
  commentstyle=\color{gray!60},
  showstringspaces=false
}

\allowdisplaybreaks
\title{Localized Dynamics-Aware Domain Adaptation for Off-Dynamics Offline Reinforcement Learning}

\author{%
  Zhangjie Xia$^{1}$\thanks{Equal Contribution},
  \quad
  Yu Yang$^{2}$\footnotemark[1],
  \quad
  Pan Xu$^{2}$ \\
  $^{1}$New York University \quad
  $^{2}$Duke University \\
  \texttt{zx1357@nyu.edu, \{yu.yang, pan.xu\}@duke.edu}
}

\begin{document}

\maketitle

\begin{abstract}
Off-dynamics offline reinforcement learning (RL) aims to learn a policy for a target domain using limited target data and abundant source data collected under different transition dynamics. Existing methods typically address dynamics mismatch either globally over the state space or via pointwise data filtering; these approaches can miss localized cross-domain similarities or incur high computational cost. We propose Localized Dynamics-Aware Domain Adaptation (LoDADA), which exploits localized dynamics mismatch to better reuse source data. LoDADA clusters transitions from source and target datasets and estimates cluster-level dynamics discrepancy via domain discrimination. Source transitions from clusters with small discrepancy are more likely to be retained compared with those from clusters with large discrepancy. This yields a fine-grained and scalable data selection strategy that avoids overly coarse global assumptions and expensive per-sample filtering. We provide theoretical insights and extensive experiments across environments with diverse global and local dynamics shifts. Results show that LoDADA consistently outperforms state-of-the-art off-dynamics offline RL methods by better leveraging localized distribution mismatch. Our code is available at \url{https://github.com/panxulab/LoDADA}.

\end{abstract}

\begin{figure}[h]
    \centering
    \includegraphics[width=0.9\linewidth]{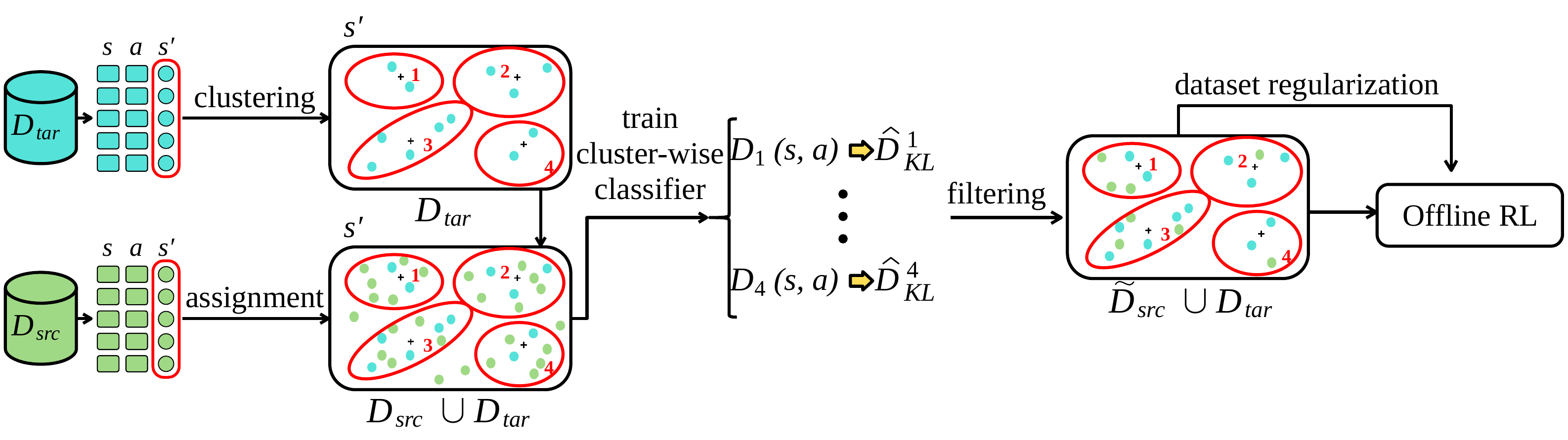}
    \caption{\textbf{An overview of our proposed framework.} We first perform K-means clustering on the next state for the target dataset and then assign source transitions to clusters via their next state. We estimate the local KL divergence between source and target dynamics in each cluster using cluster-wise classifiers, and selectively retain source transitions for downstream offline RL algorithms. We further introduce a dataset regularization term to ensure policy consistency with the target dataset.
    }
    \label{fig:graph} 
\end{figure}

\section{Introduction}
The goal of reinforcement learning (RL) \citep{kaelbling1996reinforcement,sutton1998reinforcement} is to learn a policy for optimal decision making by interacting with an environment and receiving reward feedback. However, direct interaction with the environment (the \textit{target domain}) can be dangerous or time-consuming in real-world applications such as autonomous driving \citep{kiran2021deep} and healthcare \citep{coronato2020reinforcement}. As a result, practitioners often collect experience in a structurally similar \textit{source domain} and then transfer the learned policy. A key challenge is that a policy trained in the source domain may perform poorly when deployed in the target domain due to mismatched transition dynamics. This issue is usually referred to as \textit{off-dynamics RL} \citep{eysenbach2021offdynamicsreinforcementlearningtraining,he2025samplecomplexitydistributionallyrobust}. In this paper, we study a practical setting, \textit{off-dynamics offline RL} \citep{liu2022dara,wang2026return}, where the agent cannot interact with the environment online and must instead learn from a limited target dataset together with a larger source dataset collected under different dynamics.

There has been growing interest in off-dynamics offline RL, with prior work that can be broadly categorized into three types: (i) data augmentation methods that modify rewards in the source data using dynamics-aware regularization \citep{liu2022dara, lyu2024crossdomainpolicyadaptationcapturing,wang2026return,guo2024off}; (ii) data filtering approaches that select source transitions based on their similarity to target transitions \citep{xu2023crossdomainpolicyadaptationvalueguided, wen2024contrastive,lyu2025cross}; and more recently (iii) data generation methods that use source data to learn target dynamics and generate synthetic target samples \citep{kong2025composite,guo2026mobody,yang2026cross}. However, most methods treat source-target dynamics mismatch globally over the state-action space, which can be overly conservative in regions where the domains align and overly optimistic where they differ substantially. For example, methods such as \citet{liu2022dara} train a single domain classifier over $(s,a,s')$ tuples across the full dataset, which is challenging with limited target data. A notable exception is pointwise data filtering via optimal transport \citep{lyu2025cross}, but its computational cost scales with dataset size, limiting its practicality \cite{dvurechensky2018computational}. Therefore, existing approaches can be data-inefficient and often fail to exploit fine-grained local similarities that would enable more robust and efficient transfer from the source to the target domain.

Our proposed method,  \textbf{Lo}calized \textbf{D}ynamics-\textbf{A}ware \textbf{D}omain \textbf{A}daptation (\textbf{LoDADA}), addresses these challenges in a novel way. Unlike prior approaches that rely on global or pointwise assumptions, \algname clusters source and target next states via K-means to identify localized regions in the state space where the two datasets are well aligned. We discard source samples that are far from cluster centroids, retaining those that better match the target data. Within each cluster, we train a classifier to distinguish source from target samples and use its outputs to estimate a local KL divergence between source and target transition dynamics. Clusters with large estimated KL divergence, indicating inconsistency between source and target transitions, are treated as less reliable. Instead of discarding entire clusters, \algname retains samples from all clusters to preserve state-space coverage, while prioritizing transitions from clusters with smaller estimated KL divergence. During the policy optimization stage, we further introduce a regularization term that penalizes deviations between the learned policy and a behavior policy which is fit to the target dataset, encouraging the learned policy to remain close to the target-domain's behavior.

By selectively retaining source transitions based on cluster-level KL estimates, \algname reduces the bias induced by dynamics mismatch and enables more effective use of source data. This framework provides a principled way to balance selective filtering and cross-domain adaptation under localized dynamics shifts. After the filtering step, we optimize the policy using Implicit Q-Learning (IQL) \citep{kostrikov2021offline} as a backbone, augmented with a regularization term that encourages the learned policy to stay close to the target behavior policy. Our main contributions are summarized as follows:
\begin{itemize}[nosep,leftmargin=*]
    \item We identify a key limitation of existing off-dynamics offline RL methods: they address dynamics mismatch either globally or pointwise, which can lead to data inefficiency and degraded performance. To address this, we propose \algname, which performs cluster-wise filtering by estimating local source-target dynamics discrepancy and retaining the most target-consistent source transitions.
    \item We provide a theoretical analysis that establishes return lower bounds for the policy learned by \algname, offering performance guarantees in the target domain under dynamics mismatch. 
    \item We empirically evaluate our method in four MuJoCo environments in the offline setting with different types and levels of dynamics shifts and demonstrate the effectiveness of our method with the $15.5\%$ aggregate improvement over the strongest baselines in gravity and friction shifts, $32.7\%$ on morphology shift and $29.3\%$ on local perturbation. We further show the superior performance of \algname on challenging navigation tasks in AntMaze and manipulation tasks in Adroit.
\end{itemize}

\section{Related Work}
\textbf{Domain adaptation in RL.}
Domain adaptation in RL focuses on adapting policies across domains with differing transition dynamics. Prior work include policy and value function transfer which reuse knowledge from a source domain to accelerate learning in a target domain \citep{abel2018policy,xu2023crossdomainpolicyadaptationvalueguided}, and robust MDPs which optimize policies against the worst-case dynamics within a predefined uncertainty set \citep{liu2024upperlowerboundsdistributionally,he2025samplecomplexitydistributionallyrobust,gu2025policyregularizeddistributionallyrobust}. However, these approaches often require carefully specified uncertainty sets and may fail to cover the true target dynamics. More importantly, they tend to be overly conservative due to the rigid assumption that there is no access to target data at all. More recent work have explored online RL domain adaptation using reward augmentation from dynamics mismatch \citep{eysenbach2021offdynamicsreinforcementlearningtraining,guo2024off} and representation deviation \citep{lyu2024crossdomainpolicyadaptationcapturing}. 
In contrast, our method tackles domain adaptation in a purely offline setting through a clustering-based framework, enabling more flexible and efficient reuse of source data under substantial dynamics shifts.

\textbf{Off-dynamics offline RL.}
Off-dynamics offline RL aims to leverage an abundant source dataset together with a limited target dataset to improve policy learning in the target domain where online interactions are infeasible. Prior works approached this problem by aligning source and target distributions via importance weighting \citep{nachum2019dualdice,lee2021optidice}, reward augmentation \citep{liu2022dara}, using adversarial domain adaptation to reduce the mismatch in state-action dynamics \citep{barde2020adversarial,xu2021offline}, or filtering source data based on similarity to the target \citep{liu2020provably,wen2024contrastive,lyu2025cross}. 
Another line of work adopts a dynamics model that is learned in the target domain to enable more principled reuse of source transitions or to help learn a robust Bellman operator under mismatched dynamics \citep{yu2020mopo,guo2026mobody,qiao2026dualrobust}. These methods often struggle with severe distribution shifts: importance weighting suffers from high variance under limited coverage, adversarial adaptation can be unstable in practice, heuristic filtering may reduce sample efficiency, and model-based methods are sensitive to model bias. Additionally, people study robust Markov decision processes, aiming to improve policy performance in an unseen target environment \citep{nilim2003robustness,elghaoui2005robust,clavier2024near,liu2024distributionally,lu2024distributionally,liu2024minimax,liu2025linear}. Since they assume that no target-domain data are accessible during training, these approaches rely on carefully designed uncertainty sets and may yield overly conservative policies, which may be unrealistic in practical settings. In contrast, our method balances robustness and coverage by retaining useful source experiences while downweighting unreliable clusters.

\textbf{Notations:} $\mathbb{H}(X)$ is the entropy of random variable $X$. $z=s\oplus a$ is the vector concatenation of state $s$ and action $a$, which can be viewed as an identity representation of $(s,a)$. 

\section{Preliminaries}
RL problems are typically formulated as a Markov Decision Process (MDP), specified by a tuple $\mathcal{M}=(\mathcal{S},\mathcal{A},P,r,\gamma,\rho_0)$, where $\mathcal{S}$ is the state space, $\mathcal{A}$ is the action space, $P(s'| s,a)$ is the transition kernel, $r(s,a)$ is the reward function given state-action pair $(s,a)$, $\gamma \in [0,1)$ is the discounting factor, and $\rho_0$ is the initial state distribution. In the off-dynamics offline setting, we consider two MDPs: a source domain $\mathcal{M}_{\text{src}}=(\mathcal{S},\mathcal{A},P_{\text{src}},r,\gamma,\rho_0)$ and a target domain $\mathcal{M}_{\text{tar}}=(\mathcal{S},\mathcal{A},P_{\text{tar}},r,\gamma,\rho_0)$. The two domains differ only in their transition dynamics, i.e., $P_{\text{src}}(s'| s,a) \neq P_{\text{tar}}(s'| s,a)$, while sharing the same reward function $r_{\text{src}}(s,a)=r_{\text{tar}}(s,a)$. In the offline setting, the agent cannot interact with the environment and must learn from static datasets \citep{levine2020offlinereinforcementlearningtutorial}. Suppose we are given a source dataset $D_{\text{src}}=\{(s,a,s',r)_{\text{src}}\}$ and a smaller target dataset $D_{\text{tar}}=\{(s,a,s',r)_{\text{tar}}\}$. The goal is to learn a policy $\pi$ that maximizes the expected return in the target domain, $\max_\pi \mathbb{E}_{\pi,P_{\mathrm{tar}}}\big[\sum_{t=0}^\infty \gamma^t r_{\mathrm{tar}}(s_t,a_t)\big]$, by leveraging the mixed dataset $D_{\text{mix}} = D_{\text{src}} \cup D_{\text{tar}}$. Typically we have $|D_{\text{tar}}| \ll |D_{\text{src}}|$.

\section{Methodology}

In this section, we analyze the target-domain performance gap between the optimal target policy and the policy learned from a filtered source dataset. This analysis motivates a clustering-based procedure to estimate local KL divergence for filtering source transitions, which we then integrate into policy optimization. Finally, we summarize the complete algorithm for off-dynamics offline RL. Due to space limitations, all proofs are deferred to Appendix~\ref{app:Proof}.

\subsection{Theoretical Insights}
The goal of data filtering in the off-dynamics setting is to create a filtered source dataset $\tilde{D}_{src}$ out of the source dataset $D_{src}$ such that the dynamics discrepancy between $\tilde{D}_{src}$ and the target dataset $D_{tar}$ is minimized. Following \citet{lyu2024crossdomainpolicyadaptationcapturing}, we have the following proposition.
\begin{proposition} \label[proposition]{thm:relation_btw_2_kl}
     For any $( s, a )$, denote its representation as $z$, and suppose $s_{\text {src }}^{\prime} \sim P_{\mathcal{M}_{\mathrm{src}}}(\cdot | z), s_{\mathrm{tar}}^{\prime} \sim P_{\mathcal{M}_{\mathrm{tar}}}(\cdot | z)$. Then measuring the representation deviation between the source domain and the target domain is equivalent to measuring the dynamics mismatch between two domains, i.e.,  $D_{\mathrm{KL}}\big(P(z | s_{\operatorname{src}}^{\prime}) \| P(z | s_{\mathrm{tar}}^{\prime})\big)=$ $D_{\mathrm{KL}}\big(P(s_{\mathrm{src}}^{\prime} | z) \| P(s_{\mathrm{tar}}^{\prime} | z)\big)+\mathbb{H}(s_{\mathrm{src}}^{\prime})-\mathbb{H}(s_{\mathrm{tar}}^{\prime})$.
\end{proposition}
The difference between \Cref{thm:relation_btw_2_kl} and Theorem 4.3 of \citet{lyu2024crossdomainpolicyadaptationcapturing} lies in the direction of the KL divergence. We derive the KL divergence from the source representation distribution to the target representation distribution (rather than the reverse). This choice matches our goal of characterizing how far the policy learned from filtered source data can be from the target-optimal policy (\Cref{thm:offline_perf_bound}). Since KL divergence is asymmetric, we reproduce the proof in Appendix~\ref{app:Proof} for completeness.

Since the entropy terms are constants for fixed source/target domains, minimizing dynamics deviation is equivalent to minimizing representation deviation. Concretely, let $\mathbb{H}(s_{\mathrm{src}}^{\prime})-\mathbb{H}(s_{\mathrm{tar}}^{\prime})=B$ be a constant. If the representation deviation is bounded by $\epsilon$, i.e., $D_{\mathrm{KL}}(P(z | s_{\operatorname{src}}^{\prime}) \| P(z | s_{\mathrm{tar}}^{\prime})) \leq \epsilon$, then the dynamics deviation $D_{\mathrm{KL}}(P(s_{\mathrm{src}}^{\prime} | z) \| P(s_{\mathrm{tar}}^{\prime} | z))$ is bounded by $\epsilon-B$.

The following assumption from \citet{eysenbach2021offdynamicsreinforcementlearningtraining} quantifies the performance gap that the optimal target policy $\pi^*$ may suffer when evaluated under source dynamics.
\begin{assumption}\label[assumption]{assum: no_exploit}
Let $\pi^*=\arg \max _\pi \mathbb{E}_p\big[\sum \gamma ^t r(s_t, a_t)\big]$ be the reward-maximizing policy in the target domain. Then the expected reward in source and target domains differs by at most $2 R_{\max} \sqrt{1-e^{B-\epsilon}}$:
\begin{align}
&\bigg|
\mathbb{E}_{\pi^*, P_{\text{src}}}\Big[\sum_t \gamma^t r(s_t, a_t)\Big]
-
\mathbb{E}_{\pi^*, P_{\text{tar}}}\Big[\sum_t \gamma^t r(s_t, a_t)\Big]
\bigg|
\leq 2 R_{\max} \sqrt{1 - e^{B - \epsilon}} ,\notag
\end{align}
where $B=\mathbb{H}(s_{\text {src }}^{\prime})-\mathbb{H}(s_{\text {tar}}^{\prime})$.
\end{assumption}

Now we present an upper bound for the performance gap in the target domain between the policy learned from the filtered source dataset and the optimal target policy.
\begin{theorem}[Offline performance bound]\label{thm:offline_perf_bound}
 Let $\hat{\pi}_{\tilde{D}_{src}}^*$ be the policy that maximizes the reward in the \textbf{filtered} source dataset $\tilde{D}_{src}$,  $\pi_{D_{tar}}^*$ be the policy that maximizes the reward in the \textbf{unfiltered} target dataset $D_{tar}$, and $\pi^*$ be the policy that maximizes the reward in the target domain. Assume that $\pi^*$ satisfies \Cref{assum: no_exploit} and let $C$ be the offline RL performance bound without dynamics shift, i.e., $\mathbb{E}_{P_{\text{tar}}, \pi^*} \big[\sum \gamma^tr(s_t, a_t)\big]-\mathbb{E}_{P_{\text{tar}}, \pi_{D_{tar}}^*}\big[\sum \gamma^tr(s_t, a_t)\big] \leq C$. Then $\hat{\pi}_{\tilde{D}_{src}}^*$ receives near-optimal reward on the target domain:
\begin{align}
\mathbb{E}_{P_{\text{tar}}, \hat{\pi}_{\tilde{D}_{src}}^*}
\Big[\sum_t \gamma^t r(s_t, a_t)\Big]
&\geq
\mathbb{E}_{P_{\text{tar}}, \pi^*}
\Big[\sum_t \gamma^t r(s_t, a_t)\Big]
-
\underbrace{C}_{\text{offline algorithm error}}
-
\frac{4 R_{\max}}{1-\gamma}
\underbrace{\sqrt{1-e^{B-\epsilon}}}_{\text{dynamics deviation}}
\notag\\
&\quad
-
\frac{2 R_{\max}}{1-\gamma}
\underbrace{
\mathbb{E}_{\rho_{\mathcal{M}_{\text{src}}}^{\hat{\pi}_{\tilde{D}_{src}}^*},\, P_{\text{src}}}
\Big[
D_{\mathrm{TV}}\big(
\hat{\pi}_{\tilde{D}_{src}}^*(\cdot|s')
\| \pi_{D_{tar}}^*(\cdot|s')
\big)
\Big]
}_{\text{filtering deviation}}
.\notag
\end{align}
\end{theorem}

\begin{remark}
Compared with the performance bounds from \citet{wen2024contrastive} and \citet{lyu2025cross}, our theorem offers two key advantages: (1) \emph{Guarantee against the true optimal target policy:} The bound compares the learned policy’s return in the target domain directly to the optimal target policy, ensuring the guarantee is stated against what is best achievable in the target environment. (2) \emph{Avoids reliance on the source behavior policy:} Prior works often compare against the source behavior policy, implicitly favoring imitation of the source dataset. In contrast, our bound measures policy deviation against the policy trained on the target dataset, which better reflects the objective of filtering.
\end{remark}

\subsection{Clustering-Based KL Estimation}
\Cref{thm:offline_perf_bound} motivates a clustering-based filtering strategy. In particular, the bound suggests that effective filtering should jointly reduce the \emph{dynamics deviation} and the \emph{filtering deviation}. Thus, to reduce the policy discrepancy between $\hat{\pi}_{\tilde{D}_{\mathrm{src}}}^*$ and $\pi_{D_{\mathrm{tar}}}^*$, the filtered source dataset $\tilde{D}_{\mathrm{src}}$ should match the target dataset $D_{\mathrm{tar}}$ as measured by the KL divergence of the dynamics mismatch.

We can view dynamics mismatch equivalently as the discrepancy of representation based on \Cref{thm:relation_btw_2_kl}. Concretely, we propose a clustering-based estimator for the representation divergence $D_{\mathrm{KL}}\big( P(z | s'_{\mathrm{src}}) \| P(z | s'_{\mathrm{tar}}) \big)$. The idea is to cluster similar next states so that within each cluster, differences between source and target state-action pairs primarily reflect \emph{localized} dynamics mismatch.

Given two datasets $D_{\mathrm{src}} = \{ (s_i, a_i, s'_i) \}_{i=1}^{N_{\mathrm{src}}}$ and $D_{\mathrm{tar}} = \{ (s_j, a_j, s'_j) \}_{j=1}^{N_{\mathrm{tar}}}$, we first run K-means on the next states $s'$ in the target dataset. We then assign source transitions to the nearest cluster based on their next states, while discarding those that are far away from the average diameter of all target clusters. This yields $K$ disjoint clusters $\{\mathcal{N}^n\}_{n=1}^K$, where each cluster $\mathcal{N}^n$ contains target transitions and retained source transitions with similar next states. We then pair each $s'$ in $\mathcal{N}^n$ with its corresponding $(s,a)$, forming $\mathcal{N}^n=\{(z_i,s_i')\}_{i=1}^{N_0^n+N_1^n}$ with $z_i=s_i \oplus a_i$ as an identity representation of the state-action pairs, where $N_0^n$ and $N_1^n$ are the numbers of source and target samples in $\mathcal{N}^n$ respectively. We assign binary labels within each cluster: target samples $z_i^{\mathrm{tar}}$ are labeled as class 1 and source samples $z_i^{\mathrm{src}}$ as class 0. The goal is to train a cluster-specific classifier $D_n(z) = P(y=1 |z)$ using the following cross-entropy loss
\begin{align}
L &= -  \mathbb{E}_{z \sim P(z |s'_{\mathrm{tar}})}[\log D_n(z)] -\mathbb{E}_{z \sim P(z |s'_{\mathrm{src}})}[\log (1 - D_n(z))].
\label{eq:BCE_loss}
\end{align}
By Bayes' rule, the classifier output satisfies
{\small
\begin{align}\label{classifier}
D_n(z) = \frac{P(z |y = 1)  P(y = 1)}{P(z |y = 1)  P(y = 1) + P(z |y = 0)  P(y = 0)}.
\end{align}}%
Within each cluster, the next states $s'$ are similar by construction. We therefore assume that the representation $z$ primarily reflects source domain and target domain label information, i.e., $P(z |y = 1) \approx P(z |s'_{\text{tar}}) $ and $ P(z | y = 0) \approx P(z | s'_{\text{src}}).$
Define a density ratio as
$w(z) = {P(z | s'_{\mathrm{tar}})}/{P(z | s'_{\mathrm{src}})}$.
Plug this approximation into \Cref{classifier} and we get
{\small
\begin{align*}
D_n(z) = \frac{A w_n(z)}{A w_n(z) + (1 - A)} \implies \hat{w}_n(z) = \frac{(1-A) D_n(z)}{A (1 - D_n(z))},
\end{align*}}%
where $A = P(y = 1) = N_1^n/{(N_1^n + N_0^n)}$ and $1 -A = P(y = 0) = N_0^n/(N_1^n + N_0^n)$. Since the KL divergence can be written as
$ D_{\mathrm{KL}}( P(z | s'_{\mathrm{src}}) \| P(z | s'_{\mathrm{tar}}) )= \mathbb{E}_{z \sim P(z | s'_{\mathrm{src}})} [\log 1/{w(z)}]$, we can estimate it using source samples:
\begin{align}
         \hat{D}^n_{\mathrm{KL}} = \frac{1}{N_0^n} \sum_{i=1}^{N_0^n} \log \frac{1}{\hat{w}_n(z_i^{\mathrm{src}})},
\label{eq:KL_estimate}
\end{align}
where $ \hat{D}^n_{\mathrm{KL}}$ is a \textit{cluster}  estimate of the KL divergence using cluster-specific classifier and $\log {1}/{\hat{w}_n(z_i^{\mathrm{src}})}=:d_i$
is a \textit{point} estimate of KL divergence. 

\subsection{Policy Optimization}
We now describe the policy optimization module of our approach. In offline RL, a key challenge is preventing out-of-distribution (OOD) actions at deployment time. Advantage-weighted methods \citep{peng2019advantage} mitigate this issue by weighting behavior cloning with estimated advantages, encouraging the policy to imitate dataset actions with higher values. In off-dynamics RL, however, naively imitating the mixed dataset, especially source transitions collected under mismatched dynamics, can degrade target-domain performance. To address this, we (i) filter the source dataset using KL estimates and (ii) use the (normalized) pointwise KL estimates to adaptively reweight learning. Intuitively, source transitions with smaller estimated KL divergence are more consistent with the target dynamics and therefore should have higher influence during optimization. We minimize the following loss function:
\begin{align} \label{eq:Q_opt}
\cL_Q
& = \EE_{D_{tar}}\big[\big(Q_\theta - \mathcal{T}Q_\theta\big)^2\big]  + \mathbb{E}_{(s,a,s',d) \sim \tilde{D}_{src}}\big[\exp\big(-\alpha \cdot \hat{d_i}\big)\big(Q_\theta - \mathcal{T}Q_\theta\big)^2\big],
\end{align}
where $\alpha$ is a hyperparameter and $\hat{d_i}=\frac{d_i-\max_i d_i}{\max_i d_i-\min_i d_i}$ for $i \in \{1,\dots,|\tilde{D}_{src}|\}$ is the normalized pointwise KL estimate.
After selectively filtering source transitions, where we retain more source transitions in clusters with small KL estimates and less in clusters with large KL estimates, the retained source samples are expected to be more target-consistent and cover the whole state space. Since the target dataset is relatively small and can also be noisy due to its data quality, assigning slightly higher aggregate weight to the filtered source data can improve data efficiency by leveraging abundant transitions.

In addition to KL-based reweighting, we introduce a policy-regularization term that encourages the learned policy to stay close to the support of the target dataset. Following \citet{lyu2025cross}, we train a conditional variational autoencoder (CVAE), denoted by $\hat{\pi}^{\text{tar}}_b(a | s)$, to model the target behavior policy. The policy is then optimized according to
\begin{align}
\mathcal{L}^{\text{reg}}_{\pi}
=
\mathcal{L}_{\pi}
-
\lambda
\mathbb{E}_{s \sim \tilde{D}_{\text{src}} \cup D_{\text{tar}}}
\big[
\log \hat{\pi}^{\text{tar}}_b\big(\pi(\cdot | s)| s\big)
\big],
\label{eq:pi_opt}
\end{align}
where $\mathcal{L}_{\pi}$ is the policy loss of the underlying offline RL algorithm, and $\lambda$ controls the strength of policy regularization. We use IQL \citep{kostrikov2021offline} as the backbone optimization algorithm.

Formally, we layout the general framework and detailed implementation of \algname in \Cref{fig:graph} and \Cref{app:\algname_alg} in Appendix~\ref{app:\algname} respectively.

\section{Experiment}
\label{sec:experiment}

In this section, we conduct extensive experiments to evaluate the performance of \algname on the ODRL benchmark for off-dynamics reinforcement learning \citep{lyu2024odrlbenchmarkoffdynamicsreinforcement}. We begin with MuJoCo locomotion tasks (HalfCheetah, Ant, Walker2d, and Hopper) under multiple types and levels of dynamics shifts. In addition, we evaluate \algname on navigation tasks in AntMaze and manipulation tasks in Adroit to demonstrate effectiveness beyond standard locomotion benchmarks. Finally, we conduct ablation studies to better understand the contributions of individual components, and the sensitivity to key hyperparameters. We also conduct experiments to compare the runtime complexity of \algname with that of other data filtering strategies to demonstrate the computational efficiency of our algorithm. We further conduct experiments to analyze LoDADA's sensitivity to target data coverage and quality to understand the potential limitations of \algname.

\begin{table*}[t]
\centering
\scriptsize
\setlength{\tabcolsep}{4pt}
\renewcommand{\arraystretch}{0.75}
\caption{Performance comparison on MuJoCo tasks under a medium-level offline dataset with dynamic shifts in gravity and friction at levels
0.1, 0.5, 2.0, 5.0. Source domains remain unchanged; target domains are shifted. We report normalized target-domain scores (mean $\pm$ std over five seeds). Best and second-best scores are highlighted in \textcolor{best}{green} and \textcolor{second}{blue}, respectively.
\label{tab:main_results_1}}
\resizebox{\textwidth}{!}{%
\begin{tabular}{ccccccccc}
\toprule
Env & Shift Level & BOSA & IQL & DARA & IGDF & OTDF & DmC & \algname\\
\midrule
\multirow{4}{*}{HalfCheetah Gravity}
 & 0.1 & 9.31$\pm$1.94 & 12.90$\pm$1.01 & 9.62$\pm$4.27 & 7.65$\pm$6.63 & 12.17$\pm$3.58 & \cellcolor{second}13.91$\pm$3.57 & \cellcolor{best}\textbf{14.09$\pm$3.83} \\
 & 0.5 & 39.92$\pm$5.63 & \cellcolor{second}43.68$\pm$2.96 & 40.87$\pm$5.46 & 39.73$\pm$4.62 & 37.12$\pm$4.88 & 42.24$\pm$2.86 & \cellcolor{best}\textbf{45.00$\pm$2.95} \\
 & 2.0 & 28.47$\pm$1.23 & \cellcolor{best}\textbf{31.12$\pm$0.3} & \cellcolor{second}30.49$\pm$0.54 & 30.16$\pm$1.5 & 25.38$\pm$3.12 & 26.91$\pm$2.07 & 28.02$\pm$1.67 \\
 & 5.0 & 11.27$\pm$1.77 & 20.87$\pm$21.45 & 32.42$\pm$18.54 & 30.08$\pm$34.02 & 24.01$\pm$17.88 & \cellcolor{second}40.21$\pm$29.28 & \cellcolor{best}\textbf{71.11$\pm$1.63} \\
\midrule
\multirow{4}{*}{HalfCheetah Friction}
 & 0.1 & 12.53$\pm$3.61 & \cellcolor{second}23.69$\pm$16.46 & \cellcolor{best}\textbf{26.39$\pm$11.35} & 20.56$\pm$11.22 & 10.31$\pm$1.68 & 9.26$\pm$0.49 & 9.97$\pm$0.47 \\
 & 0.5 & 63.65$\pm$9.14 & 68.26$\pm$1.42 & 68.41$\pm$3.19 & 66.30$\pm$1.79 & 66.32$\pm$0.46 & \cellcolor{best}\textbf{68.59$\pm$0.34} & \cellcolor{second}68.51$\pm$0.44 \\
 & 2.0 & 45.53$\pm$2.01 & 42.68$\pm$2.64 & 46.17$\pm$1.72 & \cellcolor{second}46.22$\pm$1.27 & 45.84$\pm$0.34 & 45.48$\pm$0.54 & \cellcolor{best}\textbf{48.19$\pm$0.75} \\
 & 5.0 & 38.26$\pm$10.82 & 48.97$\pm$4.68 & 45.31$\pm$9.83 & \cellcolor{best}\textbf{54.75$\pm$5.6} & 38.36$\pm$9.82 & 29.52$\pm$4.12 & \cellcolor{second}54.65$\pm$1.56 \\
\midrule
\multirow{4}{*}{Ant Gravity}
 & 0.1 & \cellcolor{best}\textbf{25.58$\pm$2.21} & 11.03$\pm$1.24 & 12.53$\pm$1.11 & 13.55$\pm$2.13 & 18.82$\pm$5.51 & 18.52$\pm$3.06 & \cellcolor{second}22.38$\pm$1.28 \\
 & 0.5 & 17.75$\pm$2.48 & 10.53$\pm$1.06 & 9.08$\pm$0.88 & 11.34$\pm$1.89 & 16.77$\pm$3.13 & \cellcolor{best}\textbf{25.26$\pm$2.60} & \cellcolor{second}18.61$\pm$2.79 \\
 & 2.0 & 37.17$\pm$0.96 & 38.26$\pm$3.09 & 38.54$\pm$2.19 & 41.59$\pm$4.63 & 37.11$\pm$0.74 & \cellcolor{second}42.87$\pm$2.97 & \cellcolor{best}\textbf{45.82$\pm$1.33} \\
 & 5.0 & 33.48$\pm$1.70 & 31.51$\pm$1.97 & 32.51$\pm$1.71 & 34.44$\pm$2.77 & 34.58$\pm$0.25 & \cellcolor{second}38.62$\pm$2.13 & \cellcolor{best}\textbf{48.29$\pm$3.68} \\
\midrule
\multirow{4}{*}{Ant Friction}
 & 0.1 & \cellcolor{best}\textbf{58.95$\pm$0.71} & 55.12$\pm$0.24 & \cellcolor{second}55.56$\pm$0.46 & 54.88$\pm$0.60 & 53.38$\pm$0.42 & 51.37$\pm$0.94 & 53.76$\pm$1.37 \\
 & 0.5 & \cellcolor{best}\textbf{59.72$\pm$3.57} & 58.92$\pm$0.80 & \cellcolor{second}59.28$\pm$0.80 & 55.64$\pm$5.03 & 58.20$\pm$0.84 & 48.29$\pm$6.57 & 55.54$\pm$3.80 \\
 & 2.0 & 20.18$\pm$3.79 & 17.54$\pm$2.47 & 19.84$\pm$3.20 & 19.63$\pm$4.15 & 32.33$\pm$3.45 & \cellcolor{second}59.13$\pm$5.39 & \cellcolor{best}\textbf{62.27$\pm$2.23} \\
 & 5.0 & 9.88$\pm$1.17 & 7.79$\pm$0.31 & 7.78$\pm$0.26 & \cellcolor{second}13.97$\pm$7.68 & 13.17$\pm$8.28 & 12.14$\pm$1.86 & \cellcolor{best}\textbf{28.62$\pm$3.43} \\
\midrule
\multirow{4}{*}{Walker2d Gravity}
 & 0.1 & 18.75$\pm$12.02 & 20.12$\pm$5.74 & 16.04$\pm$7.60 & 24.95$\pm$16.19 & 63.77$\pm$6.43 & \cellcolor{second}64.66$\pm$5.42 & \cellcolor{best}\textbf{71.41$\pm$1.99} \\
 & 0.5 & 40.09$\pm$2.37 & 29.72$\pm$16.02 & \cellcolor{second}42.05$\pm$10.52 & 37.36$\pm$5.49 & 32.47$\pm$7.04 & 35.45$\pm$5.89 & \cellcolor{best}\textbf{44.65$\pm$1.49} \\
 & 2.0 & 8.91$\pm$2.28 & 32.20$\pm$1.05 & 25.69$\pm$10.67 & 34.89$\pm$6.73 & 37.23$\pm$2.16 & \cellcolor{second}39.54$\pm$2.08 & \cellcolor{best}\textbf{40.27$\pm$1.77} \\
 & 5.0 & 5.25$\pm$0.50 & 5.44$\pm$0.08 & 5.42$\pm$0.29 & \cellcolor{second}5.66$\pm$0.32 & 4.92$\pm$0.08 & 5.11$\pm$0.11 & \cellcolor{best}\textbf{8.55$\pm$3.85} \\
\midrule
\multirow{4}{*}{Walker2d Friction}
 & 0.1 & 7.88$\pm$1.88 & 5.65$\pm$0.06 & 5.72$\pm$0.23 & 7.21$\pm$1.80 & \cellcolor{best}\textbf{30.40$\pm$7.54} & 18.24$\pm$6.50 & \cellcolor{second}20.21$\pm$7.90 \\
 & 0.5 & 53.42$\pm$20.06 & 67.91$\pm$7.03 & 68.92$\pm$11.43 & 70.83$\pm$4.22 & 69.06$\pm$9.44 & \cellcolor{best}\textbf{83.32$\pm$3.03} & \cellcolor{second}81.41$\pm$6.47 \\
 & 2.0 & 39.06$\pm$17.36 & \cellcolor{best}\textbf{72.91$\pm$0.37} & \cellcolor{second}65.40$\pm$7.13 & 62.55$\pm$10.36 & 59.93$\pm$14.53 & 62.15$\pm$6.13 & 62.01$\pm$4.36 \\
 & 5.0 & 5.89$\pm$3.42 & 5.38$\pm$0.09 & 5.28$\pm$0.19 & 5.21$\pm$0.36 & \cellcolor{second}14.98$\pm$9.15 & 14.75$\pm$2.89 & \cellcolor{best}\textbf{15.93$\pm$3.09} \\
\midrule
\multirow{4}{*}{Hopper Gravity}
 & 0.1 & 27.82$\pm$13.41 & 23.4$\pm$11.62 & 13.10$\pm$0.98 & 13.89$\pm$4.70 & 33.41$\pm$1.92 & \cellcolor{second}35.73$\pm$0.93 & \cellcolor{best}\textbf{37.68$\pm$6.32} \\
 & 0.5 & 17.74$\pm$7.93 & 17.78$\pm$6.85 & 14.68$\pm$5.24 & 19.81$\pm$9.56 & \cellcolor{second}35.22$\pm$5.94 & 15.28$\pm$2.43 & \cellcolor{best}\textbf{43.72$\pm$9.72} \\
 & 2.0 & 13.18$\pm$2.10 & 14.71$\pm$1.67 & 14.94$\pm$1.08 & 14.71$\pm$0.87 & \cellcolor{best}\textbf{19.82$\pm$1.77} & \cellcolor{second}18.30$\pm$1.55 & 16.99$\pm$3.74 \\
 & 5.0 & 7.97$\pm$0.61 & 7.78$\pm$0.37 & 7.81$\pm$0.48 & 8.03$\pm$0.50 & \cellcolor{second}8.78$\pm$0.03 & 8.77$\pm$0.02 & \cellcolor{best}\textbf{9.28$\pm$0.26} \\
\midrule
\multirow{4}{*}{Hopper Friction}
 & 0.1 & 25.55$\pm$2.69 & \cellcolor{second}26.13$\pm$4.24 & 24.16$\pm$4.50 & 24.02$\pm$5.00 & 24.76$\pm$1.21 & 25.57$\pm$3.22 & \cellcolor{best}\textbf{28.06$\pm$7.97} \\
 & 0.5 & 25.86$\pm$2.17 & 24.75$\pm$3.58 & 24.84$\pm$4.63 & 29.12$\pm$3.01 & \cellcolor{second}29.89$\pm$6.17 & 27.42$\pm$2.15 & \cellcolor{best}\textbf{33.57$\pm$1.31} \\
 & 2.0 & 10.32$\pm$0.06 & 10.15$\pm$0.03 & 10.15$\pm$0.06 & 10.35$\pm$0.60 & \cellcolor{best}\textbf{10.52$\pm$0.29} & 10.31$\pm$0.12 & \cellcolor{second}10.44$\pm$0.13 \\
 & 5.0 & 6.92$\pm$1.97 & 7.92$\pm$0.06 & 7.84$\pm$0.04 & 7.90$\pm$0.17 & \cellcolor{second}8.07$\pm$0.04 & \cellcolor{best}\textbf{8.08$\pm$0.08} & 8.04$\pm$0.05 \\
\midrule
Total & & 826.26 & 894.82 & 886.84 & 916.98 & 1007.10 & \cellcolor{second}1045.00 & \cellcolor{best}\textbf{1207.05}\\
\bottomrule
\end{tabular}}
\end{table*}

\subsection{Experimental Setup}
\label{subsection:setup}

\textbf{Task and Environments.} We evaluate \algname on the ODRL benchmark \citep{lyu2024odrlbenchmarkoffdynamicsreinforcement} across MuJoCo environments with different dynamics shifts. Each environment includes three types of dynamics shifts: gravity, friction and morphology shifts. The source dataset is collected in the original MuJoCo environment from D4RL  \citep{fu2021d4rldatasetsdeepdatadriven}, and the target dataset is collected by the behavior policy in the corresponding environment with dynamics shifts. 
All datasets are collected using SAC \citep{haarnoja2018softactorcriticoffpolicymaximum} trained to a medium performance level in either the original or dynamics-shifted environment. We use the datasets provided by ODRL, where the source dataset contains 1 million transitions while the target one has 5000 transitions. 
To construct dynamics-shifted target environments, we follow ODRL and modify MuJoCo simulator parameters for gravity, friction and morphology. 
For gravity, we rescale the global gravity magnitude while keeping its direction unchanged. For friction, we rescale the friction parameters.
In both cases, we consider four shift levels $\{0.1, 0.5, 2.0, 5.0\}$, where the original gravity or friction is multiplied by the corresponding factor.
Morphology shifts are created by changing the size of specific limbs or the torso of the robot.
In the main experiments, both source and target datasets are collected by medium-level behavior policies following ODRL.

In addition to the locomotion tasks, we evaluate on AntMaze problems with different map layouts following ODRL.
We focus on the medium-sized AntMaze setting and use 6 different map layouts as target environments, denoted as $\{1,2,3,4,5,6\}$.
We use the \emph{play} and \emph{diverse} datasets as source datasets, both drawn from D4RL \citep{fu2021d4rldatasetsdeepdatadriven}, while the target dataset is collected by an agent trained directly in the target environment and provided by ODRL.
The diverse dataset consists of trajectories that start from random locations and reach random goals, whereas the play dataset contains trajectories that start from hand-picked locations and reach a fixed hand-picked goal. Consequently, these two source datasets exhibit different state coverage. Finally, we evaluate \algname on manipulation tasks from the Adroit benchmark.
We consider the \emph{pen} and \emph{door} environments.
Similar to the MuJoCo experiments, we construct target environments by introducing kinematic and morphological shifts.
For kinematic shifts, we restrict the rotation ranges of hand joints, referred to as \emph{broken-joint}.
For morphological shifts, we shrink the sizes of fingers, referred to as \emph{shrink-finger}.
We consider shift levels of medium and hard in our experiments.
The source dataset is drawn from D4RL at the human level, while the target dataset is obtained from ODRL.
Additional details can be found in Appendix~\ref{app:experiment}.

\textbf{Baselines.} We compare \algname against several strong baseline methods including 
DARA \citep{liu2022dara}, BOSA \citep{liu2024beyond}, IQL \citep{kostrikov2021offline}, IGDF \citep{wen2024contrastive}, OTDF \citep{lyu2025cross}, and DmC \citep{van2025dmc}. 
We first consider methods that learn from whole combined source and target datasets.
DARA applies reward augmentation learned from the discrepancy between source and target dynamics to the source dataset to mimic the trajectories in the target environment.
IQL is an offline RL algorithm that learns a value function and an advantage-weighted policy to achieve higher performance, which demonstrates strong performance in offline RL tasks trained on mixed datasets without explicit domain adaptation. 
We also compare with the data filtering methods.
BOSA enforces support constraints with two objectives to mitigate OOD state–action pairs and OOD dynamics while filtering unreliable transitions.
IGDF learns a contrastive representation to estimate a mutual information-based domain gap for filtering source transitions.
OTDF is the most recent state-of-the-art method which uses optimal transport to match source–target transitions and performs non-parametric filtering based on the optimal transport cost.
DmC uses nearest-neighbor estimation to measure domain discrepancies and guides a diffusion model to generate additional target aligned transitions.

\textbf{Implementation Details.}
For \algname, we fix the filtering ratios at $(90\%,80\%,70\%)$ for low-, medium-, and high-KL clusters across all tasks.
We select the number of clusters $K$ from $\{30,50\}$ and the policy regularization weight $\lambda$ from $\{0.1,0.5,1.0\}$ for MuJoCo and Adroit tasks. For AntMaze, we select $K$ from $\{5,10,15, 20\}$ and $\lambda$ from $\{0.5,1.0,5.0\}$. Additional implementation details and hyperparameters are provided in \Cref{hyperparameter} in Appendix~\ref{app:experiment}.

\subsection{Results on MuJoCo Locomotion Tasks}
We first summarize the performance of \algname and the baseline methods on MuJoCo locomotion tasks under various dynamics shifts, including gravity, friction, morphology shifts, and local perturbations.
We evaluate performance using the normalized score
$\textit{Normalized Score} = \frac{\textit{score} - \textit{random score}}{\textit{expert score} - \textit{random score}} \times 100$,
where the \textit{random score} is achieved by a random policy and the \textit{expert score} is achieved by SAC \citep{haarnoja2018softactorcriticoffpolicymaximum} trained to the expert level in the target domain.

\textbf{Results on gravity and friction shifts.}
As shown in \Cref{tab:main_results_1}, \algname achieves a total score of 1207.05, yielding a 15.5\% improvement over the second-best method DmC \citep{van2025dmc}, a 19.9\% improvement over OTDF \citep{lyu2025cross}, and a 34.9\% improvement over the backbone algorithm IQL \citep{kostrikov2021offline}.
Across the 32 evaluation tasks, \algname achieves best (18) or second-best (7) performance in 25 out of 32 tasks and remains competitive on the remaining 7 tasks.
\textbf{These results indicate that \algname improves target-domain performance by effectively leveraging localized filtering.}

The policy optimization module in \algname builds on IQL. Compared to vanilla IQL, we introduce (i) adaptive weighting in critic learning and (ii) an additional CVAE-based behavior-cloning regularizer.
As shown in \Cref{tab:main_results_1}, these modifications yield significant gains over IQL on 27 out of 32 tasks.
\textbf{These observations suggest that combining localized filtering with target-aware regularization improves adaptation to the target domain.}

\textbf{Additional MuJoCo evaluations.}
We further evaluate \algname on MuJoCo morphology shifts and the setting with medium-level source data and expert-level target data.
Detailed results are reported in \Cref{morph_shift_comparison} in Appendix~\ref{app:morph_results} and in \Cref{tab:medium_expert_results} in Appendix~\ref{app:medium_expert}.
For morphology shifts, \algname achieves the best score in 13 out of 16 tasks and improves over the second-best method by 32.7\%.
When the target dataset is collected by an expert policy, \algname also achieves the highest aggregate score.
These results show that our method remains effective beyond the standard medium-level setting.

\subsection{Experiments on Localized Perturbations}
In this experiment, we evaluate \algname under \emph{locally perturbed} dynamics.
Unlike the previous setting with \emph{global} shifts (e.g., gravity changes that affect transitions uniformly), now the dynamics vary across regions of the state space.
This setting is more challenging because the distribution mismatch is heterogeneous and cannot be well addressed by methods that rely on global assumptions, such as DARA \citep{liu2022dara} and IGDF \citep{wen2024contrastive}.

We keep the target domain unchanged and perturb only the source dataset to simulate localized variations. Policies are trained on the perturbed source dataset and evaluated in the unchanged target domain.
For the target dataset, we randomly sample 5,000 transitions from D4RL \citep{fu2021d4rldatasetsdeepdatadriven} without introducing any dynamics shift.
For the source dataset, we start from the original D4RL dataset and partition the next-state space into 15 regions using K-means. We then divide the regions into $M$ groups and inject Gaussian noise into state variables, with all samples in the same group sharing the same noise variance.
We evaluate performance using the normalized score defined above. We consider all four MuJoCo environments (Ant, Hopper, HalfCheetah, Walker2d) and set $M\in\{3,5\}$. When $M=3$, the noise variances are $\{0.1, 0.5, 2.0\}$, $\{0.2, 1.0, 4.0\}$, and $\{0.3, 1.5, 6.0\}$. When $M=5$, the variances are $\{0.1, 0.2, 0.3, 0.4, 0.5\}$, $\{0.2, 0.4, 0.6, 0.8, 1.0\}$, and $\{0.3, 0.6, 0.9, 1.2, 1.5\}$.

We summarize the aggregate results in \Cref{fig:mujoco_total_scores_line} and partial results in \Cref{tab:local_perturbation_main}. More detailed results are deferred to \Cref{tab:main_results_8} in Appendix~\ref{app:localized_perturbation_results} due to space limit.
As shown in \Cref{fig:mujoco_total_scores_line}, \algname achieves the best total score under both group settings.
Across all localized perturbation tasks, \algname obtains a total score of 943.09, improving over the second-best method OTDF by 29.3\% and over IQL by 55.1\%.
These results demonstrate that \algname is effective under localized perturbations, where cluster-aware filtering becomes more beneficial than under global shifts.

\begin{table*}[t]
\centering
\small
\setlength{\tabcolsep}{4pt}
\renewcommand{\arraystretch}{0.75}
\caption{
Performance on localized perturbation tasks for Hopper.
We report normalized target-domain scores over five seeds.
}
\label{tab:local_perturbation_main}
\resizebox{\textwidth}{!}{
\begin{tabular}{cccccccccc}
\toprule
Env & \# Groups & Shift Level & BOSA & DARA & IQL & IGDF & OTDF & DmC & \algname \\
\midrule

\multirow{6}{*}{\makecell[c]{Hopper \\ Localized Perturbation}}
&\multirow{3}{*}{3}
& 0.1-0.5-2.0 & 7.84$\pm$2.84  & 8.70$\pm$6.67 & 8.14$\pm$6.31 & 13.08$\pm$10.44 & \cellcolor{second}33.96$\pm$4.57 & 30.61$\pm$10.01 & \cellcolor{best}\textbf{40.77$\pm$5.04} \\
&& 0.2-1.0-4.0  & 11.91$\pm$2.60 & 9.61$\pm$5.56 & 11.97$\pm$5.99 & 8.35$\pm$3.56 & \cellcolor{second}32.58$\pm$6.89 & 29.53$\pm$5.73 & \cellcolor{best}\textbf{37.51$\pm$9.16} \\
&& 0.3-1.5-6.0  & 11.85$\pm$3.19 & 8.11$\pm$3.54 & 10.17$\pm$5.50 & 6.59$\pm$2.59 & \cellcolor{second}31.77$\pm$3.04 & 27.87$\pm$5.27 & \cellcolor{best}\textbf{33.41$\pm$2.72} \\
\cmidrule(lr){2-10}

&\multirow{3}{*}{5}
 & 0.1-0.2-0.3-0.4-0.5 & 10.42$\pm$3.55 & 6.37$\pm$1.67 & 6.92$\pm$1.85 & 15.81$\pm$9.78 & 22.25$\pm$8.74 & \cellcolor{best}\textbf{40.35$\pm$8.19} & \cellcolor{second}40.30$\pm$4.61 \\
 && 0.2-0.4-0.6-0.8-1.0 & 11.63$\pm$1.87 & 14.07$\pm$5.65 & 9.82$\pm$4.29 & 11.24$\pm$7.86 & 21.42$\pm$9.11 & \cellcolor{second}34.36$\pm$4.82 & \cellcolor{best}\textbf{36.06$\pm$4.59} \\
 && 0.3-0.6-0.9-1.2-1.5 & 13.97$\pm$5.15 & 11.14$\pm$4.01 & 12.01$\pm$2.06 & 15.63$\pm$10.74 & \cellcolor{second}29.95$\pm$5.01 & 18.79$\pm$5.32 & \cellcolor{best}\textbf{35.02$\pm$5.70} \\
\bottomrule
\end{tabular}}
\end{table*}

\subsection{Results on AntMaze and Adroit}

In this section, we present additional results on navigation and manipulation tasks described in \Cref{subsection:setup}.
We summarize the aggregate performance on AntMaze and Adroit in \Cref{fig:total_scores_all}, and defer the detailed numerical results to the appendix.

\begin{figure}[t]
    \centering
    \begin{subfigure}{0.32\textwidth}
        \centering
        \includegraphics[width=\linewidth]{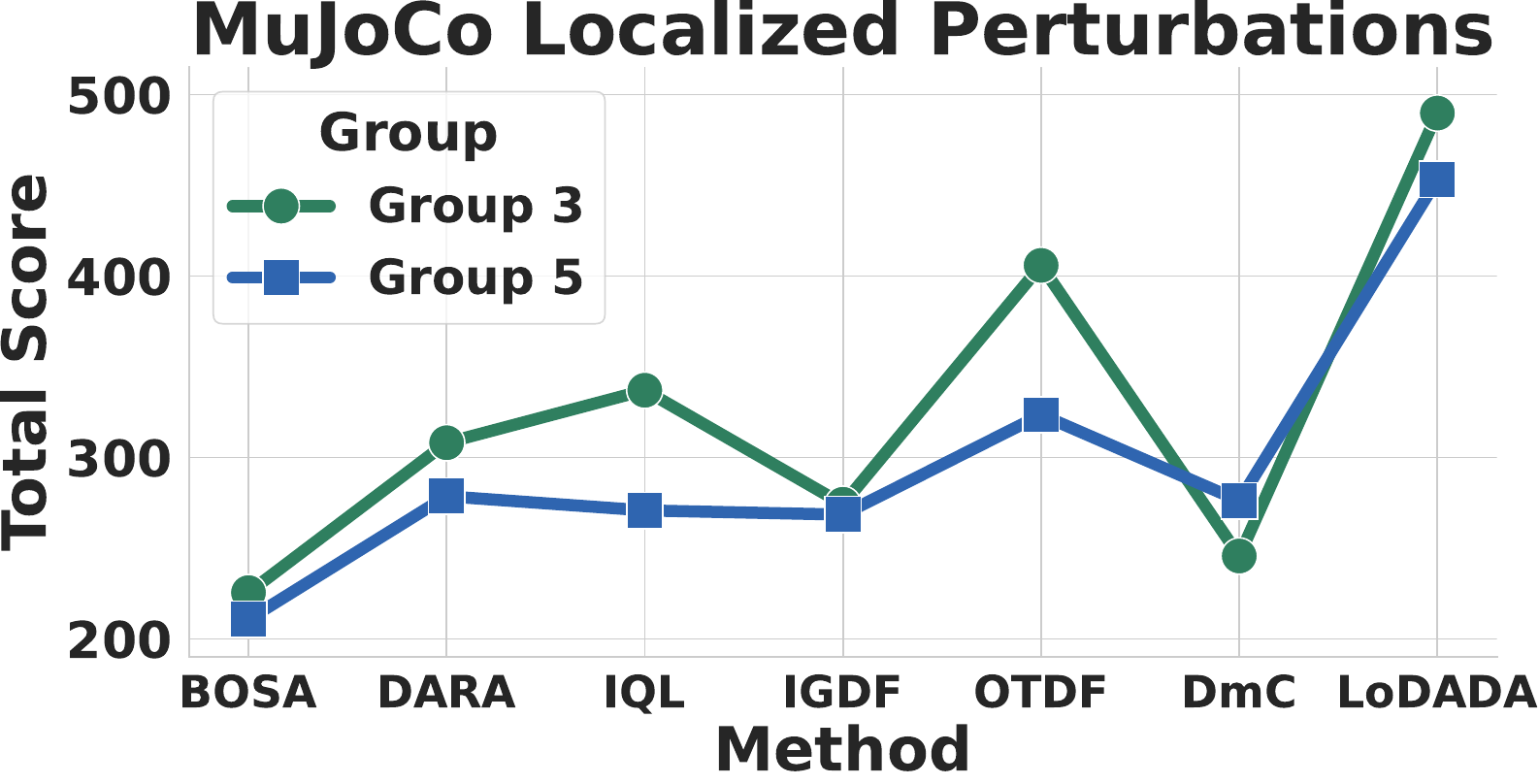}
        \caption{MuJoCo localized perturbation.}
        \label{fig:mujoco_total_scores_line}
    \end{subfigure}
    \hfill
    \begin{subfigure}{0.32\textwidth}
        \centering
        \includegraphics[width=\linewidth]{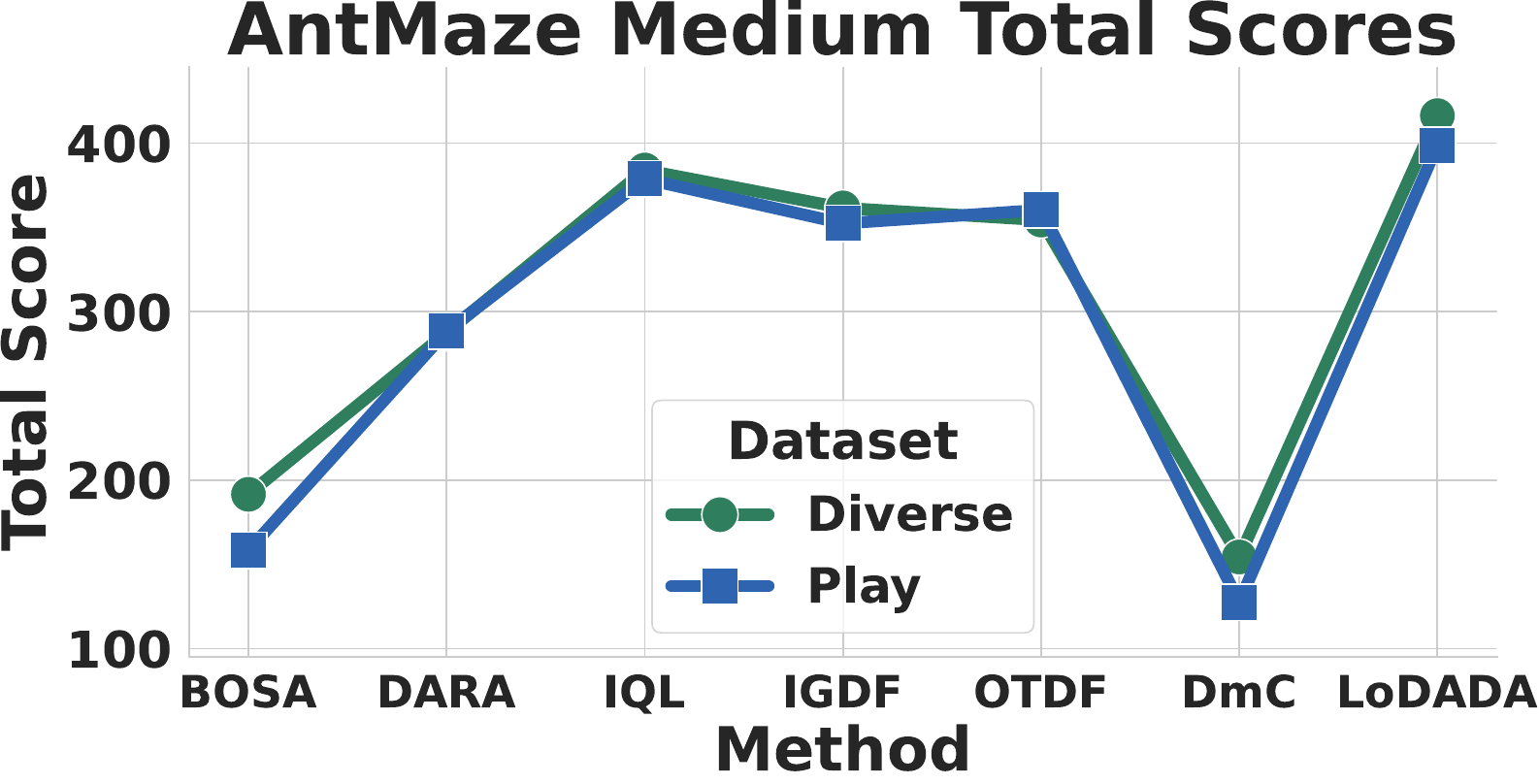}
        \caption{AntMaze-medium.}
        \label{fig:antmaze_total_scores_line}
    \end{subfigure}
    \hfill
    \begin{subfigure}{0.32\textwidth}
        \centering
        \includegraphics[width=\linewidth]{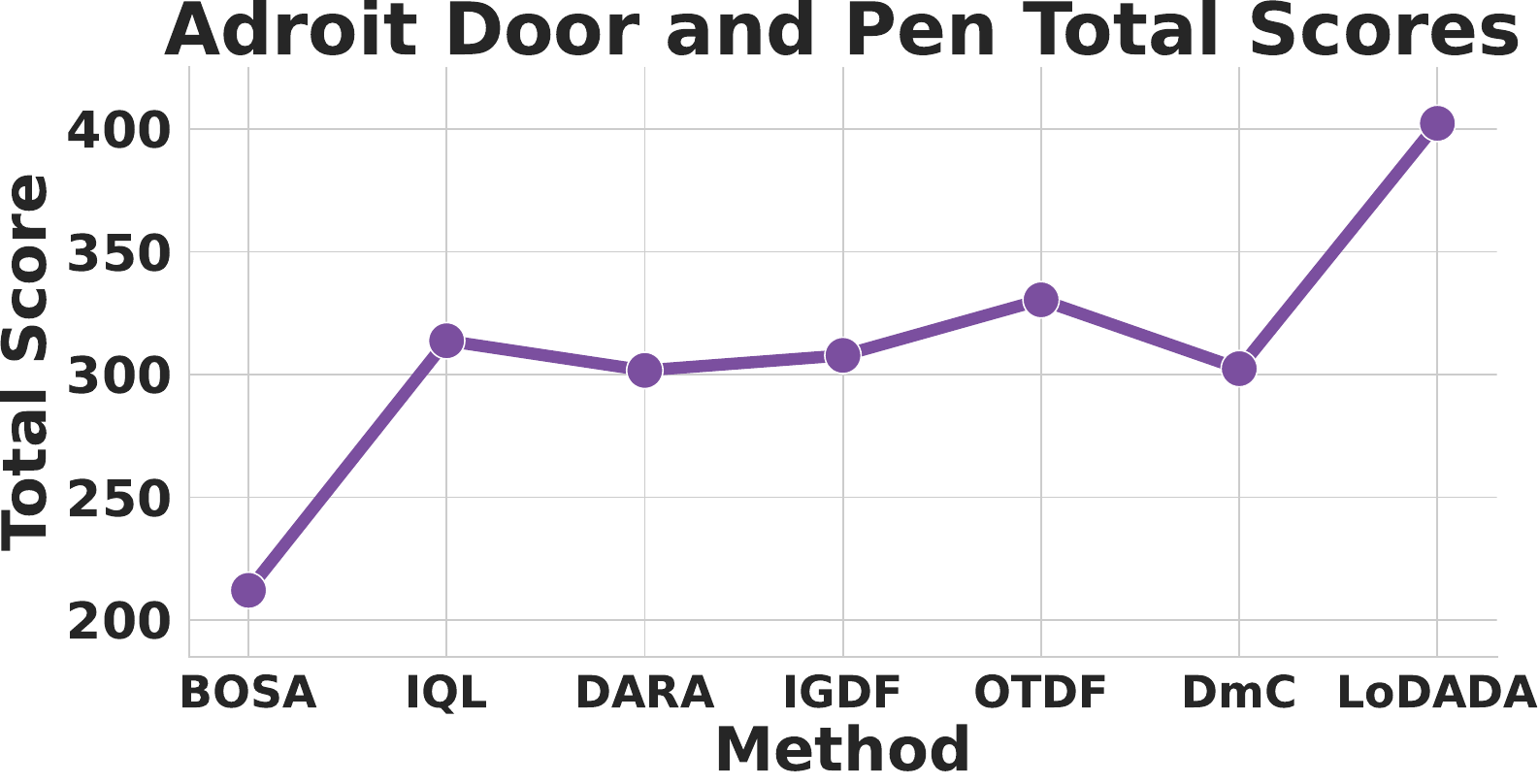}
        \caption{Adroit Door and Pen.}
        \label{fig:adroit_total_scores_line}
    \end{subfigure}
    \caption{
    Performance on MuJoCo localized perturbation, AntMaze, and Adroit tasks.
    Left: total scores on MuJoCo localized perturbation tasks with 3 and 5 perturbed region groups.
    Middle: total scores on AntMaze-medium tasks using diverse and play source datasets.
    Right: total scores on Adroit Door and Pen tasks under morphology and kinematic shifts.
    }
    \label{fig:total_scores_all}
    \vspace{-10pt}
\end{figure}

\textbf{AntMaze.}
As shown in \Cref{fig:antmaze_total_scores_line}, \algname achieves the best aggregate performance on AntMaze-medium tasks under both source datasets.
This shows that localized filtering can retain source transitions that better match the target map layouts.
Compared with prior methods, \algname better preserves useful transitions in structured navigation environments, where local state-space similarity is important.
These results show that \algname generalizes beyond locomotion tasks to navigation problems.
Detailed results are reported in \Cref{tab:ant_play_results} and \Cref{tab:ant_diverse_results} in Appendix~\ref{app:more_antmaze_results_play}.

\textbf{Adroit.}
As shown in \Cref{fig:adroit_total_scores_line}, \algname achieves the highest aggregate score on Adroit manipulation benchmarks.
This shows that our method remains effective under morphology and kinematic shifts in high-dimensional manipulation tasks.
In these tasks, global source-target matching can be unreliable because different hand configurations and object interactions may create non-uniform dynamics mismatch, while localized dynamics-aware filtering keeps source transitions that are locally closer to the target dynamics.
Detailed Door and Pen results are reported in \Cref{tab: morph_shift_comparison_adroit} in Appendix~\ref{app:experiemnts_androit}.

\subsection{Ablation Study}
We conduct eight ablation studies to provide deeper insights into \algname and the effect of key hyperparameters in locomotion tasks. For demonstration of hyperparameter sensitivity, we use the Hopper environment in gravity shift with four shift levels (0.1, 0.5, 2.0, 5.0), focusing on the number of clusters $K$ and the policy-regularization weight $\lambda$. Due to the space limit, we defer detailed results to Appendix~\ref{sec:more_exp_results}.

\textbf{Ablation 1: Number of clusters $K$.}
$K$ controls the level of locality of the state space partition used in our cluster-aware filtering mechanism. A smaller $K$ induces coarser clusters, which may merge heterogeneous transition dynamics, while a large $K$ can lead to overly fine partitions with insufficient samples per cluster. To study this trade-off, we sweep $K \in \{10,30,50,100\}$ and run \algname in the hopper-gravity environment with medium-level source and target datasets. As shown in \Cref{fig:parameter_K}, performance degrades when $K$ is either too small (\textit{e.g.}, 10) or too large (\textit{e.g.}, 100), while intermediate values achieve consistently better results.
\textbf{These results indicate that a moderate level of locality is critical for effective cluster-aware filtering, balancing dynamics homogeneity and data sufficiency.}
As a result, we use $K=30$ or $50$ for MuJoCo and Adroit tasks.  Adroit results are summarized in \Cref{tab:adroit_hyperparameter_sensitivity} in Appendix~\ref{app:additional_hyperparameter_sensitivity}.
For AntMaze, we use smaller cluster numbers due to the structured map, as discussed in \Cref{tab:antmaze_hyperparameter_sensitivity} in Appendix~\ref{app:additional_hyperparameter_sensitivity}.

\textbf{Ablation 2: Policy regularization weight $\lambda$.}
$\lambda$ controls how strongly the learned policy is regularized toward the target behavior policy.
We sweep $\lambda \in \{0,0.1,0.5,1.0\}$ in the Hopper-Gravity setting.
As shown in \Cref{fig:parameter_lambda}, removing this regularization ($\lambda=0$) leads to poor performance, while the best value can vary across shift levels.
Overall, $\lambda=0.5$ provides a strong default across MuJoCo settings.

\begin{wrapfigure}{r}{0.5\textwidth}
    \centering
    \vspace{-10pt}
    \begin{subfigure}[b]{0.24\textwidth}
        \includegraphics[width=\linewidth]{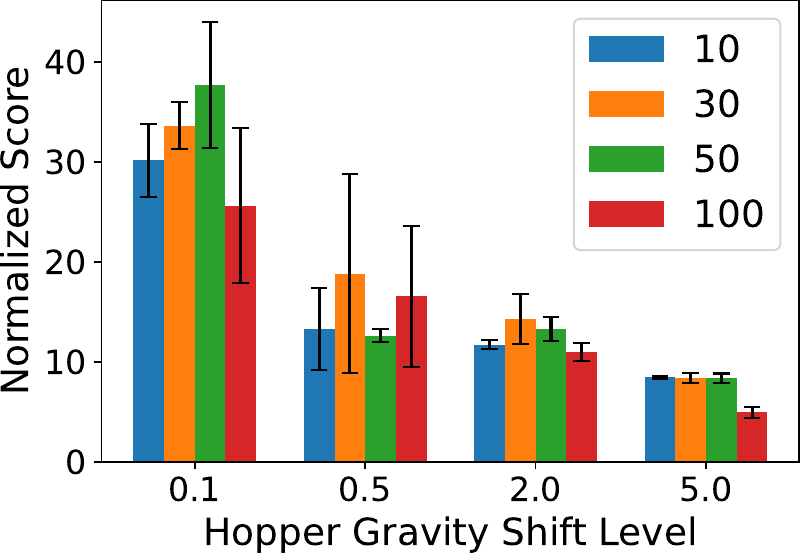}
        \caption{Ablation study on $K$.}
        \label{fig:parameter_K}
    \end{subfigure}
    \hfill
    \begin{subfigure}[b]{0.24\textwidth}
        \includegraphics[width=\linewidth]{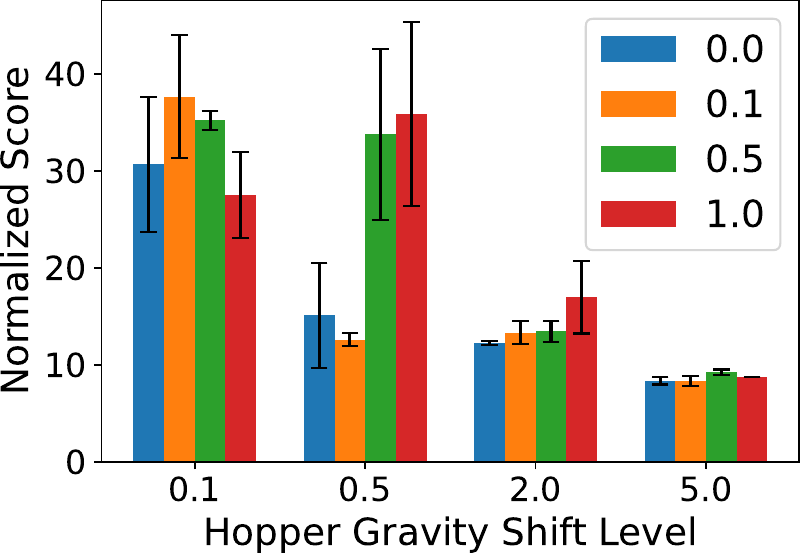}
        \caption{Ablation study on $\lambda$.}
        \label{fig:parameter_lambda}
    \end{subfigure}
    \caption{Parameter sensitivity study.}
    \label{fig:parameter_sensitivity}
    \vspace{-10pt}
\end{wrapfigure}

\textbf{Ablation 3: Cluster-level sample coverage.}
Since cluster-level KL estimation depends on target samples in each cluster, we report representative sample counts and classifier convergence diagnostics in \Cref{tab:cluster_coverage} in Appendix~\ref{app:cluster_coverage}.
The examined clusters contain samples from both domains, and the local classifiers converge early, suggesting that the KL estimates are stable in these settings.

\textbf{Ablation 4: Filtering ratios.}
We fix the filtering ratios as $(\xi_1,\xi_2,\xi_3)=(90\%,80\%,70\%)$ across all environments.
We evaluate more aggressive filtering ratios on both MuJoCo and AntMaze tasks. As shown in \Cref{tab:filtering_ratio_sensitivity} and \Cref{tab:filtering_ratio_sensitivity_antmaze}, the default setting achieves the highest performance. We visualize the filtering results on AntMaze in \Cref{fig:antmaze_filtering}. The visualization illustrates that \algname filters source data that are less aligned with the target data, while more aggressive filtering ratios remove more data and reduce state space coverage.

\begin{figure}
    \centering
    \includegraphics[width=0.9\linewidth]{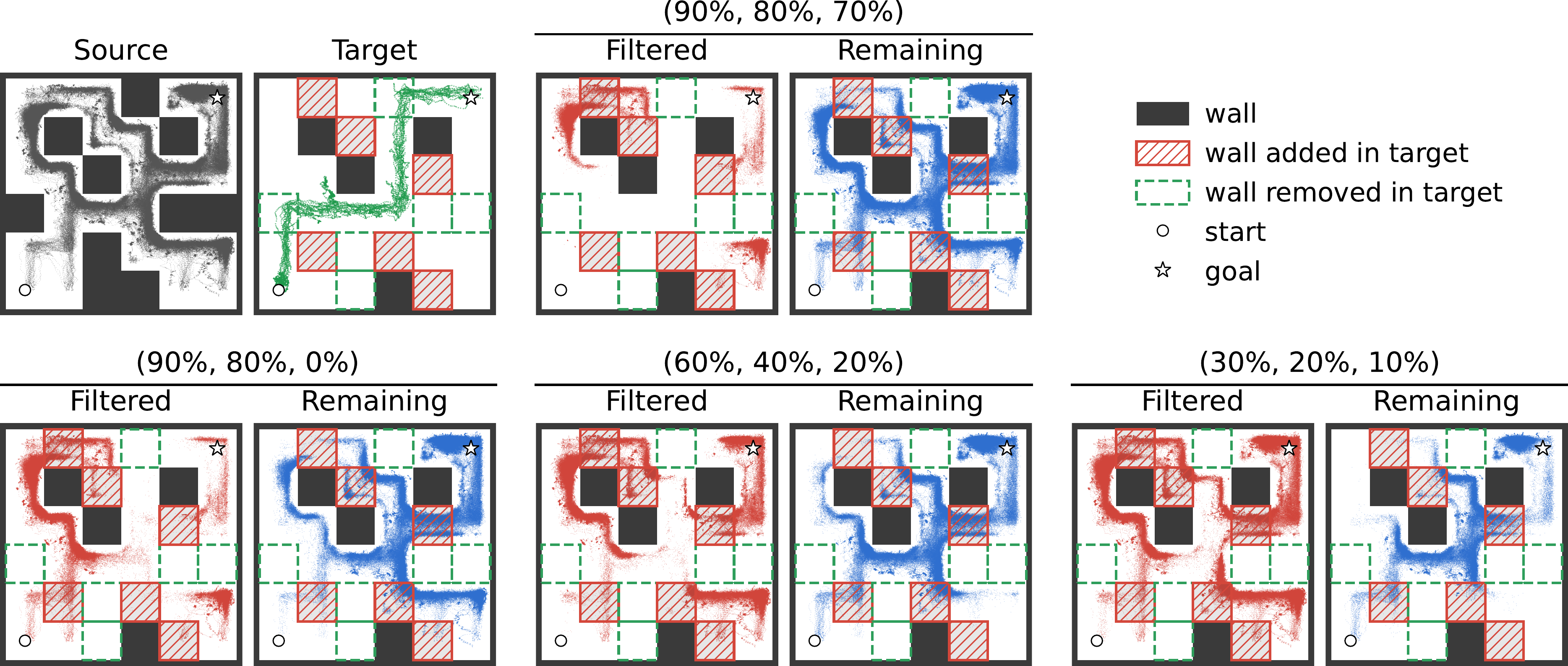}
\caption{Visualization of source data filtering on AntMaze Play Map 2 under different filtering ratios. Red and green regions indicate walls added and removed in the target domain. Red and blue points represent filtered and retained source transitions. The default setting in our navigation experiments uses $(\xi_1,\xi_2,\xi_3)=(90\%,80\%,70\%)$.}
\label{fig:antmaze_filtering}
\end{figure}

\textbf{Ablation 5: Component contributions.}
\Cref{tab:component_ablation} in Appendix~\ref{app:component_ablation} studies the effect of removing KL-based critic weighting, removing CVAE policy regularization, and replacing local discriminator based filtering with L2-distance filtering.
The full method achieves the best aggregate performance, showing that these components are complementary.
\Cref{tab:cvae_baselines} in Appendix~\ref{app:cvae_baselines} further shows that adding the same CVAE regularizer to baseline methods does not fully explain the gain of \algname.

\textbf{Ablation 6: Runtime comparison.}
We report the wall-clock runtime of different filtering methods in \Cref{tab:runtime} in Appendix~\ref{sec:experiments_runtime_comparison}.
\algname has a higher pre-optimization cost than IGDF because it trains local classifiers, but it is much faster than OTDF, which requires expensive optimal-transport computation.
Together with the performance results, this shows that \algname provides a practical trade-off between computational cost and final policy performance.

\textbf{Ablation 7: Target dataset size.}
We investigate the effectiveness of \algname with different target data availability. \Cref{tab:sparse_target_data} in Appendix~\ref{sec:target_data_robustness} evaluates \algname with progressively reduced target dataset sizes in locomotion tasks. The results show that \algname remains relatively stable when up to $50\%$ of target data are removed, with more noticeable degradation under extreme sparsity settings when only $10\%$-$20\%$ of target data remain.

\textbf{Ablation 8: Target dataset quality.}
We further investigate the performance of \algname under different target data quality. \Cref{tab:target_data_quality} in Appendix~\ref{sec:target_data_robustness} compares \algname using medium, random-medium, and random target datasets. Results suggest that LoDADA's performance generally degrades as target data quality decreases due to the limited target-policy coverage provided by low-quality data, highlighting the importance of reliable target supervision for effective adaptation.

\section{Conclusion and Limitations}
In this work, we studied off-dynamics offline RL through a localized dynamics-aware data filtering method. We provided a theoretical performance bound that motivates us to perform data filtering from a local perspective and introduce a dataset regularization term to encourage the learned policy to align with the target behavior. Building on this insight, we proposed \algname and conducted extensive experiments under various types and levels of both global and local dynamics shifts. Experiment results demonstrate that \algname outperforms existing baselines across a wide range of tasks, with particularly strong gains under locally perturbed dynamics. Our method has some limitations. It relies on cluster-level KL estimation, so performance may naturally degrade when target coverage is extremely sparse, target data quality is poor, or source-target overlap is very small. Future work includes exploring the effectiveness of \algname in real-world environments and exploring more principled clustering mechanisms to better capture fine-grained state-action structures. In addition, our proposed \algname is orthogonal to the data generation methods and can be combined with them to unleash the power of synthetic methods in solving dynamics shift problems.

\section*{Acknowledgments}
The authors would like to thank the anonymous reviewers for their constructive comments and valuable feedback. YY and PX are supported in part by the National Science Foundation under Grant Nos. DMS-2323112 and IIS-2541960 (CAREER), and by the Whitehead Scholars Program at the Duke University School of Medicine. Any opinions, findings, and conclusions or recommendations expressed in this material are those of the authors and do not necessarily reflect the views of any funding agency.

\bibliographystyle{plain}
\bibliography{reference}

\newpage
\appendix

\section{Additional Related Works}
\textbf{Clustering-based methods in RL.}
Clustering-based methods in RL exploit structure in the state-action space by grouping similar samples into clusters, thereby reducing variance and improving generalization. Recent approaches use clustering for representation learning, such as learning latent embeddings that capture environment dynamics and then clustering them to discover reusable behaviors or intermediate objectives \citep{machado2017laplacian, eysenbach2018diversity}. These approaches often rely on heuristic similarity metrics, which may not accurately reflect the underlying transition dynamics. Our method leverages clustering not only to structure the data, but also integrates distributional alignment within clusters, ensuring that the grouped samples are both coherent and dynamically relevant to the target domain.

\section{Proof of the Main Theoretical Results}
\label{app:Proof}

For simplicity, when we say a policy $\pi$ is trained on the filtered source domain, we mean it satisfies  $\mathbb{E}_{\substack{a \sim \pi(a \mid s), s^{\prime} \sim p\left(s^{\prime} \mid s, a\right)}}\left[\sum_t D_{\mathrm{KL}}\left(P_{\text{src}}\left(s'_{t+1} \mid z_t\right) \| P_{\text{tar}}\left(s'_{t+1} \mid z_t\right)\right)\right] \leq \epsilon-B$,
where $B=\mathbb{H}\left(s_{\text {src }}^{\prime}\right)-\mathbb{H}\left(s_{\text {tar}}^{\prime}\right), z_t=s_t \oplus a_t$. 

\subsection{Proof of \Cref{thm:relation_btw_2_kl}}

\begin{proof}[Proof of \Cref{thm:relation_btw_2_kl}]
This proof is similar to that of Theorem 4.3 in \citet{lyu2024crossdomainpolicyadaptationcapturing}. However, we derive this result by calculating the KL divergence between the source representation distribution and the target representation distribution, rather than vice versa. This is because our proposed method \algname requires characterizing the distance between the policy learned from the filtered source data and the target optimal policy, as demonstrated in \Cref{thm:offline_perf_bound}. Due to the asymmetry of KL divergence, we reproduce the proof here for the completeness of our work.

We would like to establish a connection between the representation deviations in the two domains and the dynamics discrepancies between the two domains. Since the source and target MDPs share the same state space, 
$s'_{\mathrm{src}}$ and $s'_{\mathrm{tar}}$ take values in the same measurable space, although they are sampled from different transition dynamics.
In the proof, $P(z,s'_{\mathrm{src}},s'_{\mathrm{tar}})$ denotes an auxiliary joint coupling whose marginals satisfy
\begin{align*}
\int_{\mathcal{S}} P(z,s'_{\mathrm{src}},s'_{\mathrm{tar}})ds'_{\mathrm{tar}}
= P(z,s'_{\mathrm{src}}), \qquad
\int_{\mathcal{S}} P(z,s'_{\mathrm{src}},s'_{\mathrm{tar}})ds'_{\mathrm{src}}
= P(z,s'_{\mathrm{tar}}).
\end{align*}
This construction does not assume that the source and target transition dynamics are the same; it only allows the two next-state variables to be written on a common product space, so that the additional integral over the unused variable marginalizes out. To achieve this, we introduce the concept of mutual information:
\begin{align}
&h\left(z ; s_{\mathrm{src}}^{\prime}, s_{\mathrm{tar}}^{\prime}\right) \notag\\
& =I\left(z ; s_{\mathrm{src}}^{\prime}\right)-I\left(z ; s_{\mathrm{tar}}^{\prime}\right)  \notag\\
& =\int_{\mathcal{Z}} \int_{\mathcal{S}} P\left(z, s_{\mathrm{src}}^{\prime}\right) \log \frac{P\left(z, s_{\mathrm{src}}^{\prime}\right)}{P(z) P\left(s_{\mathrm{src}}^{\prime}\right)} d z d s_{\mathrm{src}}^{\prime}-\int_{\mathcal{Z}} \int_{\mathcal{S}} P\left(z, s_{\mathrm{tar}}^{\prime}\right) \log \frac{P\left(z, s_{\mathrm{tar}}^{\prime}\right)}{P(z) P\left(s_{\mathrm{tar}}^{\prime}\right)} d z d s_{\mathrm{tar}}^{\prime}\notag \\
& =\int_{\mathcal{Z}} \int_{\mathcal{S}} P\left(z, s_{\mathrm{src}}^{\prime}\right) \log \frac{P\left(z | s_{\mathrm{src}}^{\prime}\right)}{P(z)} d z d s_{\mathrm{src}}^{\prime}-\int_{\mathcal{Z}} \int_{\mathcal{S}} P\left(z, s_{\mathrm{tar}}^{\prime}\right) \log \frac{P\left(z | s_{\mathrm{tar}}^{\prime}\right)}{P(z)} d z d s_{\mathrm{tar}}^{\prime}\notag \\
&= \int_{\mathcal{Z}} \int_{\mathcal{S}} \int_{\mathcal{S}} 
    P(z, s'_{\text{src}}, s'_{\text{tar}}) 
    \log \frac{P(z | s'_{\text{src}})}{P(z)} 
    \, dz \, ds'_{\text{src}} \, ds'_{\text{tar}} \notag\\
&\qquad - \int_{\mathcal{Z}} \int_{\mathcal{S}} \int_{\mathcal{S}} 
    P(z, s'_{\text{tar}}, s'_{\text{src}}) 
    \log \frac{P(z | s'_{\text{tar}})}{P(z)} 
    \, dz \, ds'_{\text{tar}} \, ds'_{\text{src}}\notag\\
& =\int_{\mathcal{Z}} \int_{\mathcal{S}} \int_{\mathcal{S}} P\left(z, s_{\mathrm{src}}^{\prime}, s_{\mathrm{tar}}^{\prime}\right) \log \frac{P\left(z | s_{\mathrm{src}}^{\prime}\right)}{P\left(z | s_{\mathrm{tar}}^{\prime}\right)} d z d s_{\mathrm{src}}^{\prime} d s_{\mathrm{tar}}^{\prime}\notag \\
& =D_{\mathrm{KL}}\left(P\left(z | s_{\mathrm{src}}^{\prime}\right) \| P\left(z | s_{\mathrm{tar}}^{\prime}\right)\right) .\notag
\end{align}
Here the last line should be understood as an averaged conditional KL divergence:
\begin{align*}
D_{\mathrm{KL}}\left(P(z|s'_{\mathrm{src}})\|P(z|s'_{\mathrm{tar}})\right)
:=
\mathbb{E}_{s'_{\mathrm{src}},s'_{\mathrm{tar}}}
\left[
D_{\mathrm{KL}}\left(P(z|s'_{\mathrm{src}})\|P(z|s'_{\mathrm{tar}})\right)
\right].
\end{align*}
For each fixed pair $(s'_{\mathrm{src}},s'_{\mathrm{tar}})$, the inner KL is taken with respect to the numerator distribution $P(z|s'_{\mathrm{src}})$, while the outer expectation averages this conditional KL under the auxiliary coupling defined above.At the same time, we also have 
\begin{align}
&h\left(z ; s_{\mathrm{src}}^{\prime}, s_{\mathrm{tar}}^{\prime}\right) \notag\\
&= I\left(z ; s_{\mathrm{src}}^{\prime}\right) - I\left(z ; s_{\mathrm{tar}}^{\prime}\right)\notag \\
&= \int_{\mathcal{Z}} \int_{\mathcal{S}} 
    P\left(z, s_{\mathrm{src}}^{\prime}\right) 
    \log \frac{P\left(z, s_{\mathrm{src}}^{\prime}\right)}{P(z) P\left(s_{\mathrm{src}}^{\prime}\right)} 
    \, dz \, ds_{\mathrm{src}}^{\prime} 
    - \int_{\mathcal{Z}} \int_{\mathcal{S}} 
    P\left(z, s_{\mathrm{tar}}^{\prime}\right) 
    \log \frac{P\left(z, s_{\mathrm{tar}}^{\prime}\right)}{P(z) P\left(s_{\mathrm{tar}}^{\prime}\right)} 
    \, dz \, ds_{\mathrm{tar}}^{\prime}\notag \\
&= \int_{\mathcal{Z}} \int_{\mathcal{S}} 
    P\left(z, s_{\mathrm{src}}^{\prime}\right) 
    \log \frac{P\left(s_{\mathrm{src}}^{\prime} | z\right)}{P\left(s_{\mathrm{src}}^{\prime}\right)} 
    \, dz \, ds_{\mathrm{src}}^{\prime} 
    - \int_{\mathcal{Z}} \int_{\mathcal{S}} 
    P\left(z, s_{\mathrm{tar}}^{\prime}\right) 
    \log \frac{P\left(s_{\mathrm{tar}}^{\prime} | z\right)}{P\left(s_{\mathrm{tar}}^{\prime}\right)} 
    \, dz \, ds_{\mathrm{tar}}^{\prime} \notag\\
&= \int_{\mathcal{Z}} \int_{\mathcal{S}} \int_{\mathcal{S}} 
    P\left(z, s_{\mathrm{src}}^{\prime}, s_{\mathrm{tar}}^{\prime}\right) 
    \log \frac{P\left(s_{\mathrm{src}}^{\prime} | z\right)}{P\left(s_{\mathrm{src}}^{\prime}\right)} 
    \, dz \, ds_{\mathrm{src}}^{\prime} \, ds_{\mathrm{tar}}^{\prime} \notag\\
&\qquad - \int_{\mathcal{Z}} \int_{\mathcal{S}} \int_{\mathcal{S}} 
    P\left(z, s_{\mathrm{tar}}^{\prime}, s_{\mathrm{src}}^{\prime}\right) 
    \log \frac{P\left(s_{\mathrm{tar}}^{\prime} | z\right)}{P\left(s_{\mathrm{tar}}^{\prime}\right)} 
    \, dz \, ds_{\mathrm{tar}}^{\prime} \, ds_{\mathrm{src}}^{\prime}\notag \\
&= \int_{\mathcal{Z}} \int_{\mathcal{S}} \int_{\mathcal{S}} 
    P\left(z, s_{\mathrm{src}}^{\prime}, s_{\mathrm{tar}}^{\prime}\right) 
    \log \frac{P\left(s_{\mathrm{src}}^{\prime} | z\right)}{P\left(s_{\mathrm{tar}}^{\prime} | z\right)} 
    \, dz \, ds_{\mathrm{src}}^{\prime} \, ds_{\mathrm{tar}}^{\prime}\notag \\
&\qquad - \int_{\mathcal{S}} 
    P\left(s_{\mathrm{src}}^{\prime}\right) 
    \log P\left(s_{\mathrm{src}}^{\prime}\right) 
    \, ds_{\mathrm{src}}^{\prime} 
    + \int_{\mathcal{S}} 
    P\left(s_{\mathrm{tar}}^{\prime}\right) 
    \log P\left(s_{\mathrm{tar}}^{\prime}\right) 
    \, ds_{\mathrm{tar}}^{\prime} \notag\\
&= D_{\mathrm{KL}}\left( 
    P\left(s_{\mathrm{src}}^{\prime} | z\right) 
    \, \| \, 
    P\left(s_{\mathrm{tar}}^{\prime} | z\right) 
    \right) 
    + \mathbb{H}\left(s_{\mathrm{src}}^{\prime}\right) 
    - \mathbb{H}\left(s_{\mathrm{tar}}^{\prime}\right).\notag
\end{align}
Similarly, the dynamics KL term is an averaged conditional KL:
\begin{align*}
D_{\mathrm{KL}}\left(P(s'_{\mathrm{src}}|z)\|P(s'_{\mathrm{tar}}|z)\right)
:=
\mathbb{E}_{z}
\left[
D_{\mathrm{KL}}\left(P_{\mathrm{src}}(s'|z)\|P_{\mathrm{tar}}(s'|z)\right)
\right],
\end{align*}
where the inner KL is taken over the same next-state variable $s'\in\mathcal{S}$ with respect to the numerator distribution $P_{\mathrm{src}}(s'|z)$.
One can see that the defined function is also connected to the dynamics discrepancy term $D_{\mathrm{KL}}\left( 
    P\left(s_{\mathrm{src}}^{\prime} | z\right) 
    \, \| \, 
    P\left(s_{\mathrm{tar}}^{\prime} | z\right) 
    \right) $ and two entropy terms. Nevertheless, we observe that the source domain and the target domain are specified and fixed, and their state distributions are also fixed, indicating that the entropy terms are constants. So in the end we have
\begin{align}
\underbrace{D_{\mathrm{KL}}\left(P\left(z | s_{\text {src}}^{\prime}\right) \| P\left(z | s_{\text {tar}}^{\prime}\right)\right)}_{\text {representation deviation }}=\underbrace{D_{\mathrm{KL}}\left(P\left(s_{\text {src }}^{\prime} | z\right) \| P\left(s_{\text {tar}}^{\prime} | z\right)\right)}_{\text {dynamics deviation }}+\underbrace{\mathbb{H}\left(s_{\text {src }}^{\prime}\right)-\mathbb{H}\left(s_{\text {tar}}^{\prime}\right)}_{\text {constants}} .\notag
\end{align}
This completes the proof.
\end{proof}

\subsection{Proof of Theorem \ref{thm:offline_perf_bound}}

\begin{proof}[Proof of \Cref{thm:offline_perf_bound}]
For simplicity, denote $\tilde{r} = \sum \gamma^tr\left(s_t, a_t\right)$.  Since It is infeasible to directly interact an online policy $\pi^*$ with an offline policy $\hat{\pi}_{\tilde{D}_{src}}^*$ in different domains,  we introduce 
$\pi_{\mathrm{src}}^*$ that maximizes the reward in the filtered source domain.
Then we have
\begin{align}
    \mathbb{E}_{P_{\text{tar}}, \pi^*}(\tilde{r})-\mathbb{E}_{P_{\text{tar}}, \hat{\pi}_{\tilde{D}_{src}}^*}(\tilde{r})&= \underbrace{\mathbb{E}_{P_{\text{tar}}, \pi^*}(\tilde{r})-\mathbb{E}_{P_{\text{src}}, \pi^*} (\tilde{r})}_{(a)} +   \underbrace{\mathbb{E}_{P_{\text{src}}, \pi^*} (\tilde{r})-\mathbb{E}_{P_{\text{src}}, \pi_{\text{src}}^*}(\tilde{r}) }_{(b)} \notag\\
    &+\underbrace{\mathbb{E}_{P_{\text{src}}, \pi_{\text{src}}^*}(\tilde{r})-\mathbb{E}_{P_{\text{tar}}, \pi_{D_{tar}}^*}(\tilde{r})}_{(c)}+\underbrace{\mathbb{E}_{P_{\text{tar}}, \pi_{D_{tar}}^*}(\tilde{r})-\mathbb{E}_{P_{\text{tar}}, \hat{\pi}_{\tilde{D}_{src}}^*}(\tilde{r})}_{(d)}.\notag
\end{align}
First note that term (b) is 0 because $\pi_{\mathrm{src}}^*$ is the optimal policy in the filtered source domain which should be identical to the target domain, meaning $\pi_{\mathrm{src}}^*$ is almost the same as $\pi^*$ and they should achieve the same reward in the same domain. Term (a) is bounded above by $2 R_{\max } \sqrt{1-e^{B-\epsilon}}$ by \Cref{assum: no_exploit}.
To bound term (c), we first need the following result by applying \Cref{lem:no_exploit}:
\begin{align}
    \mathbb{E}_{P_{\text{src}}, \pi_{\mathrm{src}}^*} (\tilde{r})-\mathbb{E}_{P_{\text{tar}}, \pi^*}(\tilde{r})
    \leq \mathbb{E}_{P_{\text{src}}, \pi_{\mathrm{src}}^*} (\tilde{r})-\mathbb{E}_{P_{\text{tar}}, \pi_{\mathrm{src}}^*} (\tilde{r})
    \leq 2 R_{\max}\sqrt{1-e^{B-\epsilon}} \notag.
\end{align}
Next we can use this and the performance bound of usual offline RL algorithm to find a bound for term (c):
\begin{align}
    \mathbb{E}_{P_{\text{src}}, \pi_{\text{src}}^*}(\tilde{r})
    \leq 2 R_{\max}\sqrt{1-e^{B-\epsilon}} +\mathbb{E}_{P_{\text{tar}}, \pi^*}(\tilde{r})\leq  2 R_{\max}\sqrt{1-e^{B-\epsilon}}+C+\mathbb{E}_{P_{\text{tar}}, \pi_{D_{tar}}^*}(\tilde{r})\notag. 
\end{align}
Rearrange the terms can we find the upper bound for term (c).
Finally we will deal with term (d):
\begin{align}
    (d)=\underbrace{\mathbb{E}_{P_{\text{tar}}, \pi_{D_{tar}}^*}(\tilde{r})-\mathbb{E}_{P_{\text{src}}, \pi_{D_{tar}}^*}(\tilde{r})}_{I_1}+\underbrace{\mathbb{E}_{P_{\text{src}}, \pi_{D_{tar}}^*}(\tilde{r})-\mathbb{E}_{P_{\text{src}}, \hat{\pi}_{\tilde{D}_{src}}^*}(\tilde{r})}_{I_2}+\underbrace{\mathbb{E}_{P_{\text{src}}, \hat{\pi}_{\tilde{D}_{src}}^*}(\tilde{r})-\mathbb{E}_{P_{\text{tar}}, \hat{\pi}_{\tilde{D}_{src}}^*}(\tilde{r})}_{I_3} . \notag
\end{align}
We bound $I_2$ by applying \Cref{lem:extendeD_{tar}elescoping_lem}:
\begin{align}
    I_2&\leq\left| \mathbb{E}_{P_{\text{src}}, \hat{\pi}_{\tilde{D}_{src}}^*}(\tilde{r})-\mathbb{E}_{P_{\text{src}}, \pi_{D_{tar}}^*}(\tilde{r}) \right|\notag\\
    &=\left| \frac{1}{1-\gamma} \mathbb{E}_{\rho_{\mathcal{M}_{\text{src}}}^{ \hat{\pi}_{\tilde{D}_{src}}^*}(s, a), s^{\prime} \sim P_{\text{src}}}\left[\mathbb{E}_{a^{\prime} \sim  \hat{\pi}_{\tilde{D}_{src}}^*}\left[Q_{\mathcal{M}_{\text{src}}}^{\pi_{D_{tar}}^*}\left(s^{\prime}, a^{\prime}\right)\right]-\mathbb{E}_{a^{\prime} \sim {\pi_{D_{tar}}^*}}\left[Q_{\mathcal{M_{\text{src}}}}^{\pi_{D_{tar}}^*}\left(s^{\prime}, a^{\prime}\right)\right]\right]\right|\notag\\
    &\leq \frac{1}{1-\gamma} 
    \mathbb{E}_{\rho_{\mathcal{M}_{\text{src}}}^{ \hat{\pi}_{\tilde{D}_{src}}^*}(s, a), s^{\prime} \sim P_{\text{src}}}
    \left| \sum_{a' \in \mathcal{A}} 
    \big( \hat{\pi}_{\tilde{D}_{src}}^*(a'|s') - {\pi_{D_{tar}}^*}(a'|s')\big) 
    Q^{\pi_{D_{tar}}^*}_{\mathcal{M}_{\text{src}}}(s',a') \right|\notag \\ 
    &\leq \frac{R_{\max}}{1-\gamma} 
    \mathbb{E}_{\rho_{\mathcal{M}_{\text{src}}}^{ \hat{\pi}_{\tilde{D}_{src}}^*}(s, a), s^{\prime} \sim P_{\text{src}}}
     \sum_{a' \in \mathcal{A}}\left| 
    \big( \hat{\pi}_{\tilde{D}_{src}}^*(a'|s') - {\pi_{D_{tar}}^*}(a'|s')) \right| \notag\\
    &= \frac{2R_{\max}}{1-\gamma} 
    \mathbb{E}_{\rho_{\mathcal{M}_{\text{src}}}^{ \hat{\pi}_{\tilde{D}_{src}}^*}(s, a), s^{\prime} \sim P_{\text{src}}}
    \Big[ D_{\mathrm{TV}}\big( \hat{\pi}_{\tilde{D}_{src}}^*(\cdot|s') \,\|\, {\pi_{D_{tar}}^*}(\cdot|s')\big) \Big].\notag
\end{align}
We bound $I_1$ by applying \Cref{lem:telescoping_lem}:
\begin{align}
    I_1 &\leq\left| \mathbb{E}_{P_{\text{src}}, \pi_{D_{tar}}^*}(\tilde{r})-\mathbb{E}_{P_{\text{tar}}, \pi_{D_{tar}}^*}(\tilde{r}) \right|\notag\\
    &=\left|\frac{\gamma}{1-\gamma} \mathbb{E}_{\rho_{\mathcal{M}_{\mathrm{src}}}^{\pi_{D_{tar}}^*}(s, a)}\left[\mathbb{E}_{s^{\prime} \sim P_{\text{src}}}\left[V_{\mathcal{M}_{\mathrm{tar}}}^{\pi_{D_{tar}}^*}\left(s^{\prime}\right)\right]-\mathbb{E}_{s^{\prime} \sim P_{\text{tar}}}\left[V_{\mathcal{M}_{\mathrm{tar}}}^{\pi_{D_{tar}}^*}\left(s^{\prime}\right)\right]\right]\right| \notag\\
    &= \left|\frac{\gamma}{1-\gamma} \mathbb{E}_{\rho_{\mathcal{M}_{\mathrm{src}}}^{\pi_{D_{tar}}^*}(s, a)}\left[ \int_{\mathcal{S}}\left(P_{\mathcal{M}_{\mathrm{src}}}\left(s^{\prime} \mid s, a\right)-P_{\mathcal{M}_{\mathrm{tar}}}\left(s^{\prime} \mid s, a\right)\right)V_{\mathcal{M}_{\mathrm{tar}}}^{\pi_{D_{tar}}^*}\left(s^{\prime}\right) ds'\right]\right|\notag\\
    &\leq  \frac{\gamma}{1-\gamma} \mathbb{E}_{\rho_{\mathcal{M}_{\mathrm{src}}}^{\pi_{D_{tar}}^*}(s, a)}\left[ \left|\int_{\mathcal{S}}\left(P_{\mathcal{M}_{\mathrm{src}}}\left(s^{\prime} \mid s, a\right)-P_{\mathcal{M}_{\mathrm{tar}}}\left(s^{\prime} \mid s, a\right)\right)\right|\times \left|V_{\mathcal{M}_{\mathrm{tar}}}^{\pi_{D_{tar}}^*}\left(s^{\prime}\right) \right|ds'\right]\notag\\
    &\leq \frac{2\gamma}{1-\gamma} R_{\max}\left[D_{\mathrm{TV}}\left(P_{\mathcal{M}_{\mathrm{src}}}(\cdot \mid s, a) \| P_{\mathcal{M}_{\mathrm{tar}}}(\cdot \mid s, a)\right)\right]\notag\\
    &\leq \frac{2\gamma}{1-\gamma} R_{\max}\sqrt{1-e^{B-\epsilon}}.\notag
\end{align}
Similarly for $I_3$, we get the same upper bound. Combining all bounds together, we get the desired result. 
\end{proof}

\section{Useful Lemmas}

\begin{lemma} 
Let policy $\pi$ be trained on the filtered source domain, and $R_{\text {max}}$ be the maximum (entropy-regularized) return of any trajectory. Then the following inequality holds:
\begin{align}
\left|\mathbb{E}_{\pi,P_{\text{src}}}\left[\sum \gamma^tr\left(s_t, a_t\right)\right]-\mathbb{E}_{\pi,P_{\text{tar}}}\left[\sum \gamma^tr\left(s_t, a_t\right)\right]\right| \leq 2 R_{\max } \sqrt{1-e^{B-\epsilon}}. \notag
\end{align}
\label{lem:no_exploit}
\end{lemma}
This lemma proves that policies trained on the filtered source domain satisfy the same condition as the optimal policy (\Cref{assum: no_exploit}). 

\begin{proof} This proof is similar to that of Lemma B.2 in \citet{eysenbach2021offdynamicsreinforcementlearningtraining}, and the only difference is that we use B-H bound for TV distance instead of Pinsker's inequality in the last step. We repeat the proof here for readers' reference.
We know $z=f(s,a)$ is injective, and apply Holder's inequality and Bretagnolle–Huber bound, which is never worse than Pinsker's inequality \citep{canonne2023shortnoteinequalitykl}, to obtain the desired result:
\begin{align}
&\left|\mathbb{E}_{P_{\text{src}}}\left[\sum \gamma^tr\left(s_t, a_t\right)\right]-\mathbb{E}_{P_{\text{tar}}}\left[\sum \gamma^tr\left(s_t, a_t\right)\right]\right| \notag\\
& =\left|\sum_\tau\left(P_{\text{src}}(\tau)-P_{\text{tar}}(\tau)\right)\left(\sum \gamma^tr\left(s_t, a_t\right)\right)\right| \notag\\
& \leq\left\|\sum \gamma^tr\left(s_t, a_t\right)\right\|_{\infty} \cdot\left\|P_{\text{src}}(\tau)-P_{\text{tar}}(\tau)\right\|_1 \notag\\
& \leq\left(\max _\tau \sum \gamma^tr\left(s_t, a_t\right)\right) \cdot 2 \sqrt{1-e^{-D_{\mathrm{KL}}\left(P_{\text{src}}(\tau) \| P_{\text{tar}}(\tau)\right)}} \notag\\
& \leq 2 R_{\max } \sqrt{1-e^{B-\epsilon}} .\notag
\end{align}
This completes the proof.
\end{proof}

\begin{lemma}[Telescoping lemma, Lemma 4.3 from \citet{luo2021algorithmicframeworkmodelbaseddeep}]. Denote $\mathcal{M}_1=\left(\mathcal{S}, \mathcal{A}, P_1, r, \gamma\right)$ and $\mathcal{M}_2=\left(\mathcal{S}, \mathcal{A}, P_2, r, \gamma\right)$ as two MDPs that only differ in their transition dynamics. Then for any policy $\pi$, we have
\begin{align}
J_{\mathcal{M}_1}(\pi)-J_{\mathcal{M}_2}(\pi)=\frac{\gamma}{1-\gamma} \mathbb{E}_{\rho_{\mathcal{M}_1}^\pi(s, a)}\left[\mathbb{E}_{s^{\prime} \sim P_1}\left[V_{\mathcal{M}_2}^\pi\left(s^{\prime}\right)\right]-\mathbb{E}_{s^{\prime} \sim P_2}\left[V_{\mathcal{M}_2}^\pi\left(s^{\prime}\right)\right]\right] .\notag 
\end{align}
\label{lem:telescoping_lem}
\end{lemma}

\begin{lemma}
Denote $\mathcal{M}=(\mathcal{S}, \mathcal{A}, P, r, \gamma)$ as the underlying MDP. Suppose we have two policies $\pi_1, \pi_2$, then the performance difference of these policies in the MDP gives:
\begin{align}
J_{\mathcal{M}}\left(\pi_1\right)-J_{\mathcal{M}}\left(\pi_2\right)=\frac{1}{1-\gamma} \mathbb{E}_{\rho_{\mathcal{M}}^{\pi_1}(s, a), s^{\prime} \sim P}\left[\mathbb{E}_{a^{\prime} \sim \pi_1}\left[Q_{\mathcal{M}}^{\pi_2}\left(s^{\prime}, a^{\prime}\right)\right]-\mathbb{E}_{a^{\prime} \sim \pi_2}\left[Q_{\mathcal{M}}^{\pi_2}\left(s^{\prime}, a^{\prime}\right)\right]\right].\notag
\end{align}
\label{lem:extendeD_{tar}elescoping_lem}
\end{lemma}
\begin{proof}
This is Lemma B.3 in \citet{lyu2024crossdomainpolicyadaptationcapturing}, please check the proof there.
\end{proof}

\section{Experiment Details}
\label{app:experiment}

In this section, we provide more details for the experiment settings and hyperparameters in \Cref{sec:experiment}.

\subsection{Locomotion Tasks}
In this section, we demonstrate the detailed modification in the MuJoCo \texttt{xml} file to construct the target environment. In the locomotion tasks, we follow the ODRL, providing three dynamics shifts including gravity, friction and morphology shift. In the gravity and friction shift, we consider the shift level with $\{0.1, 0.5, 2.0, 5.0\}$. In the morphology shift, we consider the \textit{medium} and \textit{hard} level shift.

\paragraph{Gravity Shift.} Following the ODRL benchmark \citep{lyu2024odrlbenchmarkoffdynamicsreinforcement}, we modify the gravity of the environment by editing the gravity attribute. For example, the gravity of the HalfCheetah in the target is modified to 0.5 times the gravity in the source domain with the following code.
\begin{tcolorbox}[
  colback=black!2,
  colframe=black!30,
  boxrule=0.4pt,
  left=6pt,
  right=6pt,
  top=4pt,
  bottom=4pt
]
\begin{lstlisting}[style=xmlstyle]
# gravity
<option gravity="0 0 -4.905" timestep="0.01"/>
\end{lstlisting}
\end{tcolorbox}

\paragraph{Friction Shift.} The friction shift is generated by modifying the friction attribute in the
geom elements. The frictional components are adjusted to {0.1, 0.5, 2.0, 5.0} times the frictional
components in the source domain, respectively.

\paragraph{Morphology Shift.} 
The morphology shift is achieved by modifying the size of specific limbs or torsos of the simulated robot in MuJoCo without altering the state space and action space. Visualizations can be found in \Cref{fig:morph-ant}, \Cref{fig:morph-halfcheetah}, \Cref{fig:morph-hopper} and \Cref{fig:morph-walker2d}.

\begin{figure}[ht]
  \centering
  \setlength{\tabcolsep}{2pt}
  \renewcommand{\arraystretch}{0}
  \begin{subfigure}[b]{0.24\textwidth}
    \centering\includegraphics[width=\linewidth]{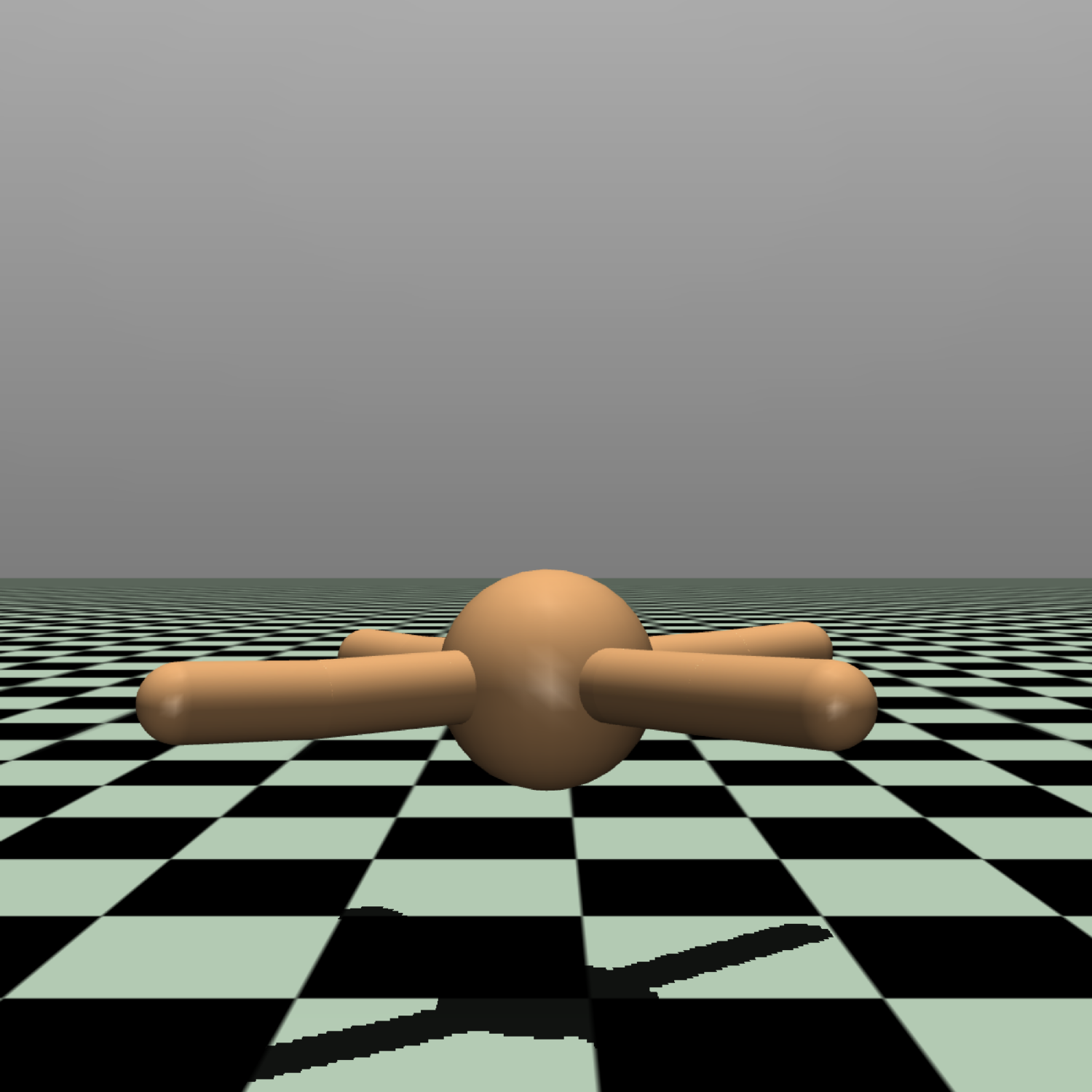}
    \caption{alllegs (M)}
  \end{subfigure}
  \begin{subfigure}[b]{0.24\textwidth}
    \centering\includegraphics[width=\linewidth]{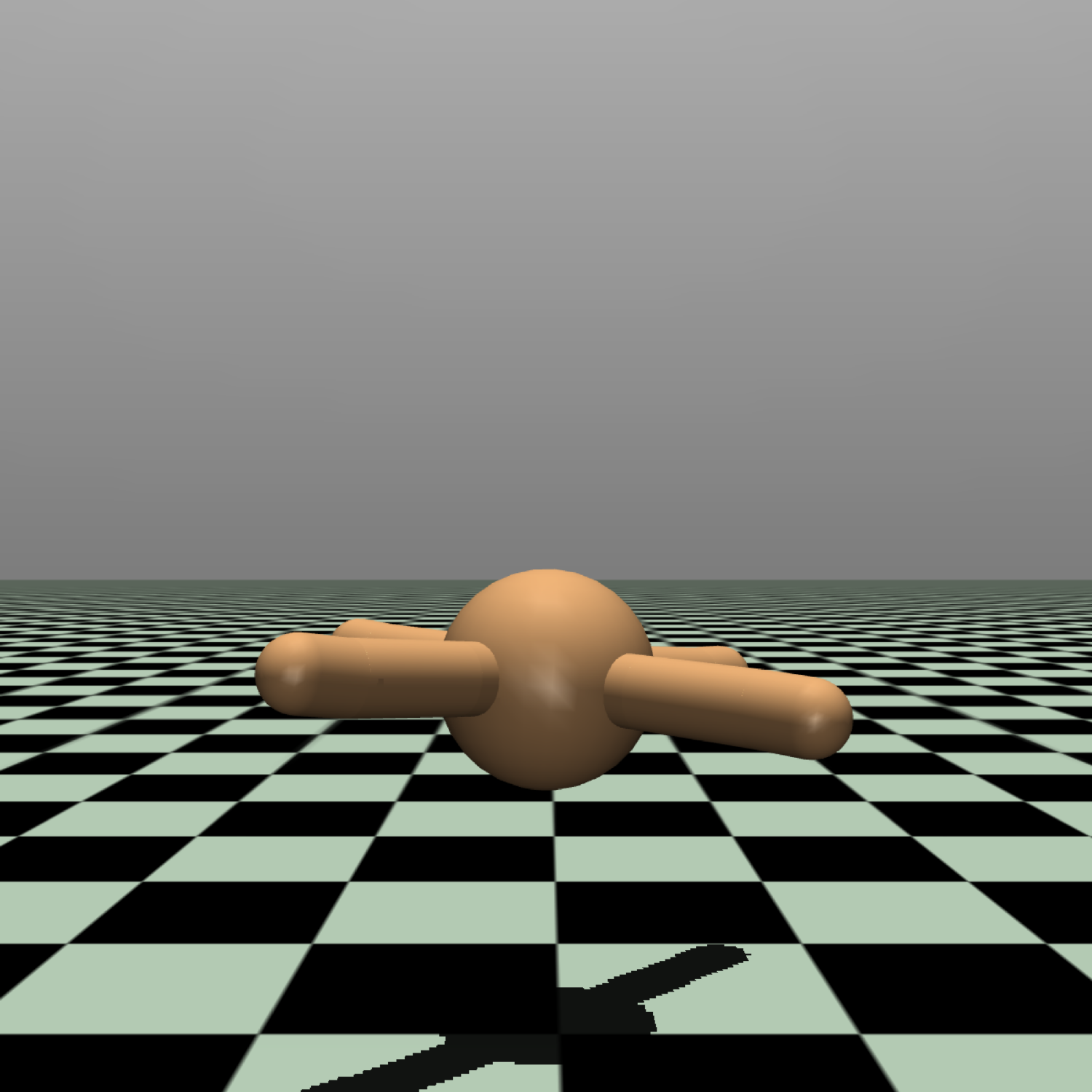}
    \caption{alllegs (H)}
  \end{subfigure}
  \begin{subfigure}[b]{0.24\textwidth}
    \centering\includegraphics[width=\linewidth]{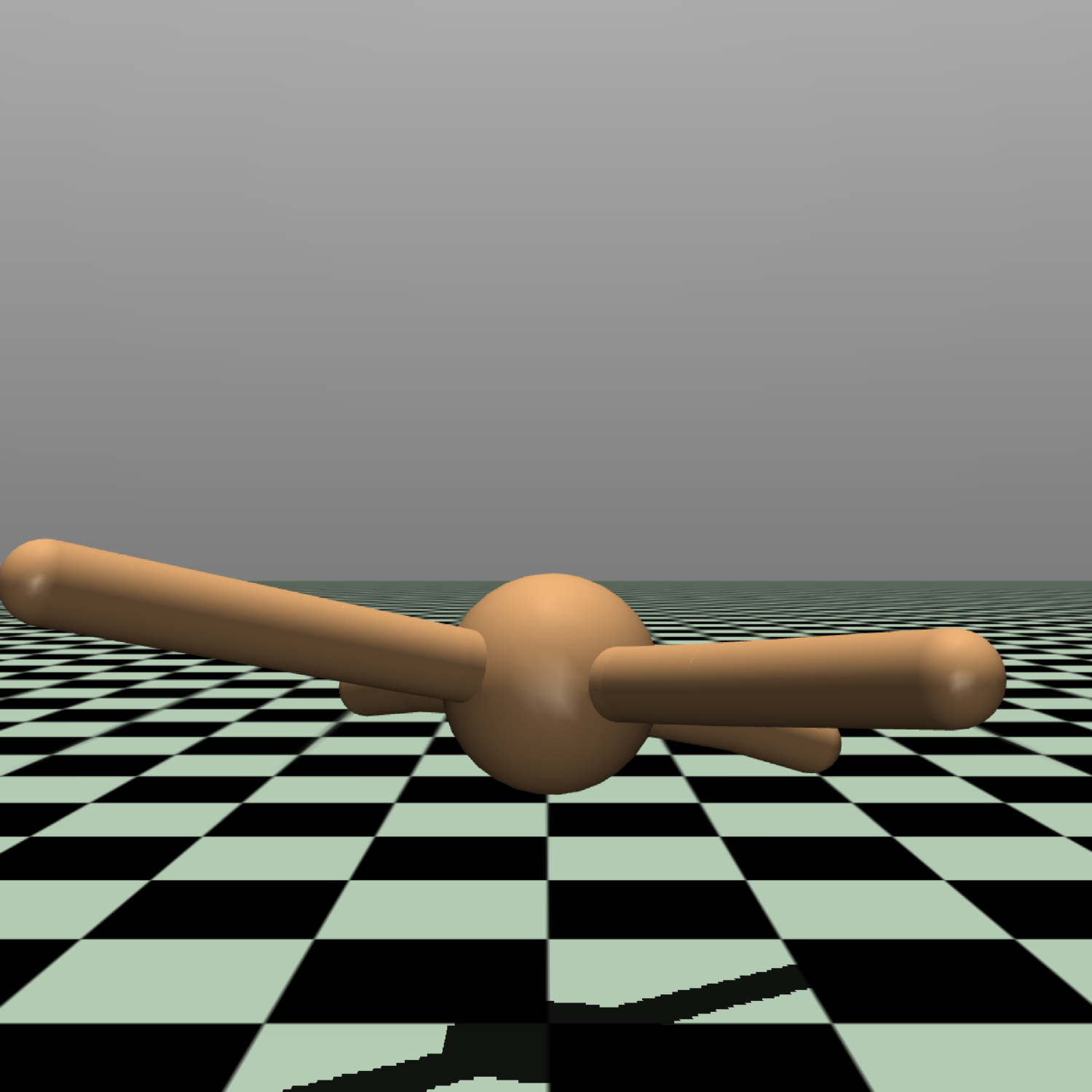}
    \caption{halflegs (M)}
  \end{subfigure}
  \begin{subfigure}[b]{0.24\textwidth}
    \centering\includegraphics[width=\linewidth]{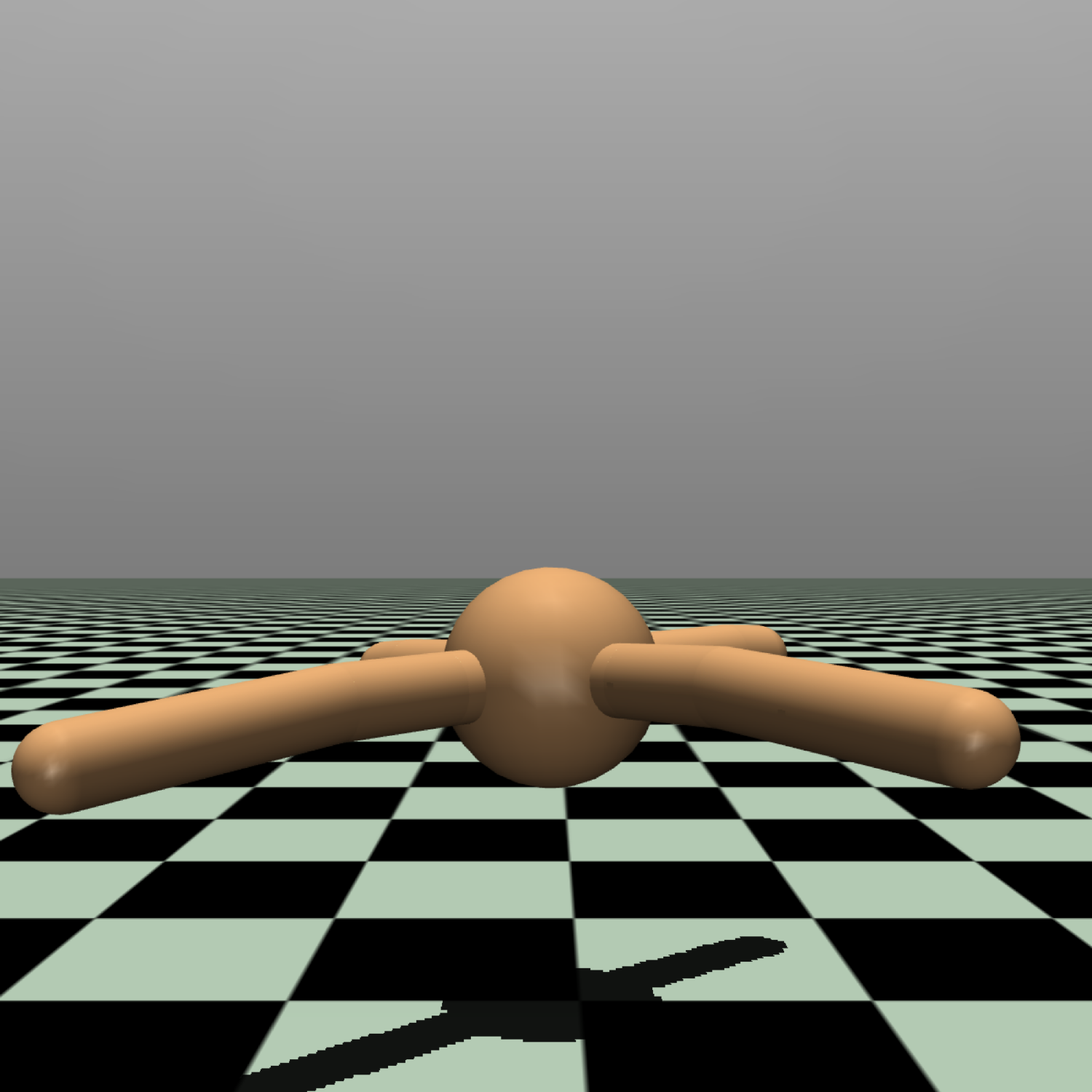}
    \caption{halflegs (H)}
  \end{subfigure}
  \caption{Visualization of morphology shift environments for Ant.}
  \label{fig:morph-ant}
\end{figure}

\begin{figure}[ht]
  \centering
  \setlength{\tabcolsep}{2pt}
  \renewcommand{\arraystretch}{0}
  \begin{subfigure}[b]{0.24\textwidth}
    \centering\includegraphics[width=\linewidth]{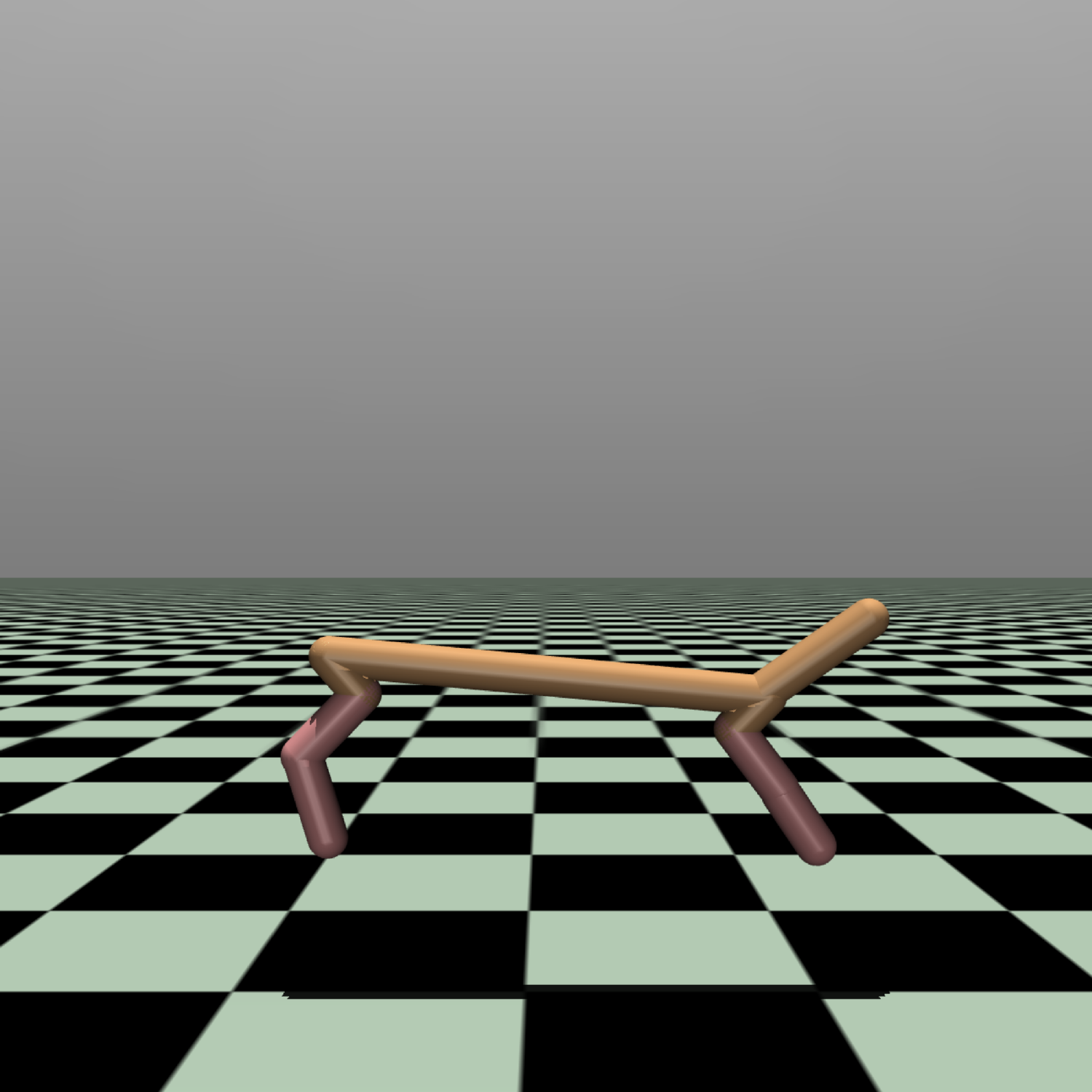}
    \caption{thigh (M)}
  \end{subfigure}
  \begin{subfigure}[b]{0.24\textwidth}
    \centering\includegraphics[width=\linewidth]{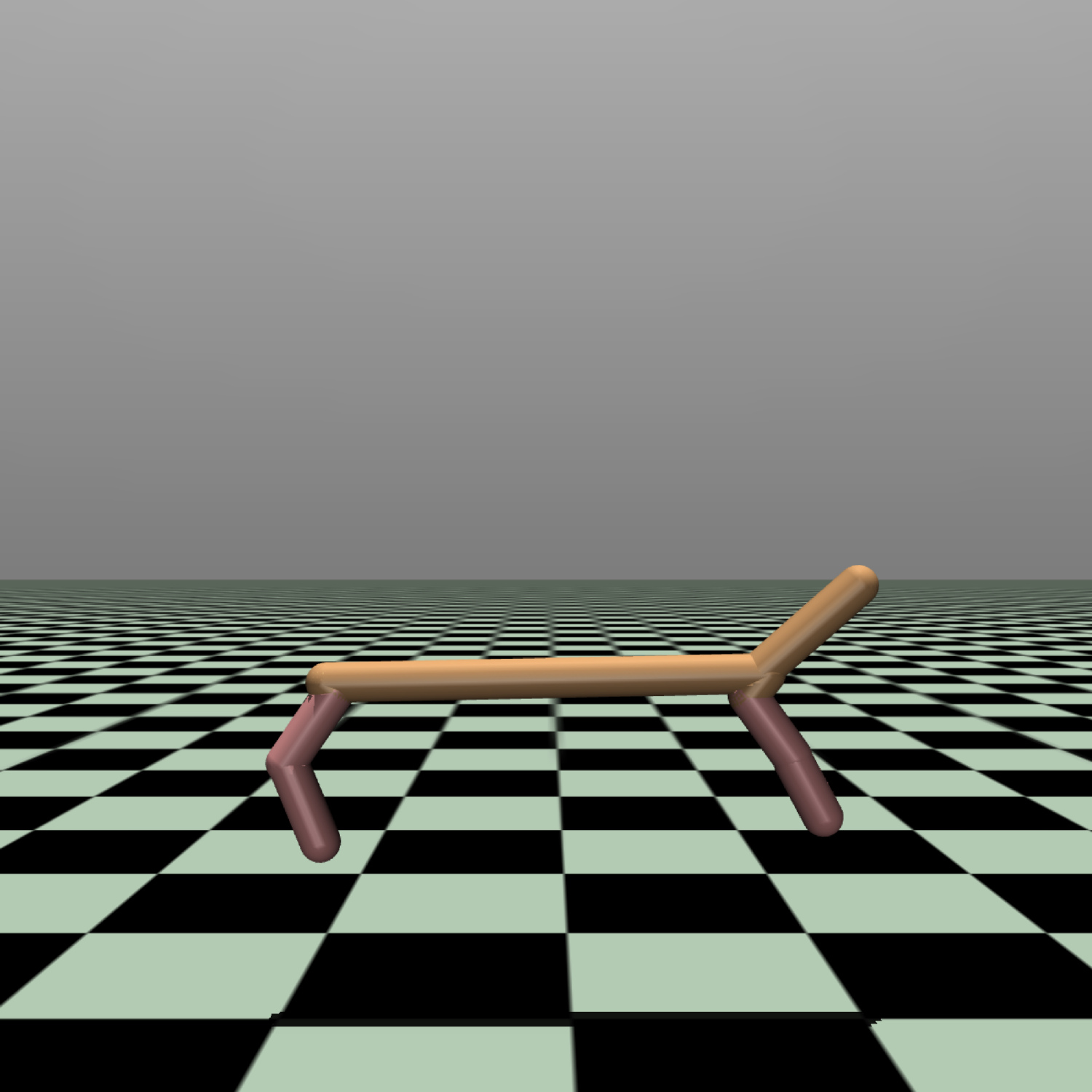}
    \caption{thigh (H)}
  \end{subfigure}
  \begin{subfigure}[b]{0.24\textwidth}
    \centering\includegraphics[width=\linewidth]{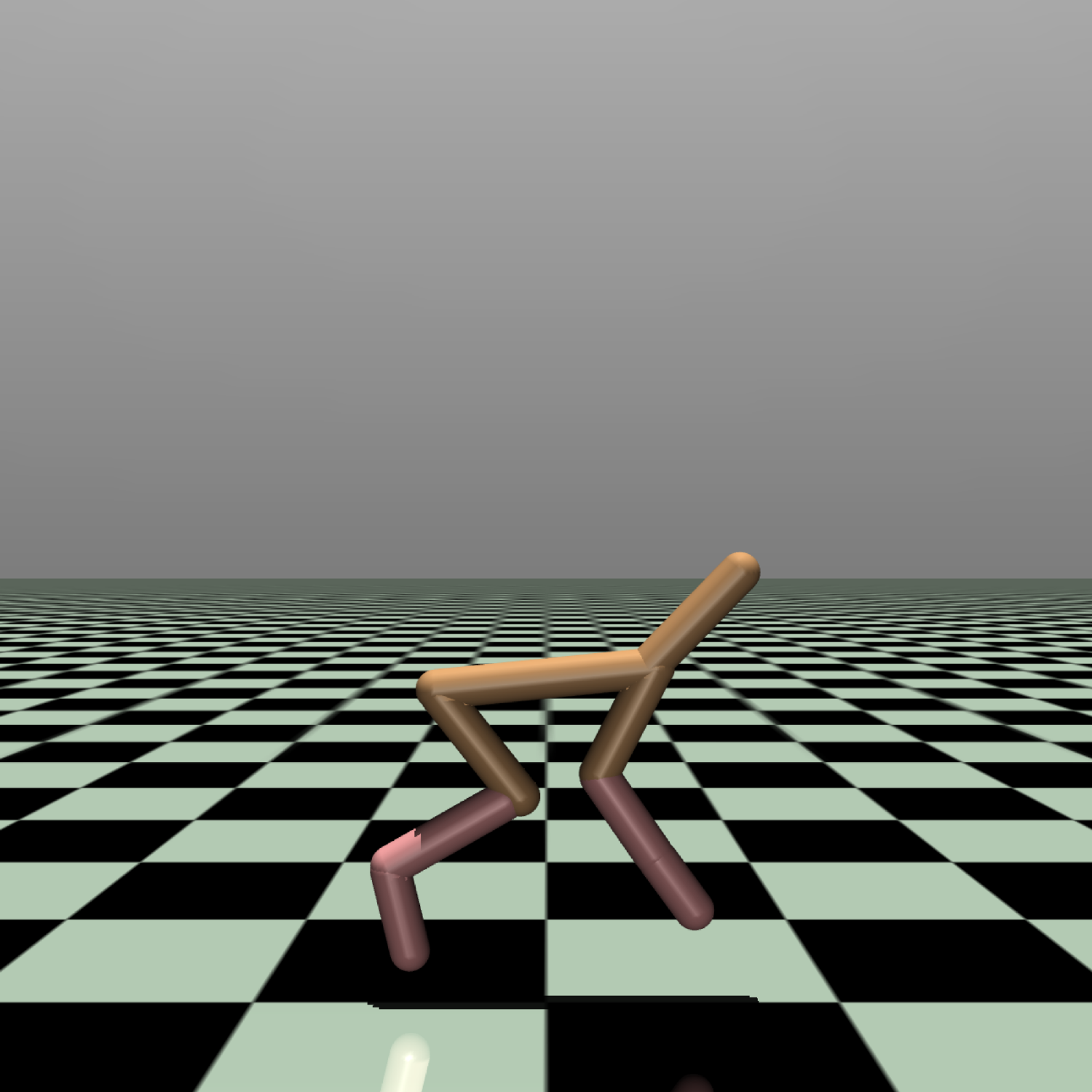}
    \caption{torso (M)}
  \end{subfigure}
  \begin{subfigure}[b]{0.24\textwidth}
    \centering\includegraphics[width=\linewidth]{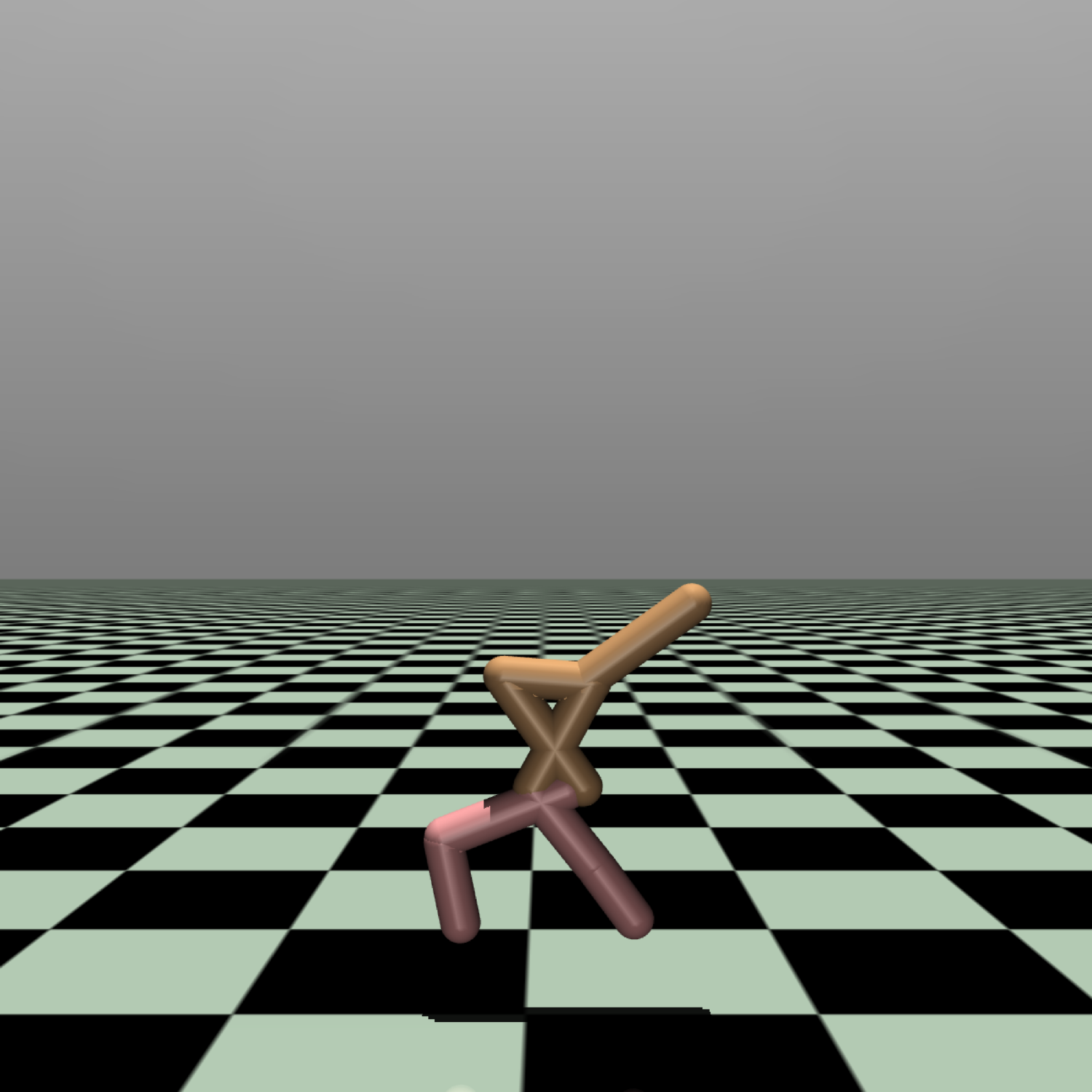}
    \caption{torso (H)}
  \end{subfigure}
  \caption{Visualization of morphology shift environments for HalfCheetah.}
  \label{fig:morph-halfcheetah}
\end{figure}

\begin{figure}[ht]
  \centering
  \setlength{\tabcolsep}{2pt}
  \renewcommand{\arraystretch}{0}
  \begin{subfigure}[b]{0.24\textwidth}
    \centering\includegraphics[width=\linewidth]{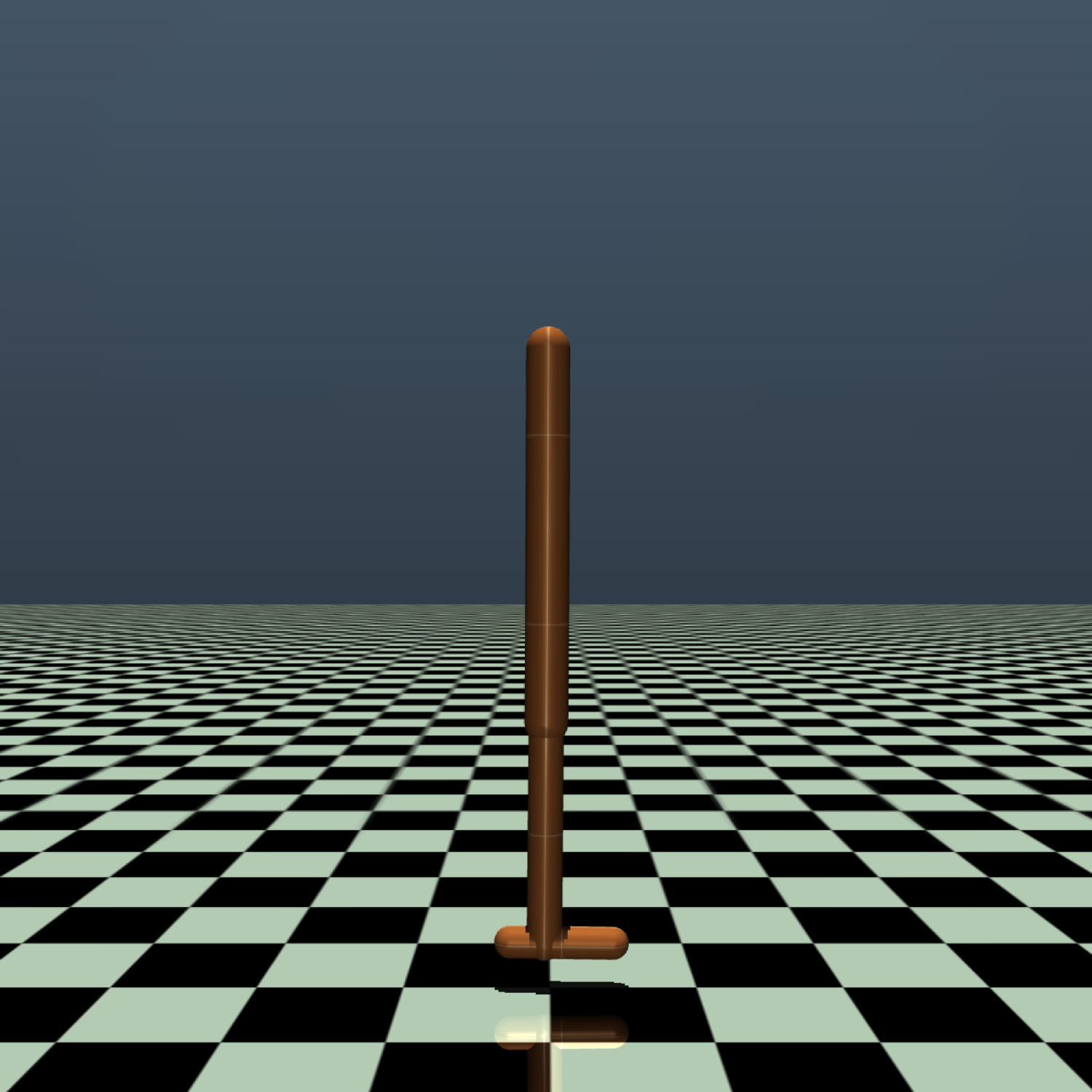}
    \caption{foot (M)}
  \end{subfigure}
  \begin{subfigure}[b]{0.24\textwidth}
    \centering\includegraphics[width=\linewidth]{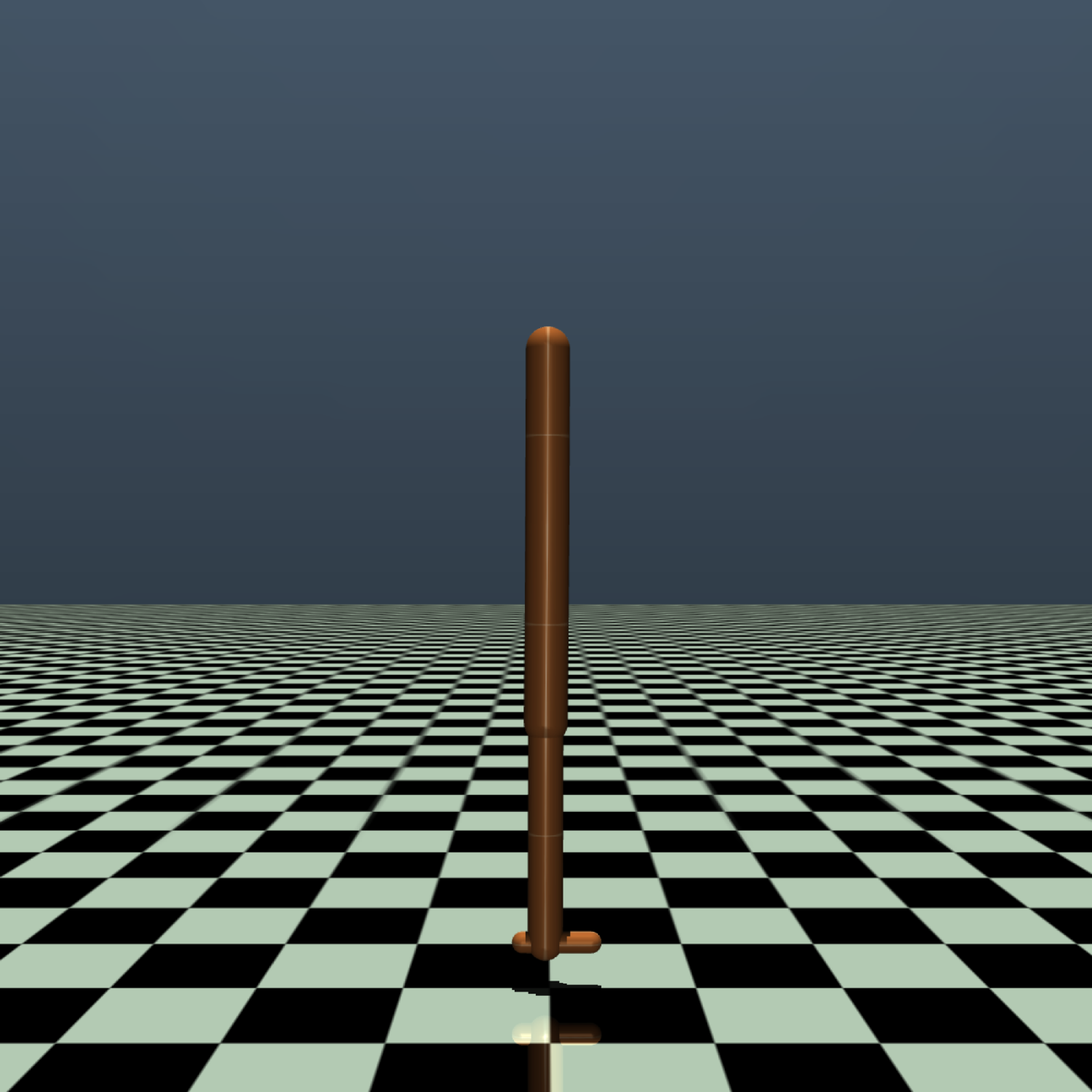}
    \caption{foot (H)}
  \end{subfigure}
  \begin{subfigure}[b]{0.24\textwidth}
    \centering\includegraphics[width=\linewidth]{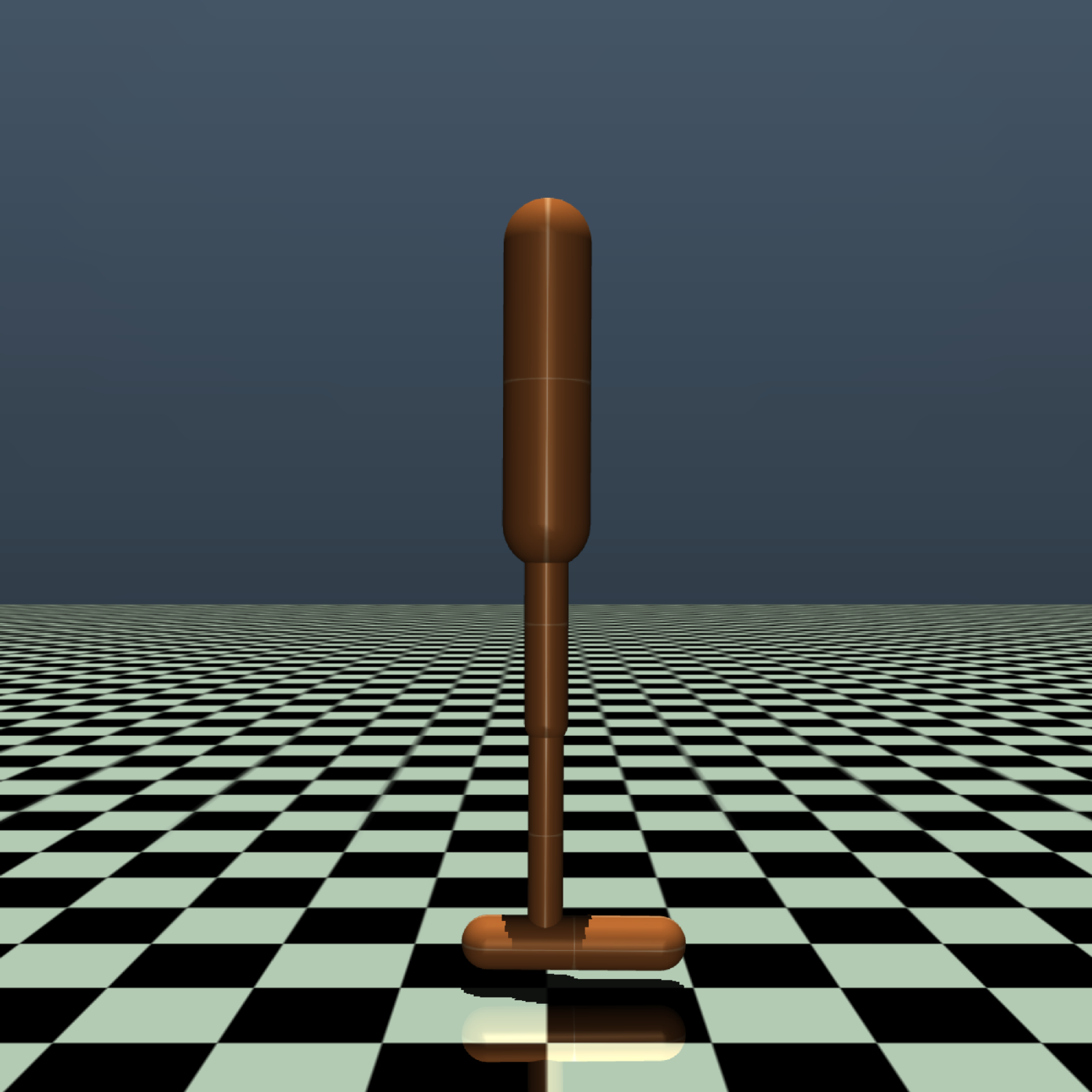}
    \caption{torso (M)}
  \end{subfigure}
  \begin{subfigure}[b]{0.24\textwidth}
    \centering\includegraphics[width=\linewidth]{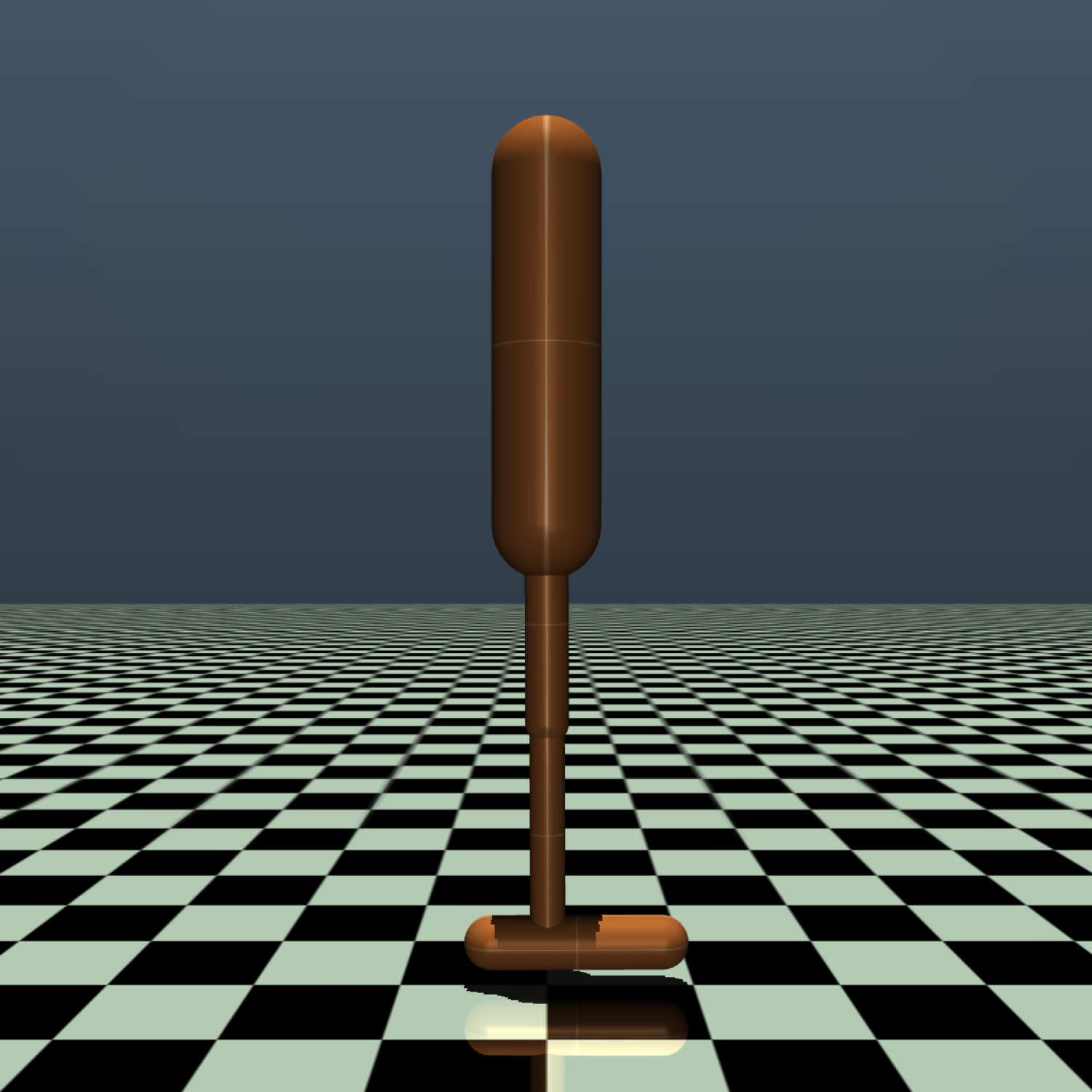}
    \caption{torso (H)}
  \end{subfigure}
  \caption{Visualization of morphology shift environments for Hopper.}
  \label{fig:morph-hopper}
\end{figure}

\begin{figure}[ht]
  \centering
  \setlength{\tabcolsep}{2pt}
  \renewcommand{\arraystretch}{0}
  \begin{subfigure}[b]{0.24\textwidth}
    \centering\includegraphics[width=\linewidth]{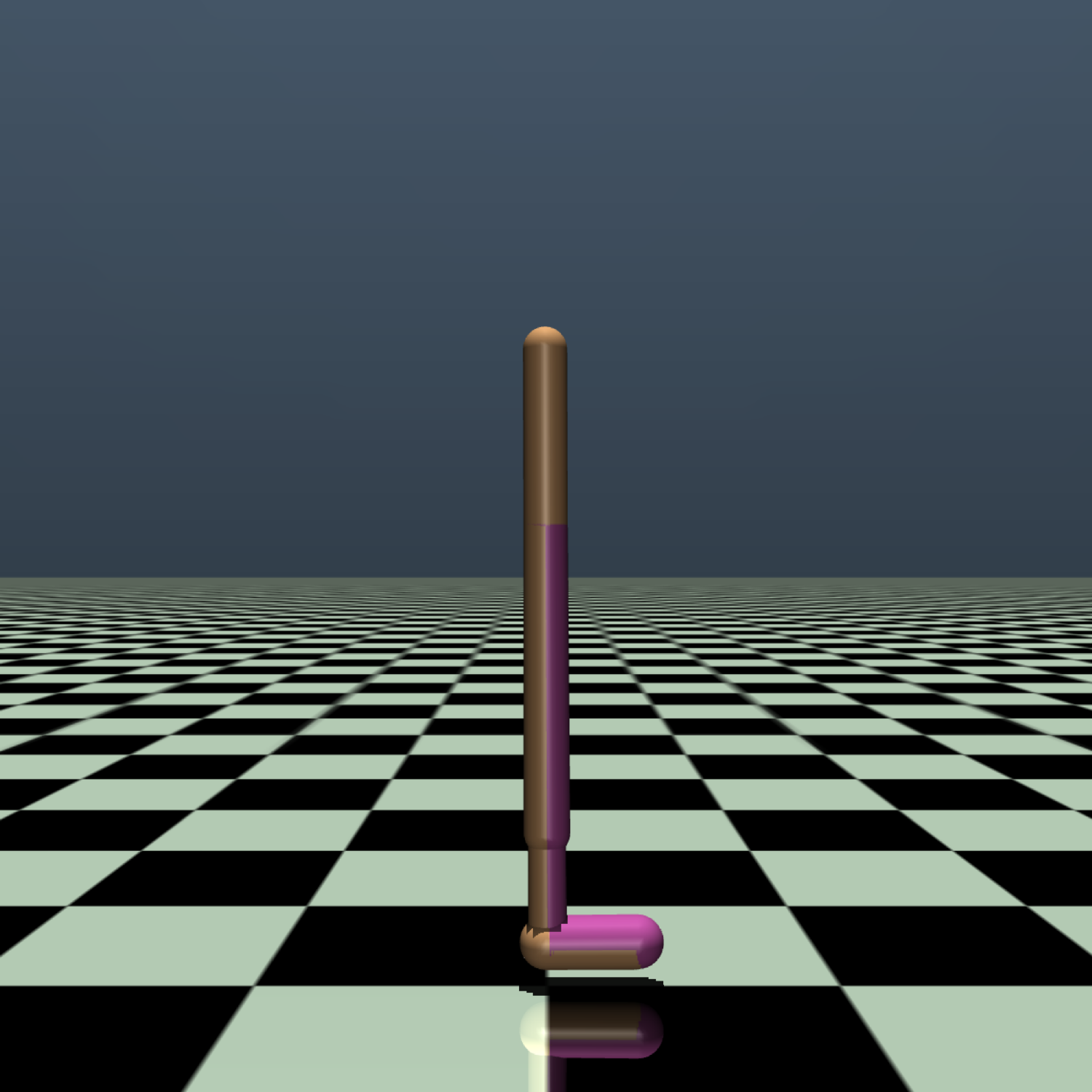}
    \caption{leg (M)}
  \end{subfigure}
  \begin{subfigure}[b]{0.24\textwidth}
    \centering\includegraphics[width=\linewidth]{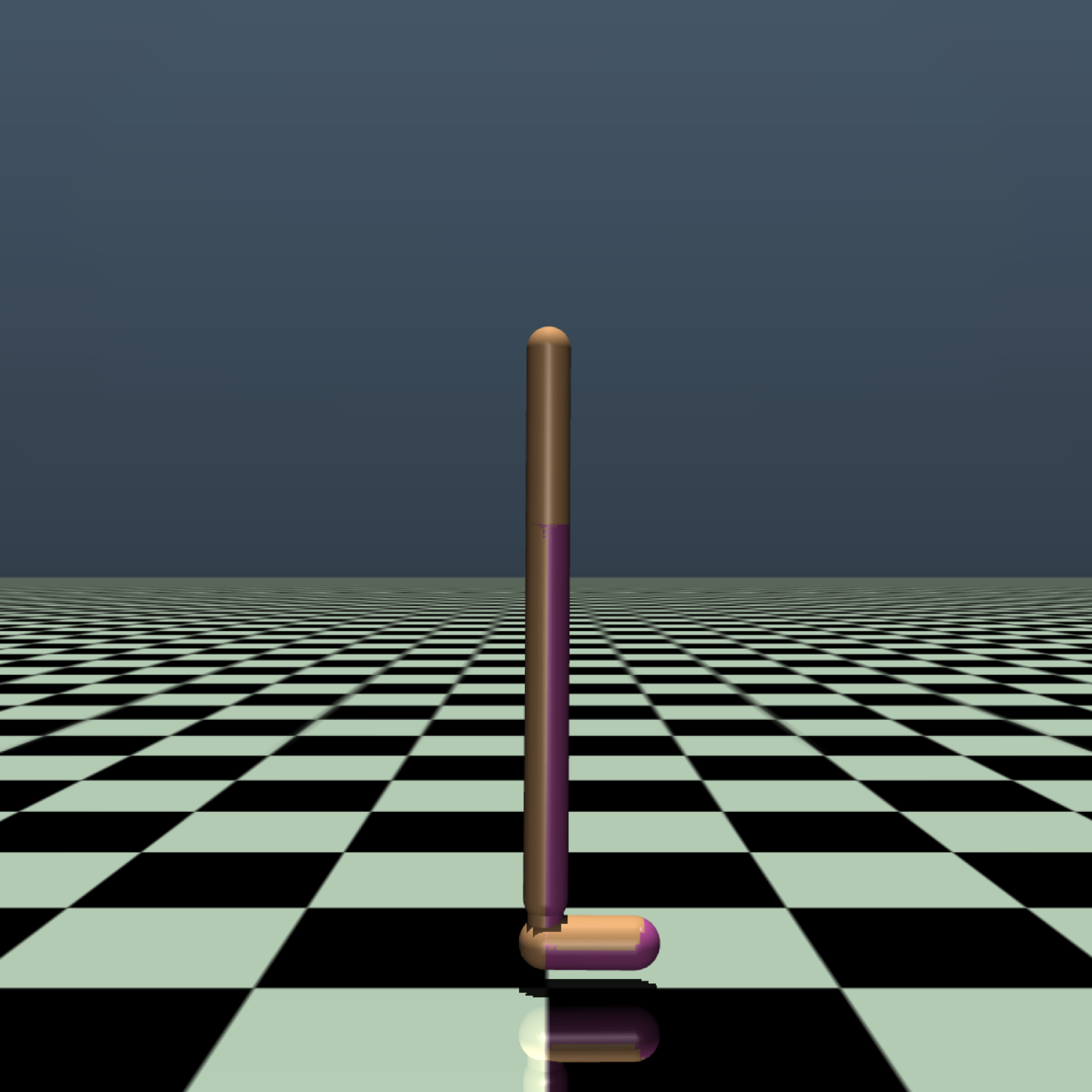}
    \caption{leg (H)}
  \end{subfigure}
  \begin{subfigure}[b]{0.24\textwidth}
    \centering\includegraphics[width=\linewidth]{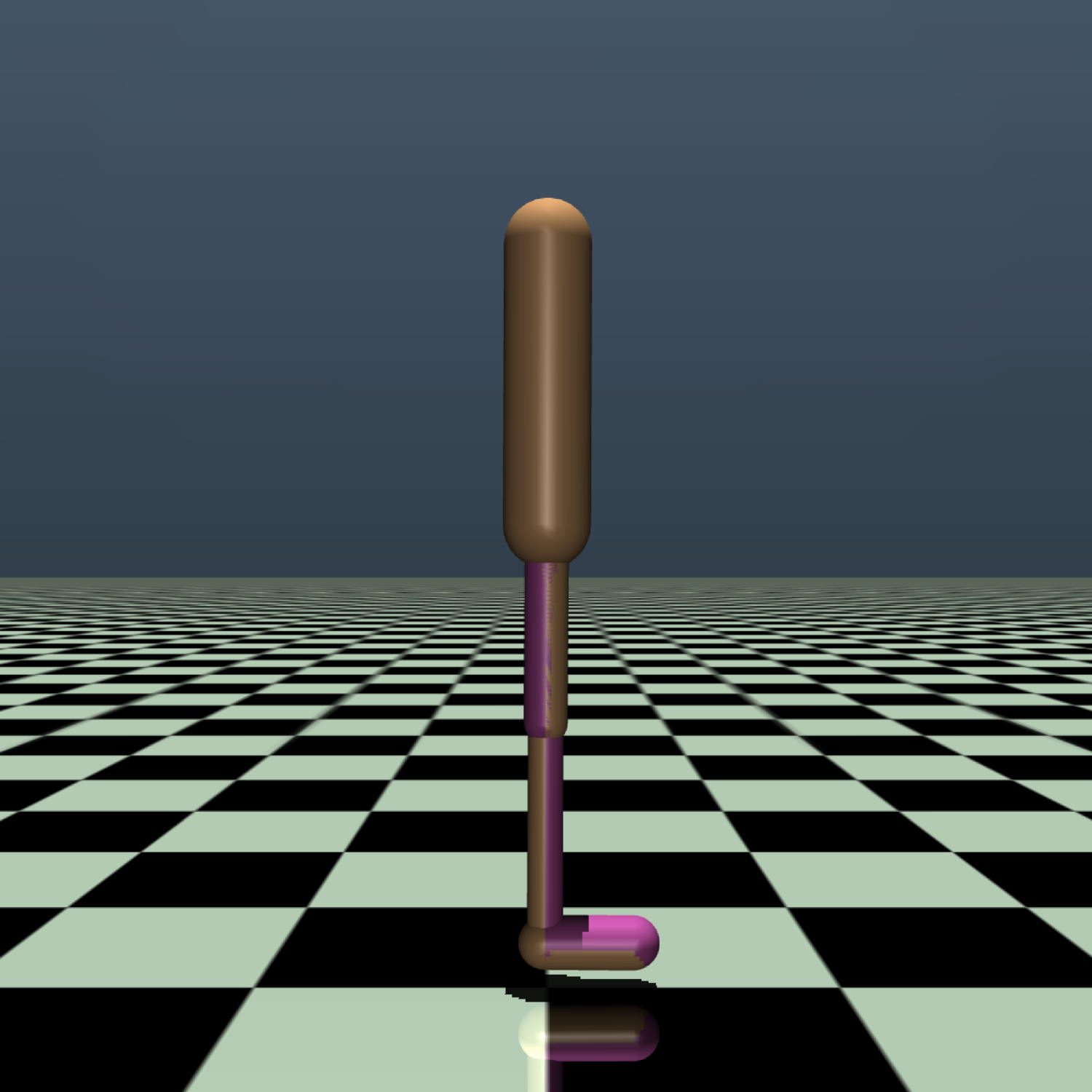}
    \caption{torso (M)}
  \end{subfigure}
  \begin{subfigure}[b]{0.24\textwidth}
    \centering\includegraphics[width=\linewidth]{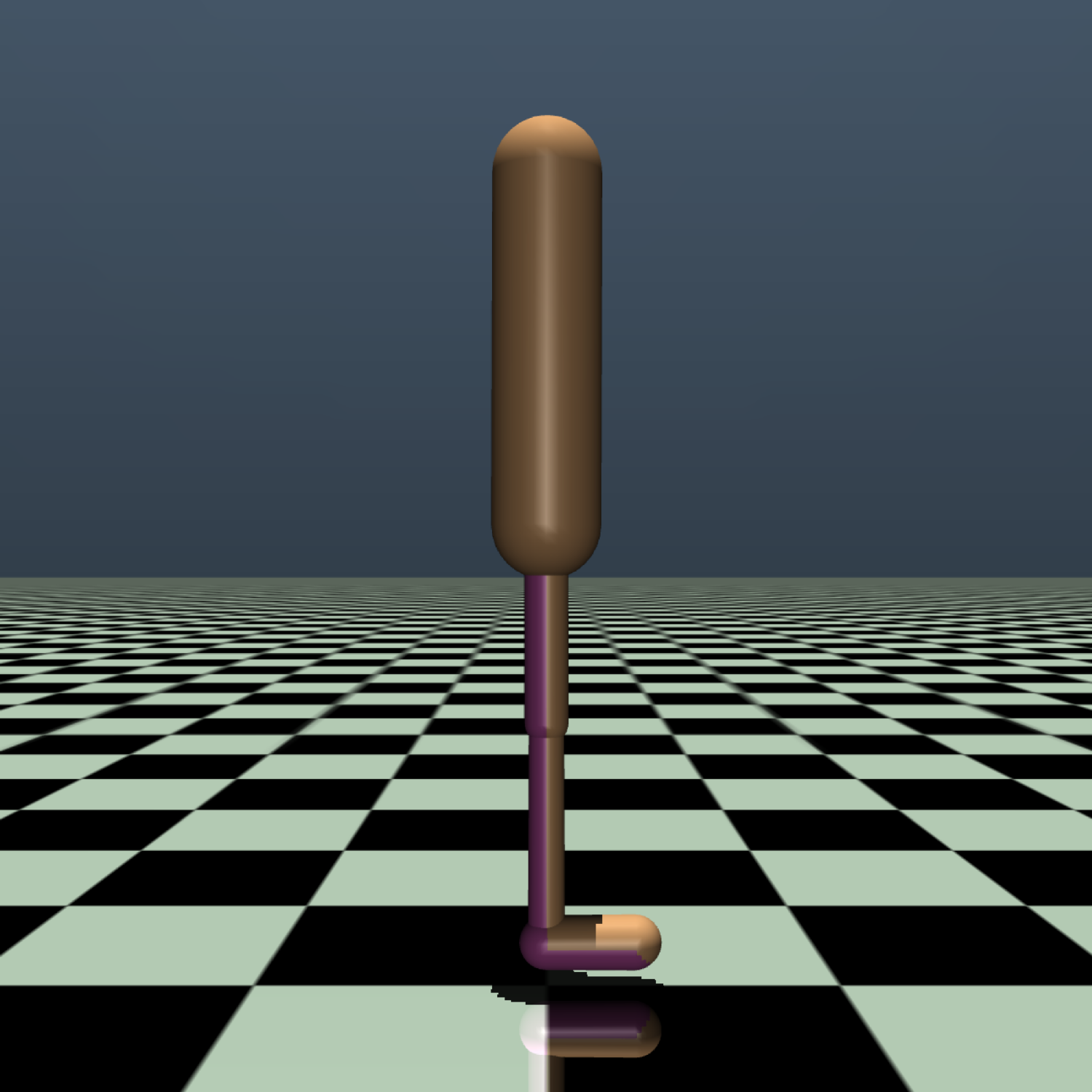}
    \caption{torso (H)}
  \end{subfigure}
  \caption{Visualization of morphology shift environments for Walker2d.}
  \label{fig:morph-walker2d}
\end{figure}

\subsection{Navigation Tasks}

In this section, we present the detailed modification in the AntMaze map to construct the target environment. In this setting, we follow the ODRL benchmark and consider the medium size with six different target maps. Visualization of the map can be found in \Cref{fig:antmaze-medium}.

\begin{figure}[ht]
  \centering
  \setlength{\tabcolsep}{2pt}
  \renewcommand{\arraystretch}{0}
    \begin{subfigure}[b]{0.32\textwidth}
    \centering\includegraphics[width=\linewidth]{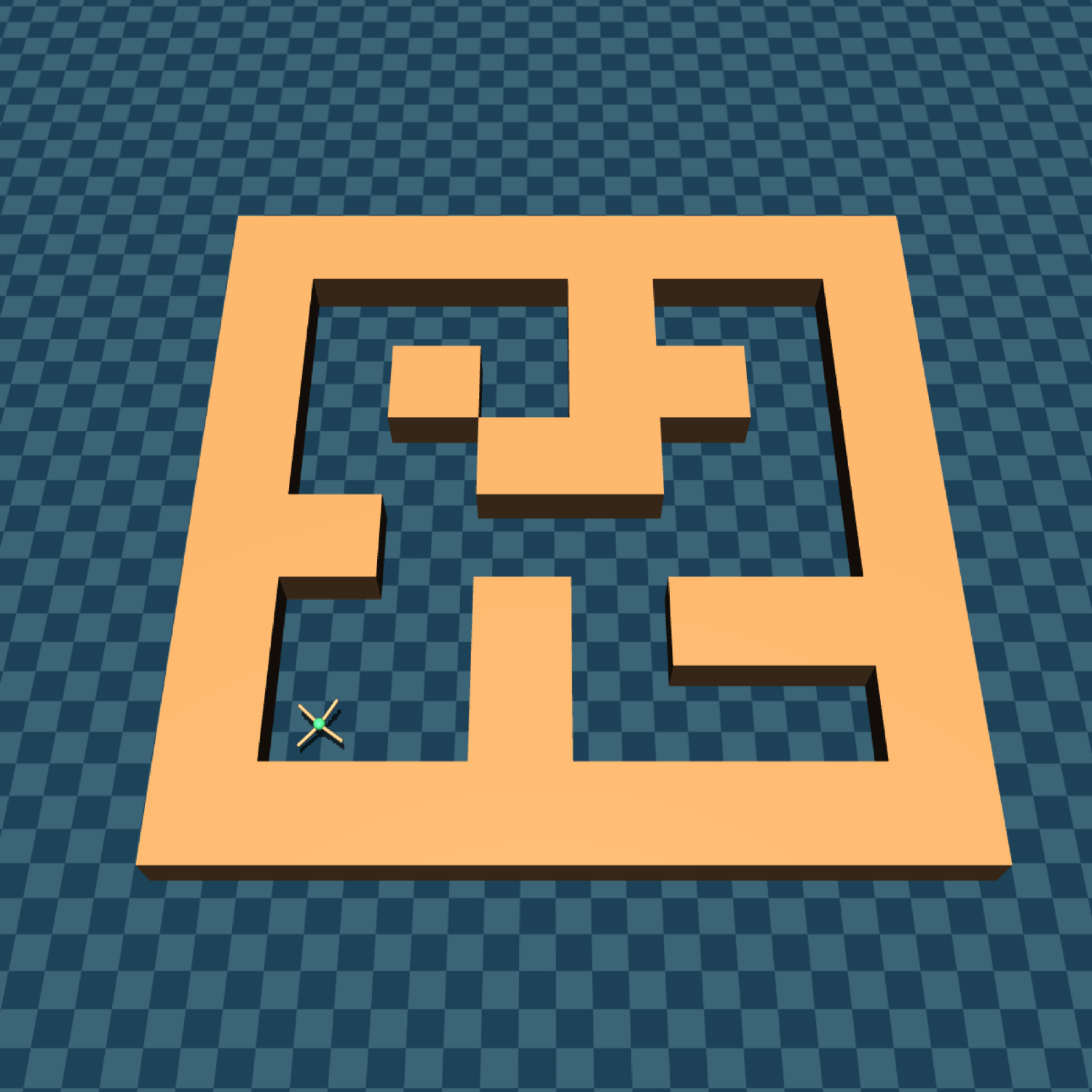}
    \caption{antmaze-medium-1}
  \end{subfigure}
  \begin{subfigure}[b]{0.32\textwidth}
    \centering\includegraphics[width=\linewidth]{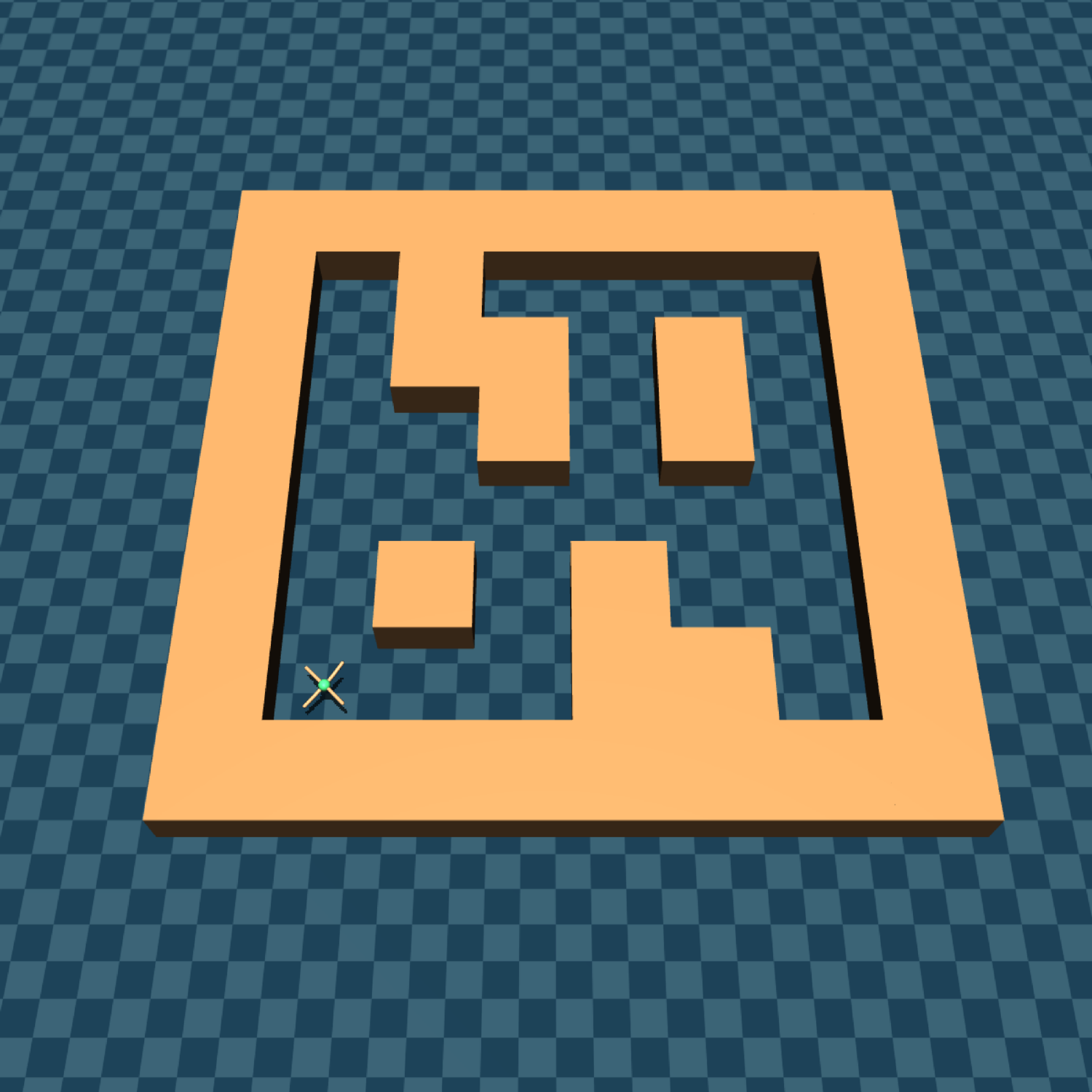}
    \caption{antmaze-medium-2}
  \end{subfigure}
  \begin{subfigure}[b]{0.32\textwidth}
    \centering\includegraphics[width=\linewidth]{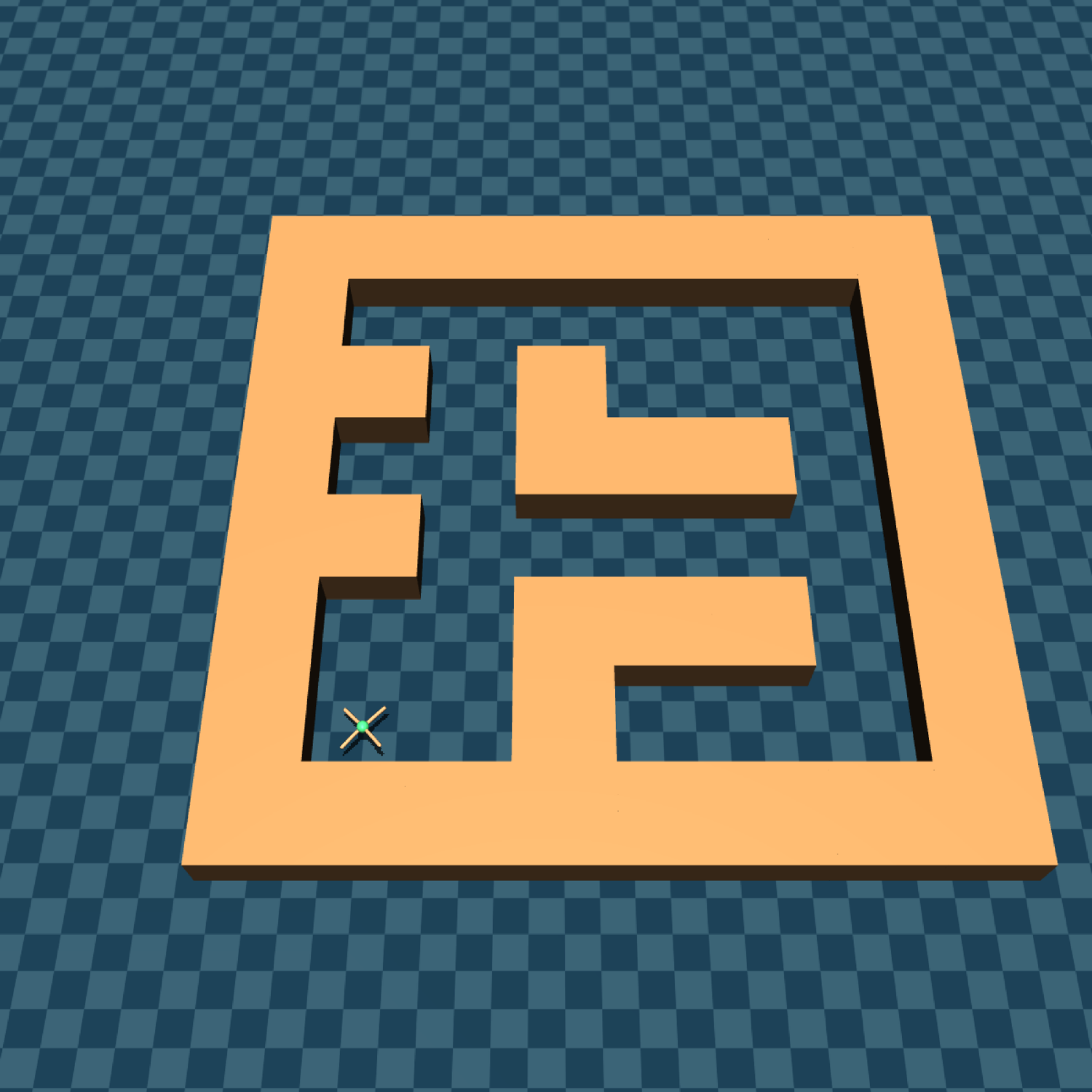}
    \caption{antmaze-medium-3}
  \end{subfigure}
  \begin{subfigure}[b]{0.32\textwidth}
    \centering\includegraphics[width=\linewidth]{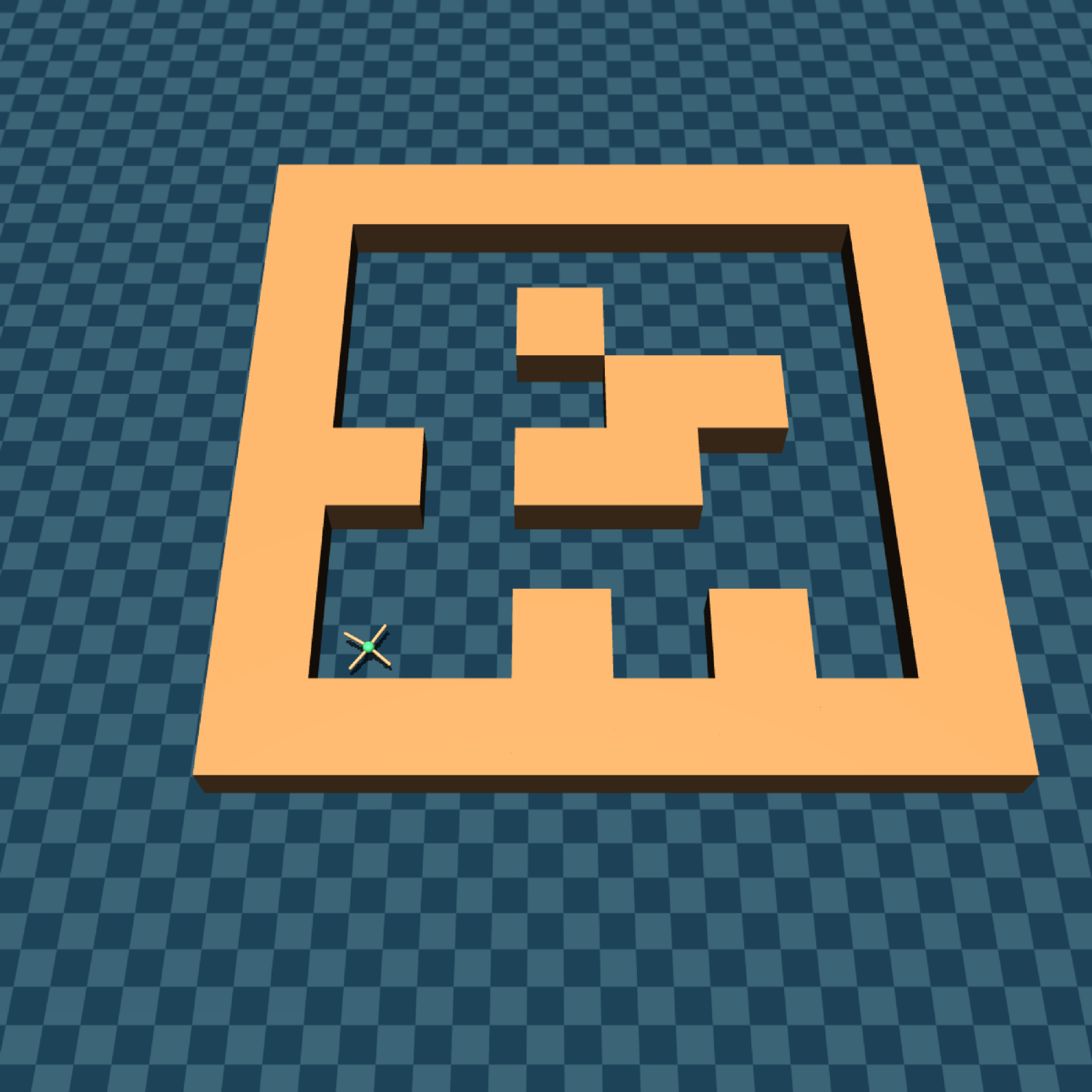}
    \caption{antmaze-medium-4}
  \end{subfigure}
  \begin{subfigure}[b]{0.32\textwidth}
    \centering\includegraphics[width=\linewidth]{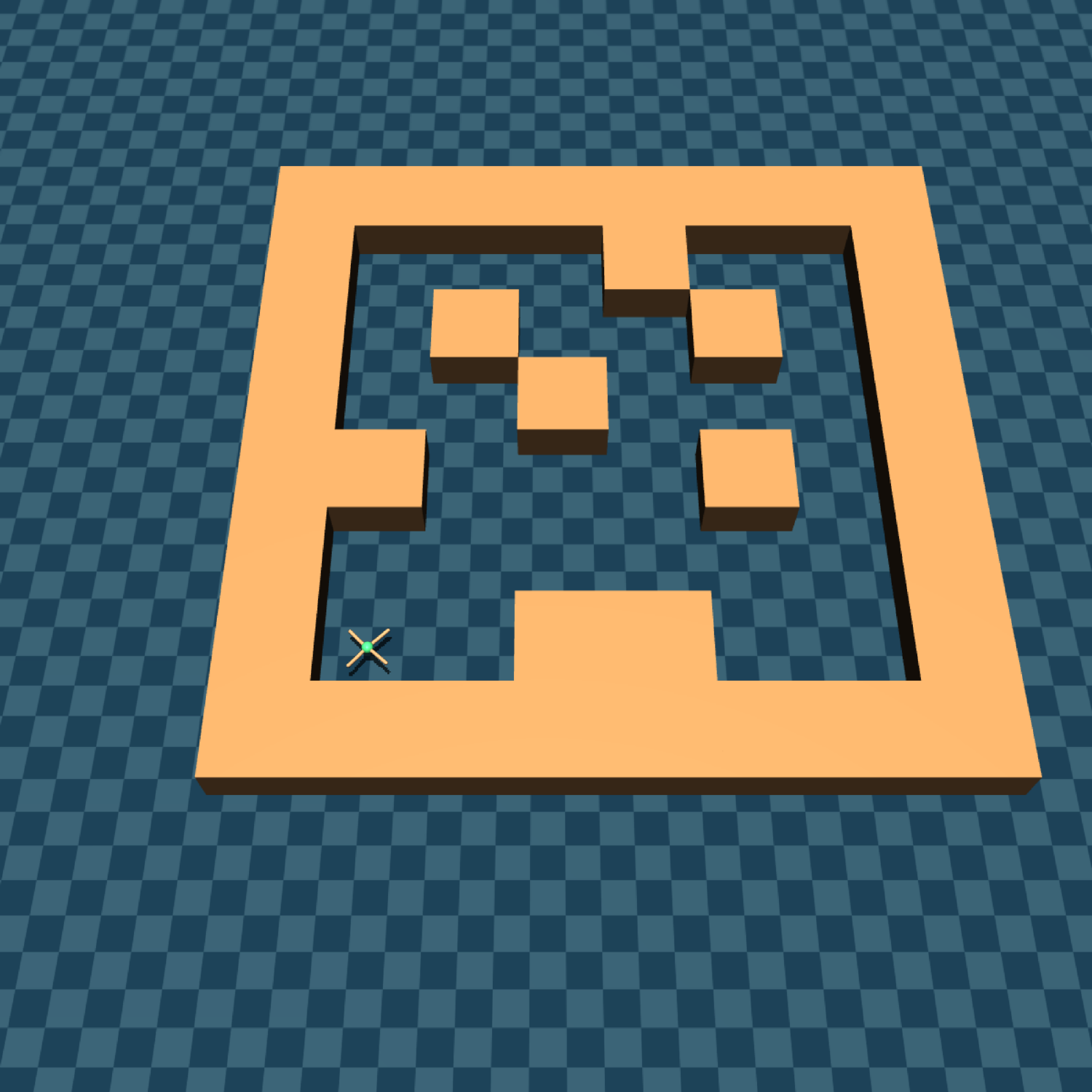}
    \caption{antmaze-medium-5}
  \end{subfigure}
  \begin{subfigure}[b]{0.32\textwidth}
    \centering\includegraphics[width=\linewidth]{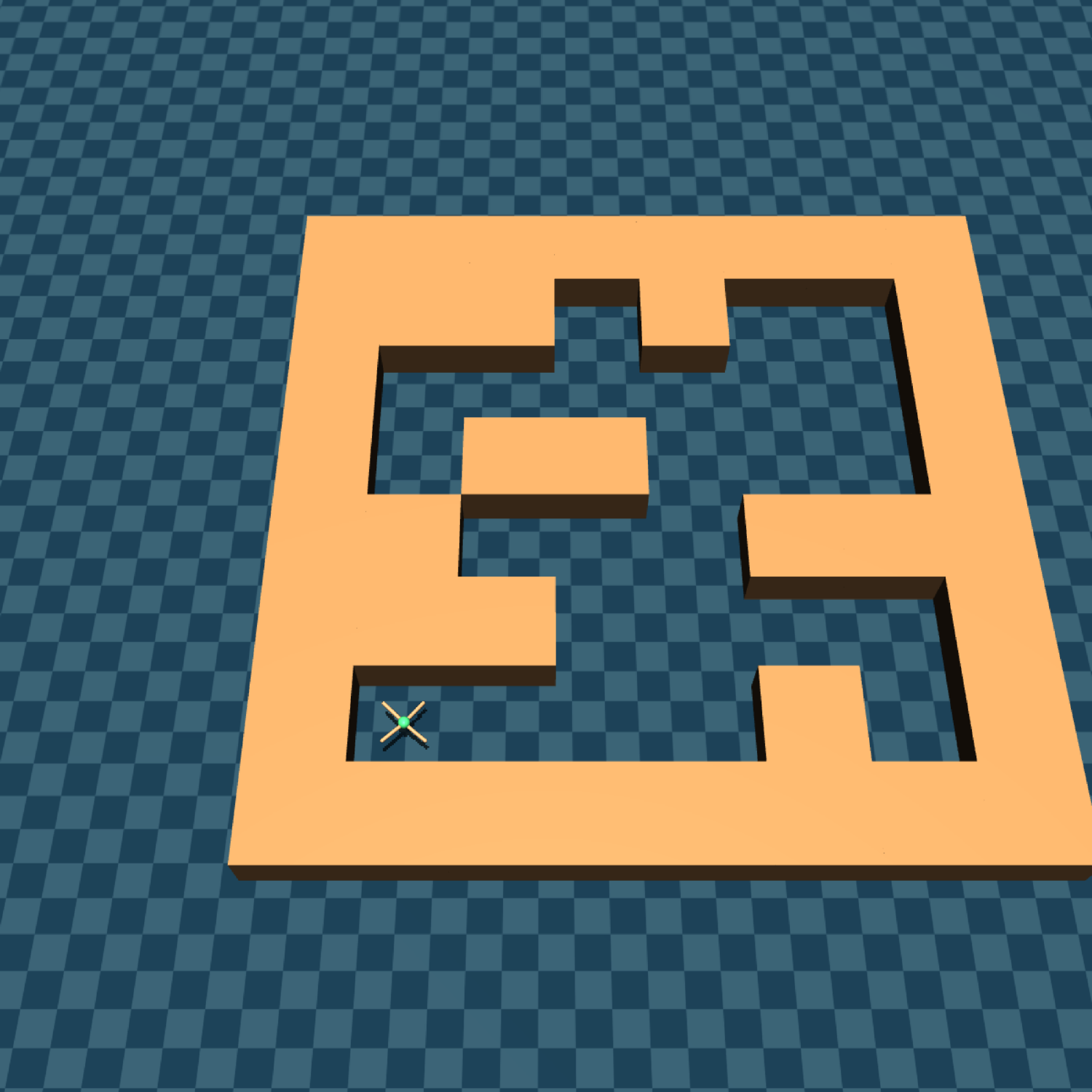}
    \caption{antmaze-medium-6}
  \end{subfigure}
  \caption{Visualization of the AntMaze map structure. We consider medium size map with 6 different map layouts. %
  } \label{fig:antmaze-medium}
\end{figure}

\subsection{Manipulation Tasks}
In this section, we discuss the detailed modifications in the Adroit environment. In this setting, we follow the ODRL benchmark using the \textit{pen} and \textit{door} as the tasks. The \textit{pen} task is to control the 24-DoF shadow hand robot to twirl a pen, and the \textit{door} task is to open a door. We consider the two types of dynamics shifts under this environment, denoted as the kinematic shift and the morphology task. We only consider the \textit{medium} and \textit{hard} level shift. Visualizations can be found in \Cref{fig:adroit-pen-morph} and \Cref{fig:adroit-door-morph}.

\paragraph{Kinematic Shift.} The kinematic shift in the Adroit robot hand occurs at all finger joints of the index finger and the thumb. We make constraints on the range of torsos. We use two types of shift levels, denoted as medium and hard level for the kinematic shift.

\paragraph{Morphology Shift.} The morphology shift in the Adroit robot hand occurs on the fingers. We shrink the sizes of these fingers with the medium and hard levels, which multiply by 0.25 and 0.125 for the size of the fingers, respectively.

\begin{figure}[ht]
  \centering
  \setlength{\tabcolsep}{2pt}
  \renewcommand{\arraystretch}{0}
  \begin{subfigure}[b]{0.24\textwidth}
    \centering\includegraphics[width=\linewidth]{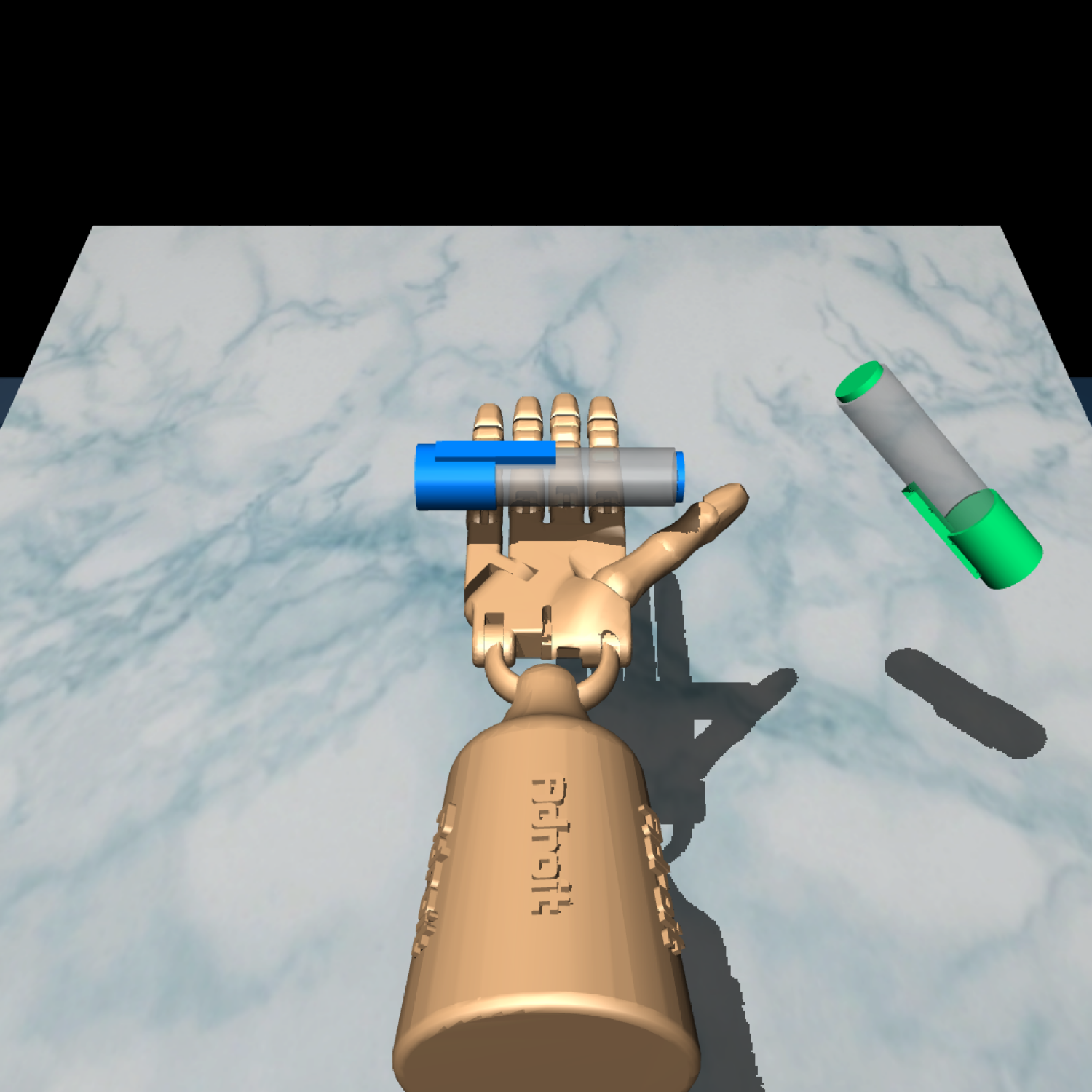}
    \caption{broken-joint (M)}
  \end{subfigure}
  \begin{subfigure}[b]{0.24\textwidth}
    \centering\includegraphics[width=\linewidth]{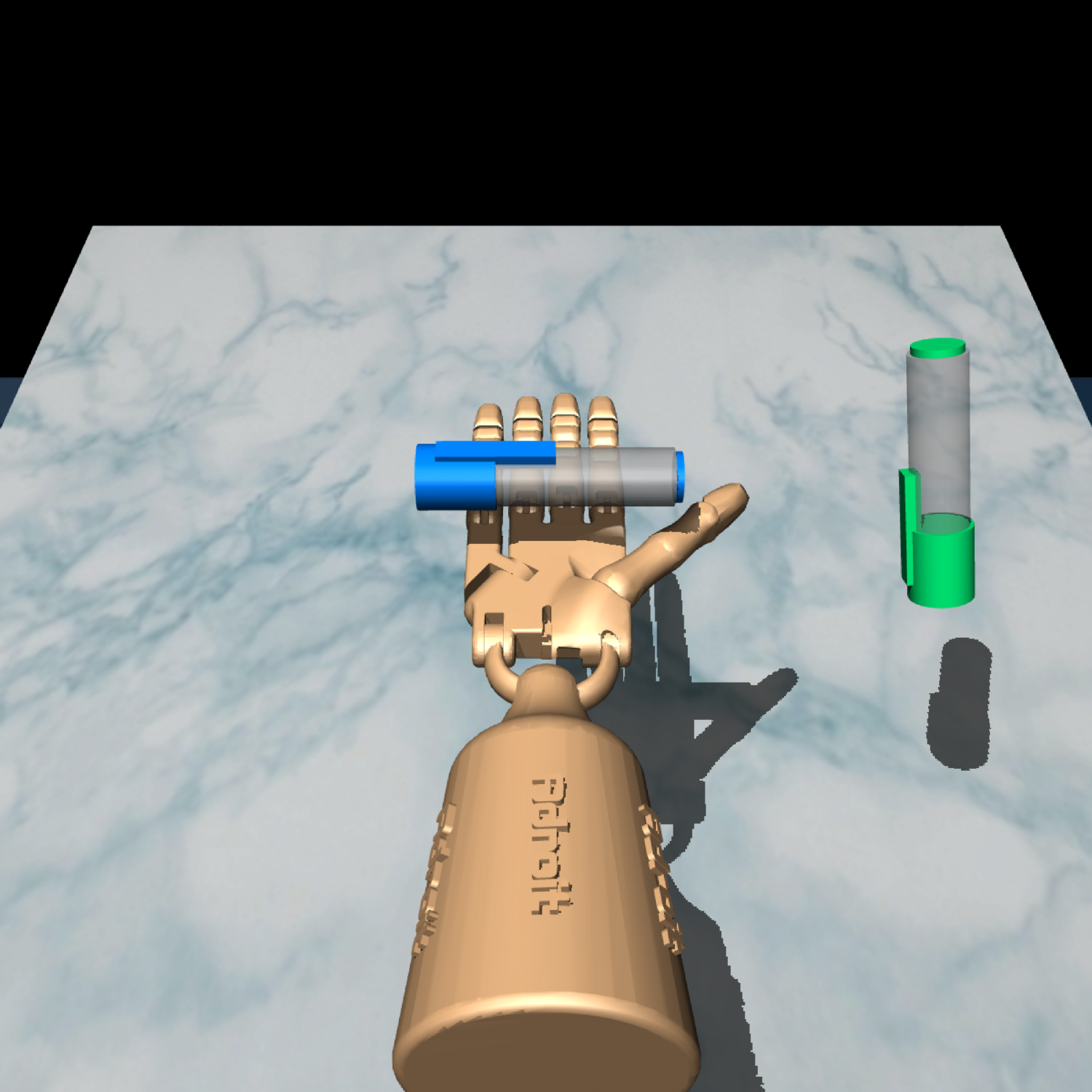}
    \caption{broken-joint (H)}
  \end{subfigure}
  \begin{subfigure}[b]{0.24\textwidth}
    \centering\includegraphics[width=\linewidth]{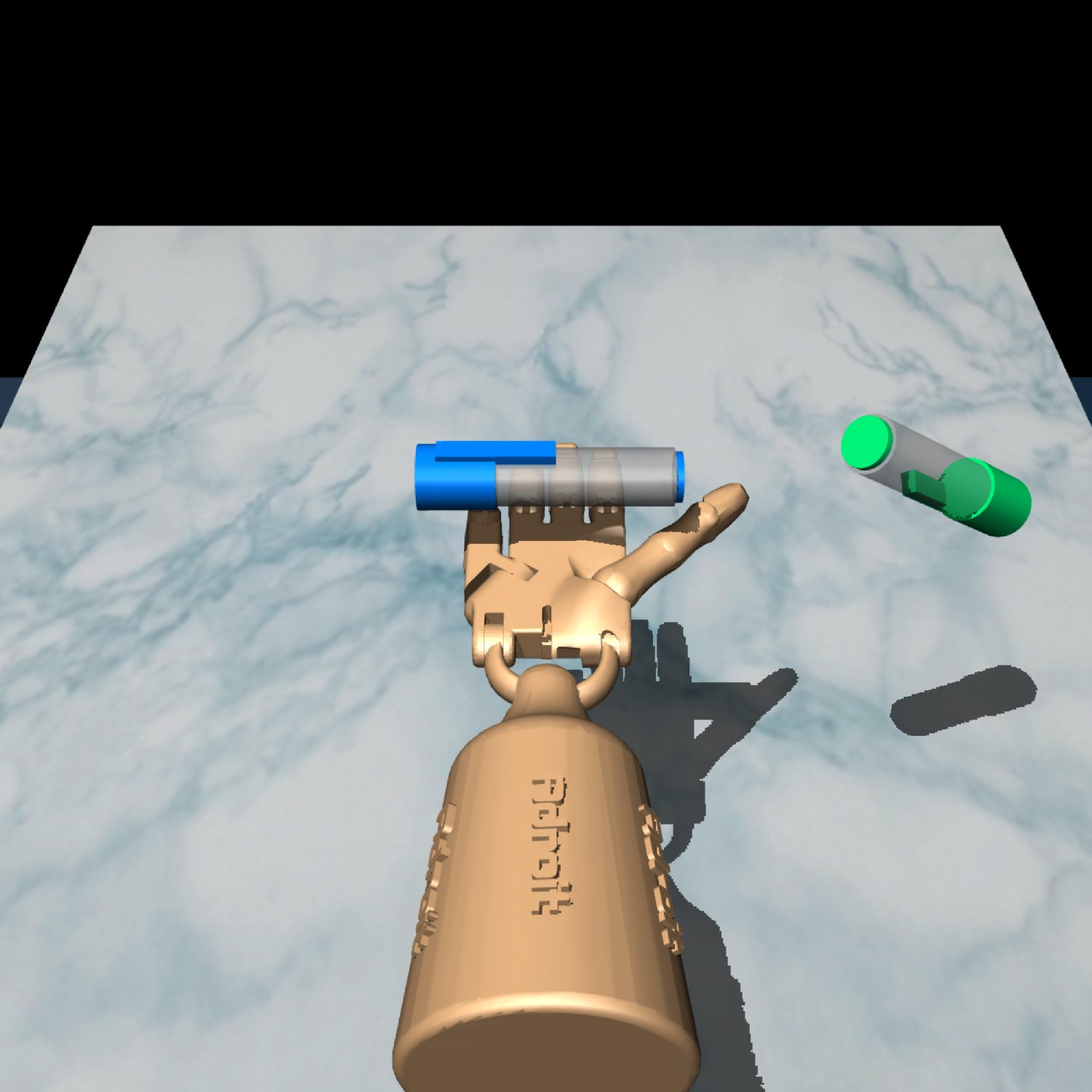}
    \caption{shrink-finger (M)}
  \end{subfigure}
  \begin{subfigure}[b]{0.24\textwidth}
    \centering\includegraphics[width=\linewidth]{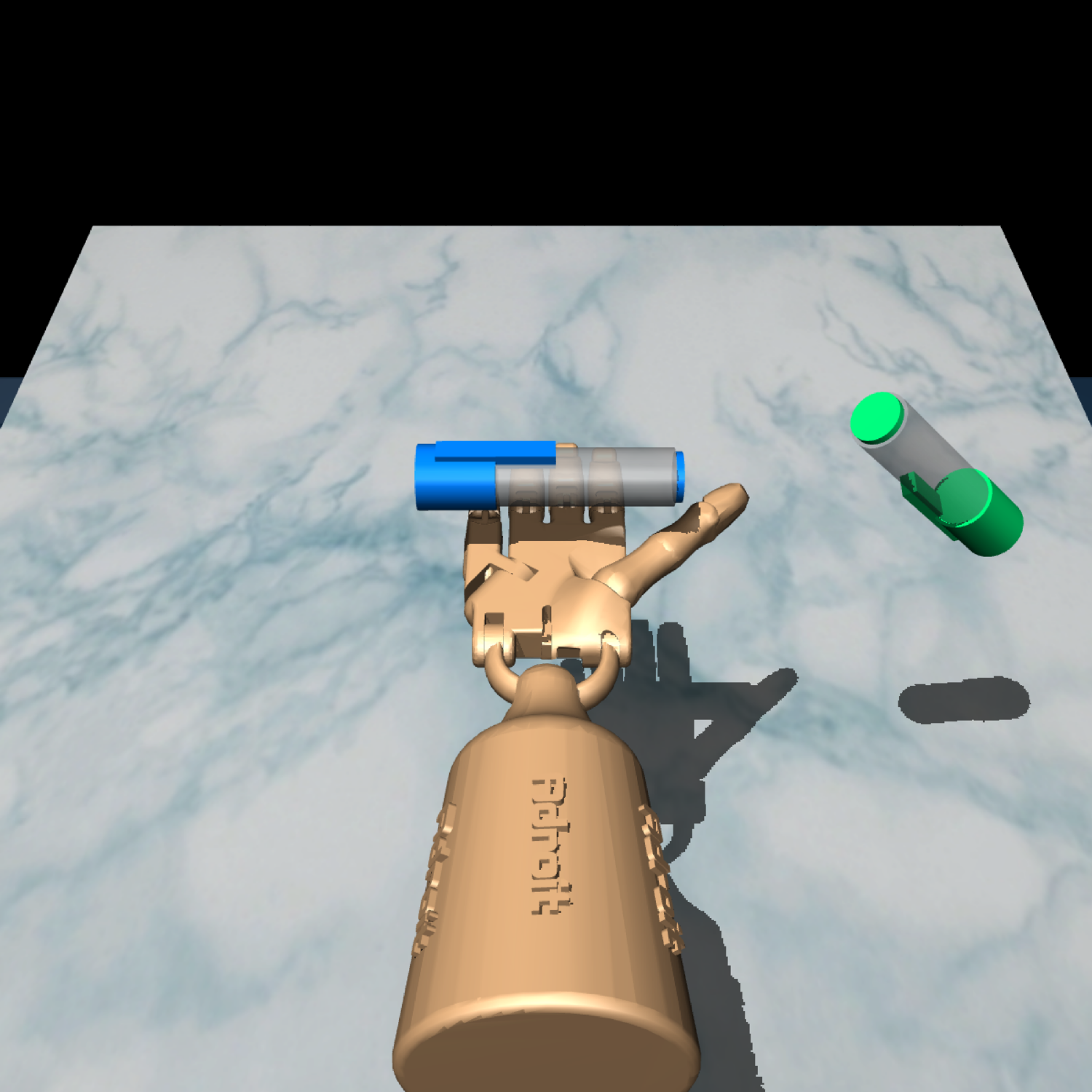}
    \caption{shrink-finger (H)}
  \end{subfigure}
  \caption{Visualization of morphology and kinematic shift environments for Adroit Pen.}
  \label{fig:adroit-pen-morph}
\end{figure}

\begin{figure}[ht]
  \centering
  \setlength{\tabcolsep}{2pt}
  \renewcommand{\arraystretch}{0}
  \begin{subfigure}[b]{0.24\textwidth}
    \centering\includegraphics[width=\linewidth]{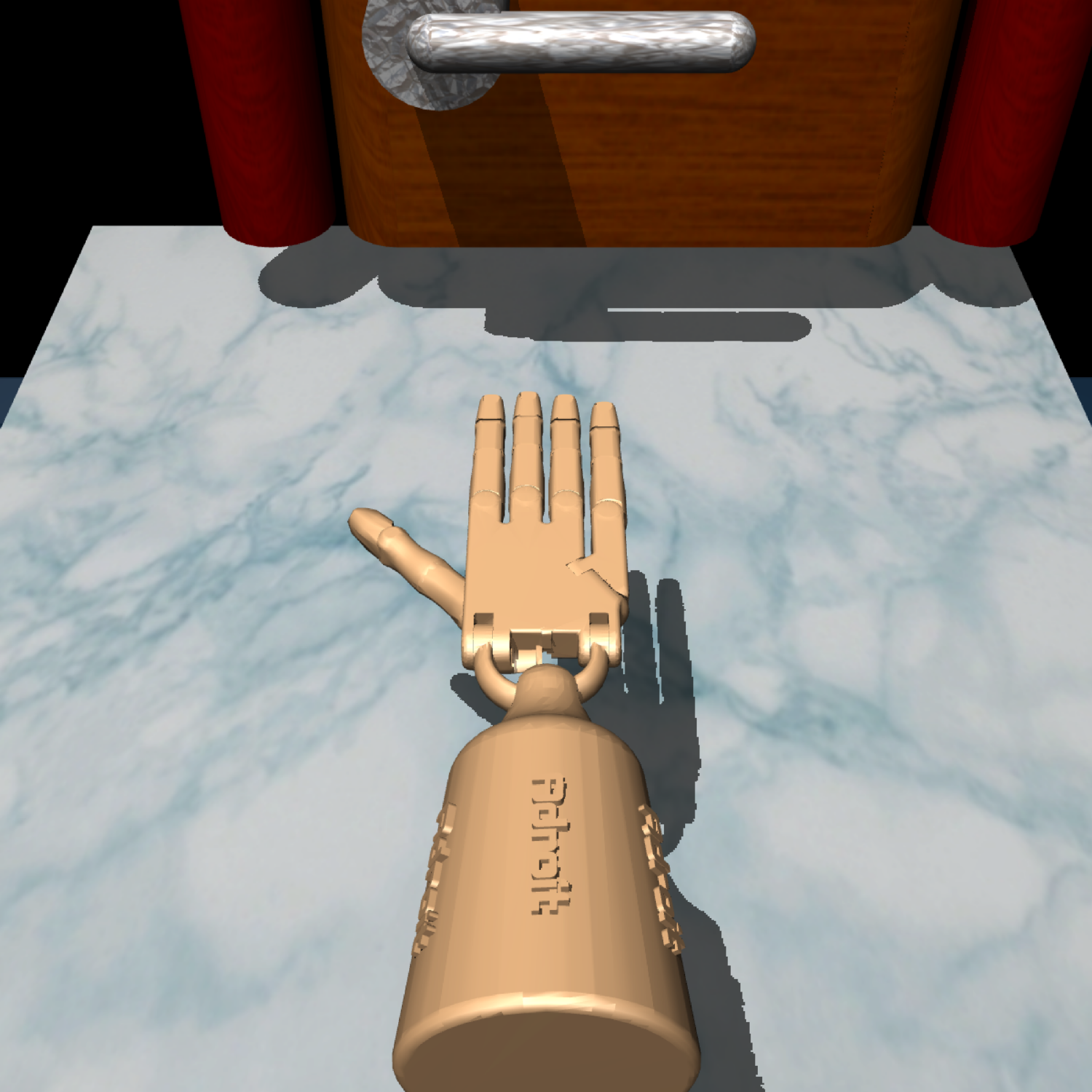}
    \caption{broken-joint (M)}
  \end{subfigure}
  \begin{subfigure}[b]{0.24\textwidth}
    \centering\includegraphics[width=\linewidth]{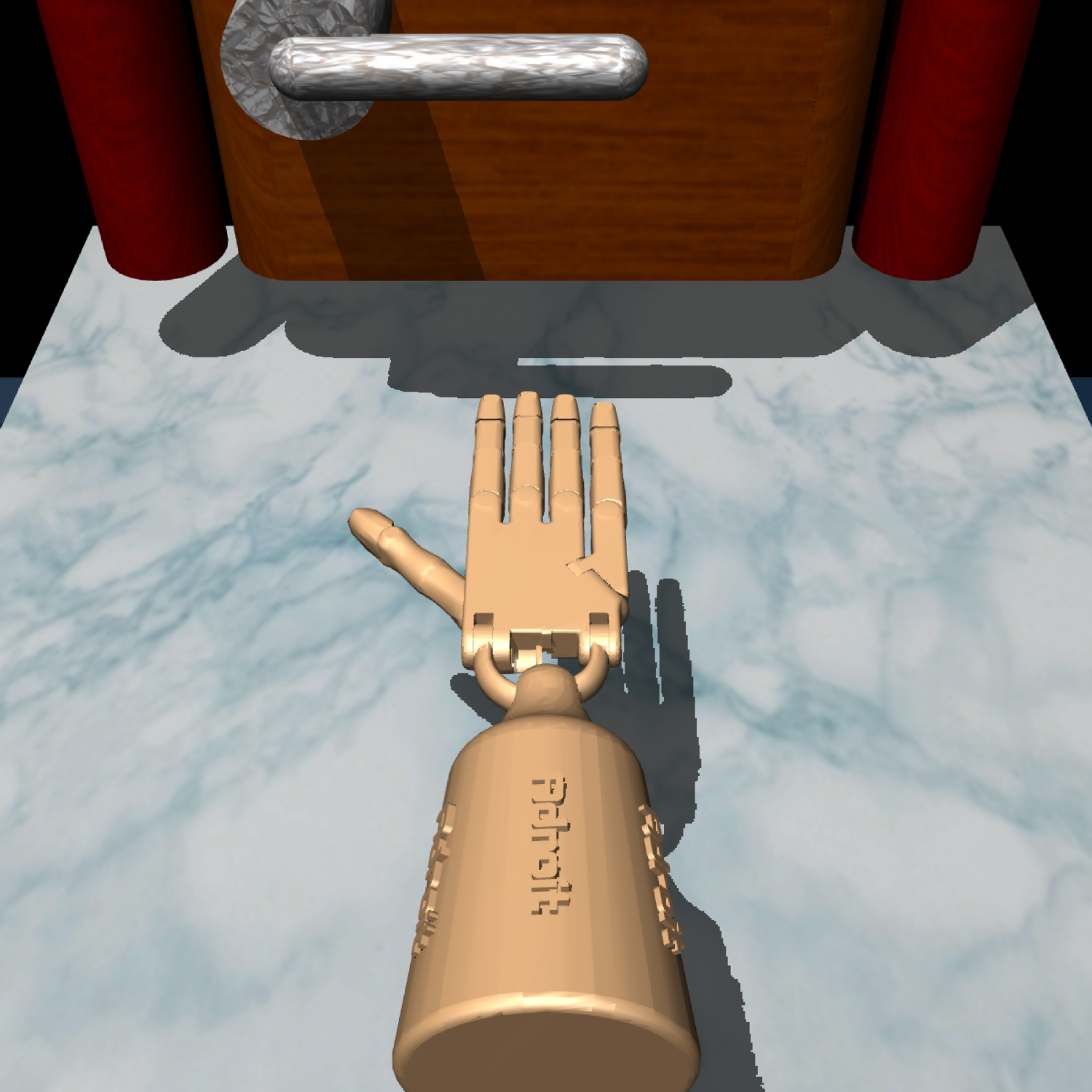}
    \caption{broken-joint (H)}
  \end{subfigure}
  \begin{subfigure}[b]{0.24\textwidth}
    \centering\includegraphics[width=\linewidth]{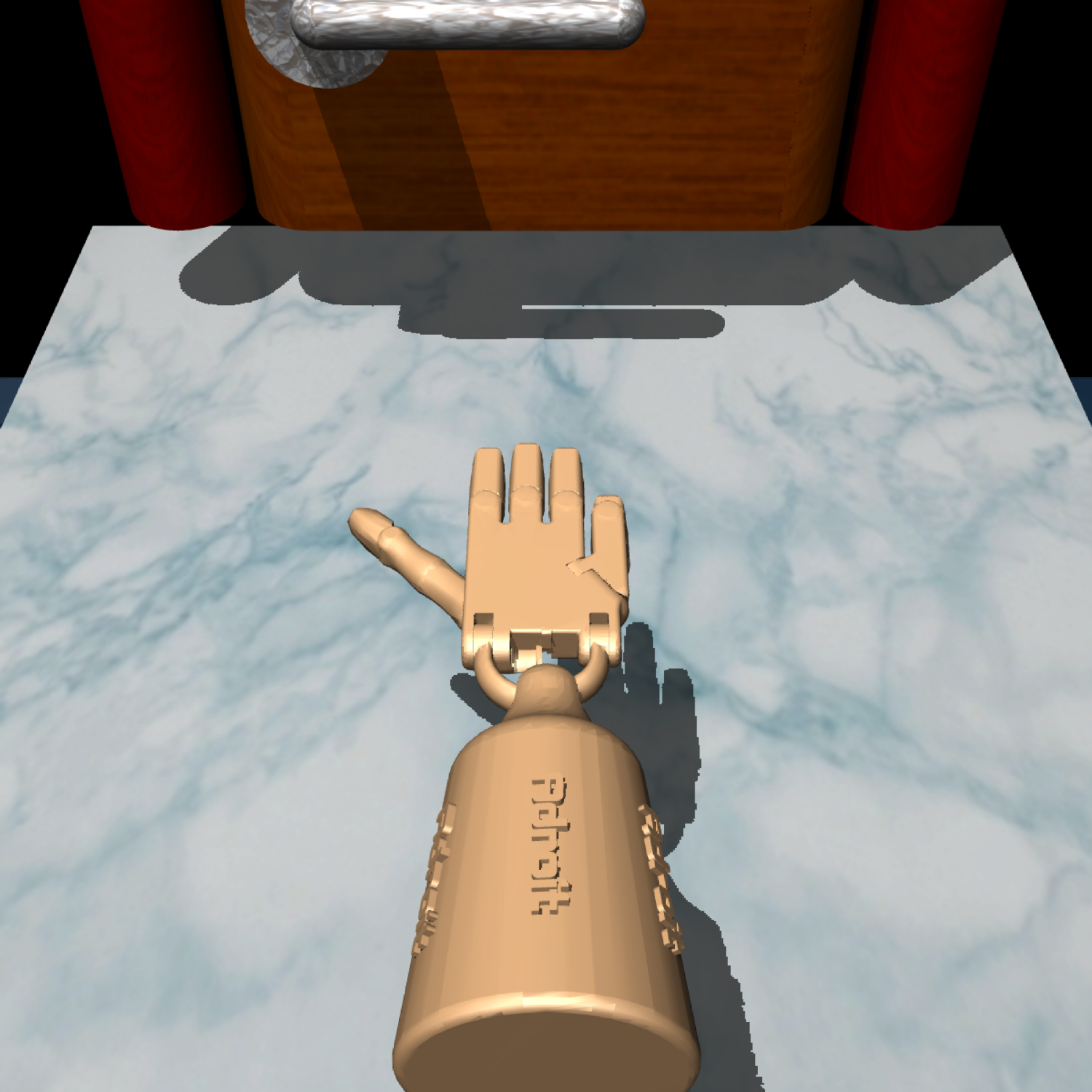}
    \caption{shrink-finger (M)}
  \end{subfigure}
  \begin{subfigure}[b]{0.24\textwidth}
    \centering\includegraphics[width=\linewidth]{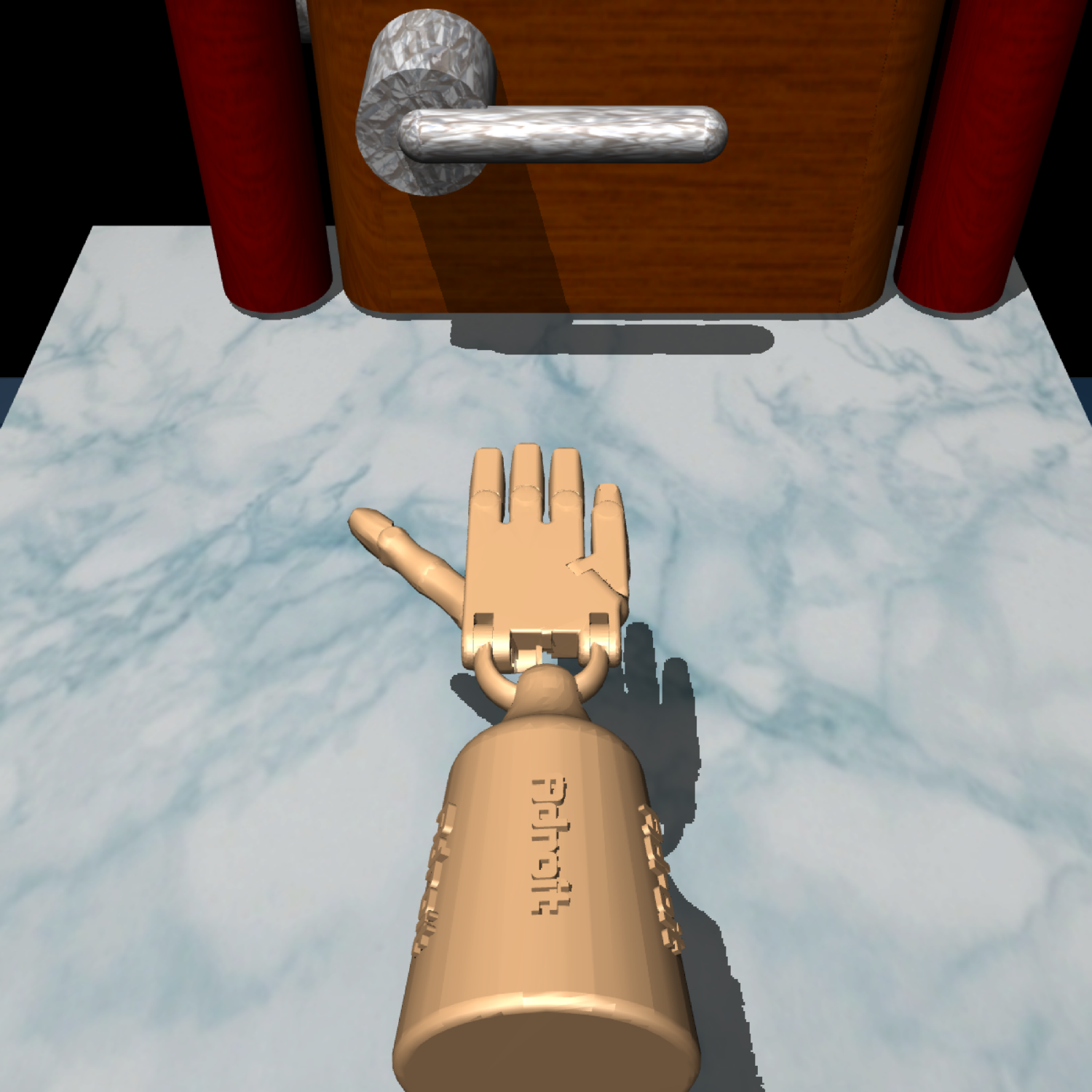}
    \caption{shrink-finger (H)}
  \end{subfigure}
  \caption{Visualization of morphology and kinematic shift environments for Adroit Door.}
  \label{fig:adroit-door-morph}
\end{figure}

\subsection{Datasets}
Note that we follow D4RL and adopt the \textit{normalized score} metric to better characterize the agent’s performance across different tasks. The normalized score in the target domain is defined as:
\begin{align*}
\mathrm{NS} = \frac{J - J_r}{J_e - J_r} \times 100,
\end{align*}
where $J$ denotes the return achieved by the agent in the target domain, and $J_r$, $J_e$ represent the returns obtained by the random and expert policies in the same domain, respectively. 

We summarize the reference scores of $J_r$ and $J_e$ under different dynamics shift scenarios in \Cref{tab:ref_scores}, \Cref{tab:ref_scores_antmaze} and \Cref{tab:ref_scores_adroit}.

\begin{table}[ht]
\centering
\setlength{\tabcolsep}{4pt}
\caption{The referenced minimum and maximum scores for MuJoCo datasets under various dynamics shift scenarios. These reference scores are used to compute normalized performance in the target domain.}
\label{tab:ref_scores}
\begin{tabular}{lccc}
\toprule
\textbf{Task Name} & \textbf{Dynamics Shift Type} & \textbf{Reference Min Score ($J_r$)} & \textbf{Reference Max Score ($J_e$)} \\
\midrule
\multirow{9}{*}{HalfCheetah} 
 & Gravity - 0.1    & -280.18 & 2466.85 \\
 & Gravity - 0.5     & -280.18 & 9509.15 \\
 & Gravity - 2.0     & -280.18 & 9509.15 \\
 & Gravity - 5.0     & -280.18 & 3756.24 \\
 & Friction - 0.1    & -280.18 & 41696.55 \\
 & Friction - 0.5   & -280.18 & 7357.07 \\
 & Friction - 2.0   & -280.18 & 11255.97 \\
 & Friction - 5.0   & -280.18 & 10199.33 \\
 & Local  & -280.18 & 12135.0  \\
\midrule
\multirow{9}{*}{Hopper} 
 & Gravity - 0.1    &-26.34  & 3234.3 \\
 & Gravity - 0.5     & -26.34 & 3234.3 \\
 & Gravity - 2.0     & -26.34 & 3234.3 \\
 & Gravity - 5.0     & -26.34 & 3234.3 \\
  & Friction - 0.1    & -26.34 & 3234.3 \\
 & Friction - 0.5   & -26.34 & 3234.3 \\
 & Friction - 2.0   & -26.34 & 3234.3 \\
 & Friction - 5.0   & -26.34 & 3234.3 \\
 & Local  & -20.272305 & 3234.3  \\
\midrule
\multirow{9}{*}{Walker2d} 
 & Gravity - 0.1    & 10.08 & 2074.90 \\
 & Gravity - 0.5     & 10.08 & 5194.71 \\
 & Gravity - 2.0     & 10.08 & 5056.45 \\
 & Gravity - 5.0     & 10.08 & 3665.39 \\
  & Friction - 0.1    & 10.08 &  3360.18 \\
 & Friction - 0.5   & 10.08 & 4229.35 \\
 & Friction - 2.0   & 10.08 & 5180.04 \\
 & Friction - 5.0   & 10.08 & 4988.84 \\
 & Local  & 1.629008   & 4592.3 \\
\midrule
\multirow{9}{*}{Ant} 
 & Gravity - 0.1    & -325.6 &  2782.09 \\
 & Gravity - 0.5     & -325.6 & 4317.07 \\
 & Gravity - 2.0     & -325.6 & 6705.12 \\
 & Gravity - 5.0     & -325.6 & 6226.89 \\
 & Friction - 0.1    & -325.6 & 7938.96 \\
 & Friction - 0.5   & -325.6 & 8301.34 \\
 & Friction - 2.0   & -325.6 & 5167.38 \\
 & Friction - 5.0   & -325.6 & 4545.02 \\
  & Local  & -325.6 & 3879.7 \\
\bottomrule
\end{tabular}
\end{table}

\begin{table}[ht]
\centering
\setlength{\tabcolsep}{4pt}
\renewcommand{\arraystretch}{1.15}
\caption{The referenced minimum and maximum scores for AntMaze medium datasets under various dynamics shift scenarios. These reference scores are used to compute normalized performance in the target domain.}
\label{tab:ref_scores_antmaze}
\begin{tabular}{lccc}
\toprule
\textbf{Task Name} & \textbf{Dynamics Shift Type} & \textbf{Reference Min Score ($J_r$)} & \textbf{Reference Max Score ($J_e$)} \\
\midrule
\multirow{6}{*}{AntMaze-medium} 
 & 1    & 0.0 & 1.0 \\
 & 2    & 0.0 & 1.0 \\
 & 3     & 0.0 & 1.0 \\
 & 4    & 0.0 & 1.0 \\
 & 5   & 0.0 & 1.0 \\
 & 6   & 0.0 & 1.0 \\
\bottomrule
\end{tabular}
\end{table}

\begin{table}[ht]
\centering
\setlength{\tabcolsep}{4pt}
\renewcommand{\arraystretch}{1.15}
\caption{The referenced minimum and maximum scores for Adroit datasets under various dynamics shift scenarios. These reference scores are used to compute normalized performance in the target domain.}
\label{tab:ref_scores_adroit}
\begin{tabular}{lccc}
\toprule
\textbf{Task Name} & \textbf{Dynamics Shift Type} & \textbf{Reference Min Score ($J_r$)} & \textbf{Reference Max Score ($J_e$)} \\
\midrule
\multirow{6}{*}{Pen}
 & broken-joint-medium & -12.17 & 6408.38 \\
 & broken-joint-hard   & -12.17 & 6408.38 \\
 & shrink-finger-medium& -12.17 & 6408.38 \\
 & shrink-finger-hard  & -12.17 & 6408.38 \\
\midrule
\multirow{6}{*}{Door}
 & broken-joint-medium & -52.34 & 2880.57 \\
 & broken-joint-hard   & -52.34 & 2880.57 \\
 & shrink-finger-medium& -52.34 & 2880.57 \\
 & shrink-finger-hard  & -52.34 & 2880.57 \\
\bottomrule
\end{tabular}
\end{table}

\subsection{Technical Details about Baseline Algorithms}
In this section, we describe the implementation details of the baseline methods, following the ODRL benchmark \citep{lyu2024odrlbenchmarkoffdynamicsreinforcement}. Furthermore, we list the hyperparameter setup used for all methods in \Cref{hyperparameter}.
\paragraph{BOSA \citep{liu2024beyond}.} 
BOSA deals with the OOD state--action pairs through a supported policy optimization and addresses the OOD dynamics issue through a supported value optimization by data filtering. 
Specifically, the policy is updated by maximizing the following objective function:
\begin{align*}
\mathcal{L}_{\text{actor}} 
= \mathbb{E}_{s \sim D_{\text{src}} \cup D_{\text{trg}}, \, a \sim \pi_\phi(s)} 
   \big[ Q(s,a) \big], 
\quad \text{s.t.} \quad 
\mathbb{E}_{s \sim D_{\text{src}} \cup D_{\text{trg}}} 
   \big[ \hat{\pi}_{\theta_{\text{offline}}} (\pi_\phi(s)\mid s) \big] > \epsilon.
\end{align*}

Here, $\epsilon$ is the threshold, $\hat{\pi}_{\theta_{\text{offline}}}$ is the learned policy for the combined offline dataset. The value critic is updated with
\begin{align}
\mathcal{L}_{\text{critic}} 
&= \mathbb{E}_{(s,a)\sim D_{\text{src}}} \big[ Q(s,a) \big] \notag \\
&\quad + \mathbb{E}_{(s,a,r,s')\sim D_{\text{src}} \cup D_{\text{trg}}, \, a' \sim \pi_\phi(\cdot \mid s')} 
   \Big[ I\left(\hat{p}_{\text{trg}}(s' \mid s,a) > \epsilon' \right) 
   \big(Q_{\theta_i}(s,a) - y \big)^2 \Big] ,\notag
\end{align}
where $I(\cdot)$ is the indicator function, 
$\hat{p}_{\text{trg}}(s' \mid s,a) = \argmax \, 
\mathbb{E}_{(s,a,s') \sim D_{\text{trg}}}\big[\log \hat{p}_{\text{trg}}(s' \mid s,a)\big]$ 
is the estimated target domain dynamics and $\epsilon'$ is the threshold. The above objective is transformed to a relaxed Lagrangian form for optimization. Both
the behavior policy and the dynamics models are modeled by the CVAE. We implement BOSA by
following the instructions in the original paper. BOSA is trained for 500K gradient steps in practice
by drawing samples from both the source domain dataset and the target domain dataset.

\paragraph{IQL \citep{kostrikov2021offline}.}
IQL learns the state value function and state--action value function 
simultaneously by minimizing the following expectile regression loss function
\begin{align*}
\mathcal{L}_{V} 
= \mathbb{E}_{(s,a)\sim D_{\text{src}} \cup D_{\text{trg}}} 
   \big[ L_{2}^{\tau}\big(Q_\theta(s,a) - V_\psi(s)\big) \big],
\end{align*}
where $L_{2}^{\tau}(u) = \lvert \tau - I(u < 0)\rvert  \lVert u \rVert^{2}$, 
$I(\cdot)$ is the indicator function, 
and $\theta$ is the target network parameter.
The state--action value function is then updated by
\begin{align*}
\mathcal{L}_{Q} 
= \mathbb{E}_{(s,a,r,s') \sim D_{\text{src}} \cup D_{\text{trg}}} 
   \left[ \big( r(s,a) + \gamma V_\psi(s') - Q_\theta(s,a) \big)^2 \right].
\end{align*}

The advantage function is $A(s,a) = Q(s,a) - V(s)$. 
The policy is optimized by the advantage-weighted behavior cloning
\begin{align*}
\mathcal{L}_{\text{actor}} 
= \mathbb{E}_{(s,a) \sim D_{\text{src}} \cup D_{\text{trg}}} 
   \big[ \exp\big(\beta \cdot A(s,a)\big) \log \pi_\phi(a \mid s) \big],
\end{align*}
where $\beta$ is the inverse temperature coefficient. We implement IQL by following its official codebase
and also train the IQL agent on offline datasets from both domains for 500K gradient steps.

\paragraph{DARA \citep{liu2022dara}.}

DARA is the offline version of DARC \citep{eysenbach2021offdynamicsreinforcementlearningtraining}. 
It trains two domain classifiers 
$q_{\theta_{\text{SAS}}}(\text{target} \mid s_t, a_t, s_{t+1})$, 
$q_{\theta_{\text{SA}}}(\text{target} \mid s_t, a_t)$ 
with the following objectives
\begin{align}
\mathcal{L}(\theta_{\text{SAS}}) 
&= \mathbb{E}_{D_{\text{tar}}}\left[ \log q_{\theta_{\text{SAS}}}(\text{target} \mid s_t, a_t, s_{t+1}) \right] 
   + \mathbb{E}_{D_{\text{src}}}\left[ \log \big( 1 - q_{\theta_{\text{SAS}}}(\text{target} \mid s_t, a_t, s_{t+1}) \big) \right], \notag\\
\mathcal{L}(\theta_{\text{SA}}) 
&= \mathbb{E}_{D_{\text{tar}}}\left[ \log q_{\theta_{\text{SA}}}(\text{target} \mid s_t, a_t) \right] 
   + \mathbb{E}_{D_{\text{src}}}\left[ \log \big( 1 - q_{\theta_{\text{SA}}}(\text{target} \mid s_t, a_t) \big) \right],\notag
\end{align}
to estimate the dynamics gap 
\[
\log \frac{P_{\mathcal{M}_{\text{tar}}}(s_{t+1}\mid s_t,a_t)}{P_{\mathcal{M}_{\text{src}}}(s_{t+1}\mid s_t,a_t)}
\]
between the source domain and the target domain. 
DARA approximates this term by leveraging the trained classifiers and proposes to modify the source domain rewards as follows
\begin{align}
\hat{r}_{\text{DARA}} 
&= r - \lambda \times \delta r,\notag \\
\delta r(s_t,a_t) \notag
&= -\log \frac{q_{\theta_{\text{SAS}}}(\text{target} \mid s_t, a_t, s_{t+1}) \, q_{\theta_{\text{SA}}}(\text{source} \mid s_t, a_t)}{q_{\theta_{\text{SAS}}}(\text{source} \mid s_t, a_t, s_{t+1}) \, q_{\theta_{\text{SA}}}(\text{target} \mid s_t, a_t)},\notag
\end{align}
where $\lambda$ is an important hyperparameter that controls the strength of the reward penalty, which is set to be 0.1 by default. After that, DARA is trained on data from both domains for 500K gradient steps. We implement DARA by following the original paper and clip the reward penalty term to lie in $[-10, 10]$. We use IQL as the base algorithm for DARA.

\paragraph{IGDF \citep{wen2024contrastive}.}

IGDF captures the dynamics gap between the source domain and the target domain through contrastive learning. 
It trains a score function $h(\cdot)$ using $(s,a,s'_{\text{tar}}) \sim D_{\text{tar}}$ from the target domain as positive samples, 
and mixed transition $(s,a,s'_{\text{src}})$ as negative samples, where $(s,a) \sim D_{\text{tar}},\, s'_{\text{src}} \sim D_{\text{src}}$. 
The score function is optimized via the following contrastive learning objective:
\begin{align}
\mathcal{L}_{\text{contrastive}} 
= - \mathbb{E}_{(s,a,s'_{\text{tar}})} \mathbb{E}_{S'^{-}}
\left[
    \log \frac{h(s,a,s'_{\text{tar}})}{\sum_{s' \sim S'^{-} \cup s'_{\text{tar}}} h(s,a,s')}
\right], \notag
\end{align}
where $S'^{-}$ is a collection of the next states in negative samples. 
Practically, IGDF adopts two neural networks $\phi(s,a), \psi(s')$ to learn representations of state-action pairs and states, 
and approximate the score function with a linear parameterization of them:
\begin{align}
h(s,a,s') = \exp\big(\phi(s,a)^\top \psi(s')\big). \notag
\end{align}
Based on the measured score function, IGDF proposes to filter out source domain data when training value functions:
\begin{align}
\mathcal{L}_{\text{critic}} 
= \tfrac{1}{2}\, \mathbb{E}_{D_{\text{tar}}}\left[(Q_\theta - \mathcal{T}Q_\theta)^2\right] 
+ \tfrac{1}{2}\, \alpha \cdot h(s,a,s') \, \mathbb{E}_{(s,a,s') \sim D_{\text{src}}} 
    \left[ \mathbf{1}\left(h(s,a,s') > h_{\xi\%}\right) (Q_\theta - \mathcal{T}Q_\theta)^2 \right], \notag
\end{align}
where $\alpha$ is the importance coefficient for weighting the TD error of the source domain data, 
and $\xi$ is the data selection ratio akin to that in OTDF. 
We run IGDF by using its official codebase and use IQL as its backbone.

\paragraph{OTDF \citep{lyu2025cross}.}
OTDF performs off-dynamics offline policy adaptation by aligning the transition dynamics of the source and target domains using optimal transport. It first solves the optimal transport problem using 
\begin{align}
W(u, u') = \min_{\mu \in \mathcal{M}} 
\sum_{t=1}^{|D_{\text{src}}|} \sum_{t'=1}^{|D_{\text{tar}}|} 
C(u_t, u'_{t'}) \, \mu_{t,t'},\notag
\end{align}
given cost function $C$. Suppose the optimal coupling obtained is $\mu^*$, then it defines the deviation between a source domain data point and the target domain dataset as:
\begin{align}
d(u_t) &= - \sum_{t'=1}^{|D_{\text{tar}}|} C(u_t, u'_{t'}) \, \mu^*_{t,t'}, u_t= (s_t^{\text{src}}, a_t^{\text{src}}, (s'_{\text{src}})_t) \sim D_{\text{src}}.\notag
\end{align}
After computing and normalizing $d$, it trains the value function via 
\begin{align}
\mathcal{L}_Q 
= \mathbb{E}_{(s,a,s') \sim D_{\text{tar}}} \Big[ (Q_\theta - \mathcal{T} Q_\theta)^2 \Big] 
+ \mathbb{E}_{(s,a,s') \sim D_{\text{src}}} \Big[ exp(\hat{d}) \cdot \mathbf{1}(d > d_{\xi\%}) \, (Q_\theta - \mathcal{T} Q_\theta)^2 \Big],\notag
\end{align}
where $\hat{d}$ is normalized $d$ and $\mathbf{1}(\cdot)$ is the indicator function.
It then trains a Conditional Variational Autoencoder (CVAE) using loss function
\begin{align}
    \mathcal{L}_{\text{CVAE}} 
= \mathbb{E}_{(s,a) \sim D_{\text{tar}},\, z \sim E_\nu(s,a)}
\Big[ (a - D_\varsigma(s, z))^2 \Big] 
+ D_{\text{KL}}\big( E_\nu(s,a) \,\|\, \mathcal{N}(0, I) \big),\notag
\end{align}
where $E_\nu(s,a)$ is the encoder and $D_\varsigma(s, z)$ is the decoder, to model the target dynamics and generates multiple latent transitions. Finally, they train the policy via 
\begin{align}
\hat{\mathcal{L}}_\pi 
= \mathcal{L}_\pi 
- \beta \, \mathbb{E}_{s \sim D_{\text{src}} \cup D_{\text{tar}}} 
\Bigg[ \log \Bigg( 
\sum_{i=1}^{M} \exp \big( \log \hat{\pi}^i_{\text{tar}}(\pi(\cdot|s) \mid s) \big) 
\Bigg) \Bigg],\notag
\end{align}
where $\hat{\pi}^i_{\text{tar}}$ denotes the $i$-th constructed Gaussian behavior policy in the target domain.
A data selection mechanism filters source samples whose generated transitions closely match the target. We run OTDF using its official codebase.

\paragraph{DmC \citep{van2025dmc}.}
DmC combines nearest-neighbor-based domain discrepancy estimation with diffusion-based data augmentation. For each source transition $x=(s,a,s')$, it estimates the domain discrepancy as
\begin{align*}
\rho_k(x) =
\log \nu_{k,\text{tar}}(x)
- \log \nu_{k,\text{src}}(x),
\end{align*}
where $\nu_{k,\text{tar}}$ and $\nu_{k,\text{src}}$ denote the $k$-th nearest-neighbor distances from $x$ to the target and source datasets, respectively. The scores are shifted as $\hat{\rho}_k(x) = \rho_k(x) - \min_{x'}\rho_k(x')$ and converted into weights $w_k(x)=1/(1+\hat{\rho}_k(x))$, where higher weights indicate better alignment with the target data. DmC then trains a diffusion model on source transitions, using these proximity weights as conditioning signals to generate target-aligned transitions. The generated transitions are added to the source dataset for offline RL training. During critic training, DmC filters and reweights source transitions according to
\begin{align*}
\omega(x)=w_k(x)\mathbf{1}(w_k(x)\geq w_{k,\xi\%}),
\end{align*}
where $w_{k,\xi\%}$ is the weight threshold determined by the filtering percentile. Using IQL as the backbone, DmC further applies a CVAE-based regularizer to keep the learned policy close to the target behavior policy.

\begin{table*}[ht]
\centering
\caption{Hyperparameters of \algname and baselines.}
\label{hyperparameter}
\small
\setlength{\tabcolsep}{4pt}
\renewcommand{\arraystretch}{0.9}
\begin{tabular}{ll}
\toprule
\textbf{Hyperparameter} & \textbf{Value} \\
\midrule

\multicolumn{2}{l}{\textbf{Shared}} \\
Actor network & (256, 256) \\
Critic network & (256, 256) \\
Learning rate & $3 \times 10^{-4}$ \\
Optimizer & Adam \\
Discount factor & 0.99 \\
Replay buffer size & $10^6$ \\
Nonlinearity & ReLU \\
Target update rate & $5 \times 10^{-3}$ \\
Source domain batch size & 128 \\
Target domain batch size & 128 \\
\midrule

\multicolumn{2}{l}{\textbf{DARA}} \\
Temperature coefficient & 0.2 \\
Maximum log std & 2 \\
Minimum log std & $-20$ \\
Classifier network & (256, 256) \\
Reward penalty coefficient $\lambda$ & 0.1 \\
\midrule

\multicolumn{2}{l}{\textbf{BOSA}} \\
Temperature coefficient & 0.2 \\
Maximum log std & 2 \\
Minimum log std & $-20$ \\
Policy regularization coefficient $\lambda_{\text{policy}}$ & 0.1 \\
Transition coefficient $\lambda_{\text{transition}}$ & 0.1 \\
Threshold parameter $\epsilon, \epsilon'$ & $\log(0.01)$ \\
Value weight $\omega$ & 0.1 \\
CVAE ensemble size & 1 (behavior), 5 (dynamics) \\
\midrule

\multicolumn{2}{l}{\textbf{IGDF}} \\
Representation dimension & \{16, 64\} \\
Contrastive encoder network & (256, 256) \\
Encoder pretraining steps & 7000 \\
Importance coefficient & 1.0 \\
Data selection ratio $\xi$ & 75\% \\
\midrule

\multicolumn{2}{l}{\textbf{OTDF}} \\
CVAE training steps & 10000 \\
CVAE learning rate & 0.001 \\
Number of sampled latents $M$ & 10 \\
Gaussian std & $\sqrt{0.1}$ \\
Cost function & cosine \\
Policy coefficient $\beta$ & 0.5 \\
Data selection ratio $\xi$ & 80\% \\
\midrule

\multicolumn{2}{l}{\textbf{IQL}} \\
Temperature coefficient & 0.2 \\
Maximum log std & 2 \\
Minimum log std & $-20$ \\
Inverse temperature $\beta$ & 3.0 \\
Expectile $\tau$ & 0.7 \\
\midrule

\multicolumn{2}{l}{\textbf{DmC}} \\
Nearest-neighbor index $k$ & 4 \\
Data selection ratio $\xi$ & 50\% \\
Diffusion training steps & $10^5$ \\
Guidance score percentile & 90--99 \\
Number of synthetic transitions & $\max(|D_{\text{src}}|, 10^6)$ \\
Policy regularization weight $\lambda$ & 0.1 \\
\midrule

\multicolumn{2}{l}{\textbf{\algname}} \\
Number of clusters $K$ & $\{30,50\}$ for MuJoCo/Adroit; \\
& $\{5,10,15,20\}$ for AntMaze \\
Policy regularization weight $\lambda$ & $\{0.1,0.5,1.0\}$ for MuJoCo/Adroit; \\
& $\{0.5,1.0,5.0\}$ for AntMaze \\
Filtering ratio $(\xi_1,\xi_2,\xi_3)$ & $(90\%,80\%,70\%)$ \\
Input noise std $\sigma$ & 1.0 \\
Importance weight $\alpha$ & 1.0 \\
Diameter threshold $\delta$ & 1.5 \\
CVAE pretraining gradient steps & 10000\\
CVAE learning rate &  0.001 \\
\bottomrule
\end{tabular}
\end{table*}

\subsection{Technical Details about \algname}
\label{app:\algname}

The \algname algorithm (Localized Dynamics-Aware Domain Adaptation), summarized in \Cref{app:\algname_alg}, leverages clustering and divergence-based filtering to improve off-dynamics policy learning. First, target domain data are clustered by next-state variables to form clusters, and source samples are assigned to clusters based on proximity to centroids within a diameter threshold. For each cluster, a classifier estimates KL divergence between source and target distributions, and clusters are ranked accordingly. Source samples are then filtered adaptively based on the divergence ranking of their assigned clusters, which are divided into three classes according to their KL divergence: low, medium, and high. Clusters in the low-divergence class, which are most similar to the target dataset, admit the largest proportion $\xi_1$ of their source samples; clusters in the medium-divergence class admit a moderate proportion $\xi_2$; and clusters in the high-divergence class—indicating greater dissimilarity—admit only a small fraction $\xi_3$. 

After filtering, each retained source sample is assigned an importance weight derived from its normalized KL estimate, so that samples from more similar clusters have higher influence during training, while samples from less similar clusters are down-weighted. During training, mini-batches from both domains are used to update the value function, compute target values, and optimize the state-action value function with divergence-weighted losses. Finally, the target network is updated, and the policy is optimized with CVAE-based regularization, enabling effective adaptation across domains.

We use fixed filtering ratios $(\xi_1,\xi_2,\xi_3)=(90\%,80\%,70\%)$ across all environments and do not tune them separately for each task. This three-tier schedule keeps the filtering rule simple and reproducible, while retaining a nonzero fraction of samples from high-KL clusters to preserve state-space coverage. For the number of clusters $K$ and the policy regularization weight $\lambda$, we tune them within small task-dependent candidate sets. For MuJoCo locomotion and Adroit manipulation tasks, we use $K\in\{30,50\}$ and $\lambda\in\{0.1,0.5,1.0\}$. For AntMaze, we use a smaller cluster range $K\in\{5,10,15\}$ and a larger regularization range $\lambda\in\{0.5,1.0,5.0\}$, because useful source-target overlap is concentrated around structured map regions. Additional sensitivity results on filtering ratios, AntMaze, and Adroit are reported in Appendix~\ref{app:filtering_ratio_sensitivity} and Appendix~\ref{app:additional_hyperparameter_sensitivity}.

\begin{algorithm}[ht]
\caption{\algname: Localized Dynamics-Aware Domain Adaptation}
\label{app:\algname_alg}
\begin{algorithmic}[1]
\STATE \textbf{Input:} Source domain data $D_{src}$, target domain data $D_{tar}$
\STATE \textbf{Initialize:} Policy $\pi_{\phi}$, value networks $V_{\psi},Q_{\theta}$, target Q function $Q_{\theta'}$, batch size $B$, number of clusters $K$, diameter threshold $\delta$, critic importance ratio $\alpha$, regularization weight $\lambda$, filtering ratio $\xi_1,\xi_2,\xi_3$, target update rate $\eta$, CVAE with encoder $ E_{\nu}(s,a)$ and decoder $D_{\varsigma}(s, z)$, number of sampled latent variables $M$
\STATE Train CVAE policy to model the behavior policy in the target domain dataset using 
\begin{align}
    \mathcal{L}_{\text{CVAE}}
= \mathbb{E}_{(s,a)\sim \mathcal{D}_{\text{tar}},\, z\sim E_{\nu}(s,a)}
\left[
\left\| a - D_{\varsigma}(s, z) \right\|_2^2
+ D_{\mathrm{KL}}\left( E_{\nu}(s,a)\| \mathcal{N}(0, I) \right)
\right]. \notag
\end{align}
\STATE Cluster $D_{tar}$ into $\{\mathcal{N}^{i}\}_{i=1}^K$ based on $s'$, with centroid $c^i$ and diameter $d^i$
\FOR{\texttt{each source sample $(s,a,s') \in D_{src}$}}
    \STATE Find nearest centroid $c^i$ among $\{\mathcal{N}^i\}_{i=1}^K$
    \IF{$\text{dist}(s', c^i) \leq \delta \cdot \frac{1}{K}\sum_{i=1}^Kd^i$}
        \STATE Assign $(s,a,s')$ to cluster $\mathcal{N}^i$
    \ENDIF
\ENDFOR
\FOR{i=1, 2..., K}
    \STATE Train classifiers $D_n(z)$ based on \Cref{eq:BCE_loss}
    \STATE Calculate cluster estimate of KL divergence $\hat{D}_{KL}^n$ based on \Cref{eq:KL_estimate}
\ENDFOR
\STATE Sort $\{\mathcal{N}^{i}\}_{i=1}^K$ based on $\hat{D}_{KL}^n$ from smallest to largest and get $\{\mathcal{N}^{(j)}\}_{j=1}^K$
\FOR{each cluster $\mathcal{N}^{(j)}$, $j=1,\dots,K$}
    \STATE Rank point estimates $\{d_i^j\}_{i=1}^{|\mathcal{N}_0^{(j)}|}$ of sampled source data
    \STATE Admit the top $\xi_{g(j)}\%$ samples, where
    $
    g(j)=
    \begin{cases}
    1, & j \le \lceil K/3 \rceil,\\
    2, & \lceil K/3 \rceil < j \le \lceil 2K/3 \rceil,\\
    3, & j > \lceil 2K/3 \rceil.
    \end{cases}
    $
\ENDFOR
\STATE Normalize point estimate of KL divergence via $\hat{d_i}=\frac{d_i-\max_id_i}{\max_i d_i-\min_i d_i}, i \in \{1,...,|\tilde{D}_{src}|\}$
\FOR{\texttt{each gradient step}}
    \STATE Sample a mini-batch $b_{\text{src}}$ of size $B/2$ from $\tilde{D}_{src}$ and $b_{\text{tar}}$ of size $B/2$ from $D_{tar}$
    \STATE Update the state value function $V_{\psi}$ via: $\mathcal{L}_{V} = \EE_{(s,a)\sim \tilde{D}_{src}\cup D_{tar}}
    \Big[ L_{2}^{\tau}\big(Q_{\theta'}(s,a) - V_{\psi}(s)\big)\Big]$
    \STATE Compute the target value via: $y = r + \gamma V_{\psi}(s')$
    \STATE Optimize the state-action value function $Q_{\theta}$ on $b_{\text{src}} \cup b_{\text{tar}}$ via:
\[
    \cL_Q
= \EE_{D_{tar}}[(Q_\theta - y)^2]
+ \mathbb{E}_{(s,a,s',d) \sim \tilde{D}_{src}}[\exp(-\alpha \cdot \hat{d_i})(Q_\theta -y)^2].
\]
    \STATE Update the target network via: $\theta' \gets \eta \theta + (1 - \eta)\theta'$
    \STATE Decode $M$ actions from the CVAE and construct Gaussian distributions
$\{\hat{\pi}^{\text{tar}}_{i}(\cdot \mid s)\}_{i=1}^{M}$.
    \STATE Compute the advantage A and optimize the policy $\pi_{\phi}$ on $b_{\text{src}} \cup b_{\text{tar}}$ using
    \begin{align} 
    \mathcal{L}_{\pi}
=
-\mathbb{E}_{(s,a)\sim \tilde{D}_{\text{src}} \cup \mathcal{D}_{\text{tar}}}
\left[
\exp\left(A\right)
\log \pi_{\phi}(a \mid s)
\right]
-
\lambda 
\mathbb{E}_{s\sim \tilde{D}_{\text{src}} \cup \mathcal{D}_{\text{tar}}}
\left[
\log \left(
\sum_{i=1}^{M}
\hat{\pi}^{i}_{\text{tar}}\left(\pi_{\phi}(\cdot \mid s) \mid s\right)
\right)
\right]. \notag
    \end{align}
\ENDFOR
\end{algorithmic}
\end{algorithm}

\subsection{Compute Infrastructure}
We run the experiment on a single GPU: NVIDIA RTX A5000-24564MiB with 8-CPUs: AMD Ryzen Threadripper 3960X 24-Core. Each experiment requires 12GB RAM and require 20GB available
disk space for storage of the data.

\section{More Experimental Results}\label{sec:more_exp_results}

In this section, we provide more experimental results that are omitted from the main text due to space limitations.

\subsection{Results under morphology shifts} \label{app:morph_results}

As shown in \Cref{morph_shift_comparison}, our method improves the total score to 597.88, yielding a 32.7\% improvement over the second best approach OTDF, and a 63\% improvement over backbone algorithm IQL. Across the 16 evaluation tasks (8 environments × 2 shift levels), \algname achieves best performance in 13 out of 16 tasks. These results indicate that our proposed approach consistently enhances robustness under morphology shifts. Relative to the baseline IQL, our method delivers consistent and significant performance gains across 13 out of 16 tasks.

\begin{table*}[ht]
\centering
\caption{
Performance comparison on HalfCheetah, Ant, Walker2d, and Hopper environments under
\textbf{morphology shifts} across 5 seeds. We use \textbf{M} and \textbf{H} to denote the medium and hard levels of the dynamics shift.
}
\label{morph_shift_comparison}
\resizebox{\textwidth}{!}
{
\begin{tabular}{llcccccccc}
\toprule
Env & Type & Level & BOSA & IQL & DARA & IGDF & OTDF & DmC & \algname \\
\midrule

\multirow{4}{*}{HalfCheetah}
& morph-thigh & M & \cellcolor{second}25.12$\pm$4.33 & 22.22$\pm$1.13 & 20.68$\pm$2.49 & 19.31$\pm$4.35 & 21.94$\pm$6.30 & 17.65$\pm$2.90 & \cellcolor{best}\textbf{26.42$\pm$3.44} \\
&              & H & 20.74$\pm$3.22 & 15.90$\pm$1.47 & 16.35$\pm$0.97 & 15.57$\pm$2.06 & \cellcolor{second}22.75$\pm$2.76 & 17.79$\pm$1.54 & \cellcolor{best}\textbf{25.81$\pm$5.29} \\
& morph-torso & M & 0.66$\pm$0.90 & 1.95$\pm$1.69 & 1.40$\pm$1.66 & 1.84$\pm$0.37 & \cellcolor{second}3.38$\pm$1.30 & 2.34$\pm$0.52 & \cellcolor{best}\textbf{9.96$\pm$4.07} \\
&              & H & 16.44$\pm$12.82 & \cellcolor{best}\textbf{30.87$\pm$10.00} & \cellcolor{second}29.89$\pm$5.88 & 28.85$\pm$3.89 & 26.06$\pm$6.00 & 17.49$\pm$3.11 & 23.44$\pm$4.82 \\

\midrule
\multirow{4}{*}{Ant}
& morph-halflegs & M & 64.52$\pm$7.46 & 67.15$\pm$4.33 & 65.82$\pm$5.31 & 69.07$\pm$5.76 & \cellcolor{second}73.86$\pm$0.79 & 50.16$\pm$5.72 & \cellcolor{best}\textbf{77.48$\pm$0.77} \\
&                & H & 37.21$\pm$6.15 & 54.27$\pm$5.60 & \cellcolor{second}61.31$\pm$5.71 & 60.08$\pm$2.09 & 59.76$\pm$9.27 & 52.06$\pm$3.14 & \cellcolor{best}\textbf{64.47$\pm$4.62} \\
& morph-alllegs  & M & 31.09$\pm$7.09 & 30.05$\pm$7.41 & 26.49$\pm$4.52 & 37.87$\pm$17.24 & 32.98$\pm$4.91 & \cellcolor{second}63.63$\pm$2.27 & \cellcolor{best}\textbf{77.55$\pm$3.44} \\
&                & H & \cellcolor{second}13.89$\pm$1.46 & 5.65$\pm$3.30 & 3.21$\pm$3.26 & 7.24$\pm$1.40 & 12.84$\pm$2.06 & 11.19$\pm$2.47 & \cellcolor{best}\textbf{24.25$\pm$0.02} \\

\midrule
\multirow{4}{*}{Walker2d}
& morph-torso & M & 8.63$\pm$3.49 & 11.44$\pm$0.91 & 14.80$\pm$1.69 & 14.16$\pm$2.19 & 11.47$\pm$4.93 & \cellcolor{second}21.12$\pm$11.43 & \cellcolor{best}\textbf{29.67$\pm$11.80} \\
&             & H & 1.93$\pm$0.53 & 4.78$\pm$1.13 & 6.86$\pm$2.38 & 5.38$\pm$1.34 & \cellcolor{second}7.05$\pm$0.57 & 6.92$\pm$1.16 & \cellcolor{best}\textbf{7.76$\pm$2.42} \\
& morph-leg   & M & 28.10$\pm$8.64 & 40.61$\pm$4.78 & 43.87$\pm$6.22 & 41.71$\pm$10.14 & 45.07$\pm$10.09 & \cellcolor{second}50.61$\pm$5.59 & \cellcolor{best}\textbf{52.86$\pm$8.82} \\
&             & H & 10.17$\pm$1.60 & 20.54$\pm$5.17 & 20.33$\pm$3.49 & 18.78$\pm$3.84 & 15.13$\pm$3.25 & \cellcolor{second}22.56$\pm$5.66 & \cellcolor{best}\textbf{59.01$\pm$1.76} \\

\midrule
\multirow{4}{*}{Hopper}
& morph-foot  & M & 12.87$\pm$0.09 & \cellcolor{best}\textbf{25.62$\pm$18.89} & 20.05$\pm$7.32 & 14.91$\pm$4.52 & \cellcolor{second}22.56$\pm$8.65 & 13.06$\pm$0.05 & 13.11$\pm$0.05 \\
&             & H & 9.84$\pm$0.17 & 10.72$\pm$2.84 & 25.74$\pm$10.70 & 18.35$\pm$8.08 & \cellcolor{second}66.78$\pm$9.79 & 18.90$\pm$2.91 & \cellcolor{best}\textbf{67.53$\pm$14.03} \\
& morph-torso & M & 13.43$\pm$0.37 & 13.66$\pm$0.98 & 14.70$\pm$4.48 & 13.06$\pm$4.33 & \cellcolor{second}16.84$\pm$1.44 & 14.27$\pm$0.71 & \cellcolor{best}\textbf{27.81$\pm$9.08} \\
&             & H & 11.28$\pm$0.24 & 11.30$\pm$0.44 & 10.82$\pm$0.19 & 10.66$\pm$1.19 & \cellcolor{best}\textbf{12.23$\pm$0.15} & \cellcolor{second}12.03$\pm$0.19 & 10.75$\pm$0.89 \\

\midrule
\multicolumn{3}{l}{Total} 
& 305.92 & 366.73 & 382.32  & 376.84 & \cellcolor{second}450.70 & 391.78 & \cellcolor{best}\textbf{597.88} \\
\bottomrule
\end{tabular}
}
\end{table*}

\subsection{Results on Localized Perturbations}
\label{app:localized_perturbation_results}

We provide the detailed setup and numerical results for the localized perturbation experiments.
In this setting, the target domain remains unchanged, while the source dataset is perturbed to simulate region-dependent dynamics mismatch.
Policies are trained on the perturbed source dataset together with the limited target dataset and evaluated in the unchanged target domain.

For the target dataset, we randomly sample 5,000 transitions from D4RL without adding dynamics perturbations.
For the source dataset, we first partition the next-state space into 15 regions using K-means.
We then divide these regions into $M$ groups and inject Gaussian noise into state variables, where all samples in the same group share the same noise variance.
We consider $M\in\{3,5\}$.
When $M=3$, the noise variance groups are $\{0.1,0.5,2.0\}$, $\{0.2,1.0,4.0\}$, and $\{0.3,1.5,6.0\}$.
When $M=5$, the noise variance groups are $\{0.1,0.2,0.3,0.4,0.5\}$, $\{0.2,0.4,0.6,0.8,1.0\}$, and $\{0.3,0.6,0.9,1.2,1.5\}$.

As shown in \Cref{tab:main_results_8}, \algname achieves a total score of 943.09, corresponding to a 29.3\% improvement over the second-best approach OTDF and a 55.1\% improvement over IQL.
Across the 24 evaluation tasks (4 environments $\times$ 6 shift levels), \algname achieves best (18) or second-best (3) performance in 21 out of 24 tasks. In contrast, DARA and IGDF often underperform in this setting due to their reliance on global mismatch assumptions.
These results demonstrate that \algname is particularly effective under localized perturbations, where cluster-aware filtering becomes more beneficial than under global shifts.

\begin{table*}[ht]
\centering
\scriptsize
\setlength{\tabcolsep}{4pt}
\renewcommand{\arraystretch}{0.8}
\caption{Performance comparison on MuJoCo tasks under a medium-level offline dataset. Target domains remain unchanged; source datasets are shifted by adding local perturbations with varying shift levels across regions. We report normalized target-domain scores (mean ± std over five seeds). Best and second-best scores are highlighted in \textcolor{best}{green} and \textcolor{second}{blue}, respectively.}
\label{tab:main_results_8}
\resizebox{\textwidth}{!}{
\begin{tabular}{cccccccccc}
\toprule
Env & \# Groups & Shift Level & BOSA & DARA & IQL & IGDF & OTDF & DmC & \algname \\
\midrule
\multirow{6}{*}{HalfCheetah}
&\multirow{3}{*}{3}
& 0.1-0.5-2.0  & 27.36$\pm$2.22 & 19.86$\pm$4.42 & 17.57$\pm$3.41 & 14.59$\pm$3.13 & \cellcolor{second}29.26$\pm$3.10 & 28.84$\pm$1.66 & \cellcolor{best}\textbf{34.97$\pm$0.85} \\ %
&& 0.2-1.0-4.0  & \cellcolor{second}25.28$\pm$3.10 & 16.94$\pm$3.31 & 16.07$\pm$3.45 & 13.16$\pm$4.83 & 24.21$\pm$2.82 & 22.50$\pm$2.78 & \cellcolor{best}\textbf{28.36$\pm$2.76} \\
&& 0.3-1.5-6.0 & 20.75$\pm$2.66 & 6.59$\pm$0.45 & 10.73$\pm$2.88 & 8.73$\pm$2.31 & \cellcolor{second}22.84$\pm$0.54 & 17.54$\pm$3.18 & \cellcolor{best}\textbf{27.24$\pm$4.37} \\
\cmidrule(lr){2-10}
&\multirow{3}{*}{5}
 & 0.1-0.2-0.3-0.4-0.5  & 19.30$\pm$2.38 & 12.43$\pm$1.95 & 10.83$\pm$1.05 & 11.27$\pm$2.35 & \cellcolor{second}28.99$\pm$1.85 & 24.73$\pm$4.46 & \cellcolor{best}\textbf{34.28$\pm$2.91} \\
 && 0.2-0.4-0.6-0.8-1.0  & \cellcolor{second}23.59$\pm$5.81 & 13.18$\pm$2.09 & 14.77$\pm$2.41 & 11.43$\pm$3.96 & 21.41$\pm$3.95 & 19.41$\pm$4.10 & \cellcolor{best}\textbf{25.73$\pm$5.38} \\
 && 0.3-0.6-0.9-1.2-1.5 & \cellcolor{best}\textbf{23.95$\pm$2.23}  & 7.10$\pm$0.97 & 8.98$\pm$2.61 & 7.37$\pm$3.03 & 11.53$\pm$2.04 & 18.22$\pm$6.33 & \cellcolor{second}18.74$\pm$1.55 \\
\midrule

\multirow{6}{*}{Walker2d}
&\multirow{3}{*}{3}
& 0.1-0.5-2.0 & 5.28$\pm$0.54 & 17.85$\pm$10.67 & \cellcolor{second}36.44$\pm$15.19 & 14.02$\pm$10.20 & 34.52$\pm$11.33 & 13.45$\pm$7.94 & \cellcolor{best}\textbf{37.45$\pm$18.71} \\
&& 0.2-1.0-4.0 & 6.42$\pm$1.10 & 19.79$\pm$3.58 & 15.29$\pm$9.44 & 14.62$\pm$8.16 & \cellcolor{second}20.66$\pm$7.06 & 16.56$\pm$6.16 & \cellcolor{best}\textbf{31.10$\pm$18.00} \\
&& 0.3-1.5-6.0 & 7.16$\pm$2.91 & 9.86$\pm$4.44 & 12.28$\pm$8.14 & 11.14$\pm$3.69 & \cellcolor{second}16.78$\pm$9.39 & 13.18$\pm$5.16 & \cellcolor{best}\textbf{18.69$\pm$10.51} \\
\cmidrule(lr){2-10}

&\multirow{3}{*}{5}
 & 0.1-0.2-0.3-0.4-0.5 & 7.17$\pm$1.48  & 8.10$\pm$3.46 & 6.96$\pm$2.01 & 6.99$\pm$2.77 & \cellcolor{second}33.87$\pm$16.17 & 25.31$\pm$10.35 & \cellcolor{best}\textbf{35.84$\pm$16.01} \\
 && 0.2-0.4-0.6-0.8-1.0 & 6.51$\pm$4.60 & 13.90$\pm$9.23 & 12.07$\pm$7.55 & 9.37$\pm$2.01 &  13.87$\pm$6.95 & \cellcolor{second}16.72$\pm$6.93 & \cellcolor{best}\textbf{24.92$\pm$5.96} \\
 && 0.3-0.6-0.9-1.2-1.5 & 6.43$\pm$2.53  & 8.65$\pm$3.18 & 7.91$\pm$3.80 & 8.99$\pm$5.44 & \cellcolor{best}\textbf{18.52$\pm$15.46} & \cellcolor{second}16.99$\pm$5.77 & 12.98$\pm$3.88 \\
 
\midrule
\multirow{6}{*}{Hopper}
&\multirow{3}{*}{3}
& 0.1-0.5-2.0 & 7.84$\pm$2.84  & 8.70$\pm$6.67 & 8.14$\pm$6.31 & 13.08$\pm$10.44 & \cellcolor{second}33.96$\pm$4.57 & 30.61$\pm$10.01 & \cellcolor{best}\textbf{40.77$\pm$5.04} \\
&& 0.2-1.0-4.0  & 11.91$\pm$2.60 & 9.61$\pm$5.56 & 11.97$\pm$5.99 & 8.35$\pm$3.56 & \cellcolor{second}32.58$\pm$6.89 & 29.53$\pm$5.73 & \cellcolor{best}\textbf{37.51$\pm$9.16} \\
&& 0.3-1.5-6.0  & 11.85$\pm$3.19 & 8.11$\pm$3.54 & 10.17$\pm$5.50 & 6.59$\pm$2.59 & \cellcolor{second}31.77$\pm$3.04 & 27.87$\pm$5.27 & \cellcolor{best}\textbf{33.41$\pm$2.72} \\
\cmidrule(lr){2-10}

&\multirow{3}{*}{5}
 & 0.1-0.2-0.3-0.4-0.5 & 10.42$\pm$3.55 & 6.37$\pm$1.67 & 6.92$\pm$1.85 & 15.81$\pm$9.78 & 22.25$\pm$8.74 & \cellcolor{best}\textbf{40.35$\pm$8.19} & \cellcolor{second}40.30$\pm$4.61 \\
 && 0.2-0.4-0.6-0.8-1.0 & 11.63$\pm$1.87 & 14.07$\pm$5.65 & 9.82$\pm$4.29 & 11.24$\pm$7.86 & 21.42$\pm$9.11 & \cellcolor{second}34.36$\pm$4.82 & \cellcolor{best}\textbf{36.06$\pm$4.59} \\
 && 0.3-0.6-0.9-1.2-1.5 & 13.97$\pm$5.15 & 11.14$\pm$4.01 & 12.01$\pm$2.06 & 15.63$\pm$10.74 & \cellcolor{second}29.95$\pm$5.01 & 18.79$\pm$5.32 & \cellcolor{best}\textbf{35.02$\pm$5.70} \\
 \midrule
\multirow{6}{*}{Ant}
&\multirow{3}{*}{3}
& 0.1-0.5-2.0 & 55.51$\pm$10.37 & 69.67$\pm$8.59 & 69.65$\pm$10.93 & 63.79$\pm$9.55 & \cellcolor{second}70.30$\pm$5.24 & 8.77$\pm$4.00 & \cellcolor{best}\textbf{85.28$\pm$6.52} \\ %
&& 0.2-1.0-4.0 & 31.75$\pm$10.39 & 54.02$\pm$5.97 & \cellcolor{second}59.71$\pm$7.84 & 55.91$\pm$7.68 & 54.55$\pm$6.35 & 16.44$\pm$6.55 & \cellcolor{best}\textbf{61.51$\pm$6.89} \\
&& 0.3-1.5-6.0 & 14.41$\pm$4.68 & \cellcolor{second}67.13$\pm$5.07 & \cellcolor{best}\textbf{69.03$\pm$9.18} & 50.24$\pm$7.41 & 34.47$\pm$4.88 & 20.40$\pm$7.87 & 53.54$\pm$7.07 \\ 
\cmidrule(lr){2-10}

&\multirow{3}{*}{5}
 & 0.1-0.2-0.3-0.4-0.5 & 44.36$\pm$8.50 & \cellcolor{second}75.03$\pm$7.08 & 66.50$\pm$8.95 & 63.40$\pm$10.26 & 58.51$\pm$10.29 & 36.23$\pm$12.33 & \cellcolor{best}\textbf{78.96$\pm$7.34} \\
 && 0.2-0.4-0.6-0.8-1.0 & 30.26$\pm$11.22 & 56.84$\pm$5.10 & \cellcolor{best}\textbf{67.07$\pm$4.40} & 55.65$\pm$13.74 & 47.99$\pm$8.48 & 14.20$\pm$2.07 & \cellcolor{second}61.92$\pm$10.42 \\
 && 0.3-0.6-0.9-1.2-1.5 & 13.14$\pm$5.58 & \cellcolor{best}\textbf{51.81$\pm$6.79} & 46.98$\pm$9.14 & \cellcolor{second}51.36$\pm$8.39 & 15.23$\pm$4.79 & 10.84$\pm$2.38 & 48.51$\pm$9.77 \\
 \midrule
 Total & &  & 436.25 &586.77  &607.87  &542.73  & \cellcolor{second} 729.44 & 521.84 & \cellcolor{best}\textbf{943.09} \\

\bottomrule
\end{tabular}}
\end{table*}

\subsection{Results on Antmaze}
\label{app:more_antmaze_results_play}

In this section, we present additional experiment results on the AntMaze benchmark using the play and diverse source dataset.
We follow the same evaluation procedure as in the diverse dataset setting and report normalized target environment performance across 6  AntMaze medium environments.The results are summarized in \Cref{tab:ant_play_results} and \Cref{tab:ant_diverse_results}.
As shown in \Cref{tab:ant_play_results}, our proposed \algname achieves the best overall performance under the play dataset, with a total score of 398.8, outperforming all baseline methods.
Compared to the second-best method, \algname improves the total score by 5.3\%, demonstrating its robustness under limited state coverage in the source dataset.
Our method achieves the best performance on 3 out of 6 target maps and ranks second on 2 additional maps.
Specifically, \algname performs best on maps $\{2,3,4\}$ and achieves second-best performance on maps $\{5,6\}$.
As shown in \Cref{tab:ant_diverse_results}, \algname achieves the best overall performance in the diverse dataset with a total score of 416.4, outperforming the second-best model by 8.4\%. Our method achieves the best performance in 5 out of 6 maps and ranks second on the remaining one graph. Together, these results demonstrate the strong performance of our proposed \algname across different AntMaze environments, consistently outperforming baselines in both source dataset settings. They highlight the effectiveness of \algname in adapting to challenging navigation environments, while previous methods struggle to identify and utilize source data that are relevant to target navigation, leading to less consistent performance across different maze layouts.

\begin{table*}[ht]
\centering
\setlength{\tabcolsep}{4pt}
\caption{Performance comparison on Antmaze tasks under the play offline dataset with dynamic shifts in the map. Source domains remain unchanged; target domains are shifted. We report normalized target-domain scores (mean ± std over five seeds). Best and second-best scores are highlighted in \textcolor{best}{green} and \textcolor{second}{blue}, respectively. }
\label{tab:ant_play_results}
\resizebox{\textwidth}{!}{
\begin{tabular}{ccccccccc}
\toprule
Env & Shift Level & BOSA & DARA & IQL & IGDF  & OTDF & DmC & \algname\\
\midrule
\multirow{6}{*}{Antmaze Medium}
 & 1 &  $33.6\pm8.0$  & $52.4\pm7.8$ & \cellcolor{best}\textbf{76.4$\pm$10.1} & \cellcolor{second}$58.0\pm12.4$  & $56.8\pm6.87$ & $25.2\pm4.5$ & $52.8\pm9.4$ \\
 & 2 &  $25.2\pm4.6$  & $29.6\pm6.2$ & $46.8\pm4.1$ & $42.4\pm8.2$  & \cellcolor{second} $63.2\pm7.95$ & $24.0\pm2.2$ & \cellcolor{best}\textbf{67.6$\pm$6.7}  \\
 & 3 &  $30.0\pm8.1$  & $56.4\pm11.1$ & $66.4\pm6.7$ & $64.8\pm7.0$  & \cellcolor{second}$69.6\pm5.90$ & $25.2\pm6.5$ & \cellcolor{best}\textbf{71.6$\pm$3.8}  \\
 & 4 &  $34.0\pm8.9$  & $52.8\pm6.3$ & $48.8\pm6.9$ & \cellcolor{second}$58.0\pm1.4$  & $44.8\pm3.90$ & $20.0\pm2.2$ & \cellcolor{best}\textbf{71.6$\pm$6.8} \\
 & 5 &  $22.8\pm13.9$  & $34.4\pm9.3$ & \cellcolor{best}\textbf{73.6$\pm$8.2} & $54.4\pm6.2$  & $55.2\pm3.03$ & $24.4\pm7.1$ & \cellcolor{second}$60.8\pm7.2$  \\
 & 6 &  $12.8\pm4.8$  & $62.8\pm3.0$ & $66.8\pm8.1$ & \cellcolor{best}\textbf{74.8$\pm$4.1}  & $70.8\pm5.93$ & $8.8\pm2.0$ & \cellcolor{second}$74.4\pm3.0$  \\ \midrule
  Total& & 158.4 & 288.4 & \cellcolor{second} 378.8 & 352.4 & 360.4 & 127.6 & \cellcolor{best} \textbf{398.8}\\
\bottomrule
\end{tabular}}
\end{table*}

\begin{table*}[ht]
\centering
\setlength{\tabcolsep}{4pt}
\caption{Performance comparison on AntMaze tasks under the diverse source dataset with dynamic shifts in the map. Source domains remain unchanged; target domains are shifted. We report normalized target-domain scores (mean ± std over five seeds) with the mean of 50 episodes. Best and second-best scores are highlighted in \textcolor{best}{green} and \textcolor{second}{blue}, respectively. }
\label{tab:ant_diverse_results}
\resizebox{\textwidth}{!}{
\begin{tabular}{cccccccccc}
\toprule
Env & Shift Level & BOSA & DARA & IQL & IGDF  & OTDF & DmC & \algname\\
\midrule
\multirow{6}{*}{Antmaze Medium}
 & 1 &  $43.6\pm11.2$  & $53.6\pm9.9$ & \cellcolor{best}\textbf{76.5$\pm$8.4} & $60.4\pm9.1$  & $51.6\pm10.53$ & $33.6\pm6.6$ & \cellcolor{second}$66.8\pm4.6$ \\
 & 2 &  $22.4\pm5.5$  & $33.6\pm3.3$ & $52.8\pm7.9$ & $45.2\pm7.8$  & \cellcolor{second}$56.4\pm4.34$ & $31.2\pm2.7$ & \cellcolor{best}\textbf{69.2$\pm$8.9} \\
 & 3 &  $39.6\pm11.5$  & $50.8\pm3.0$ & $57.6\pm7.7$ & $65.2\pm6.6$  & \cellcolor{second}71.6$\pm$5.90 & $25.6\pm6.4$ & \cellcolor{best}\textbf{73.2$\pm$7.3} \\
 & 4 &  $40.0\pm3.2$  & $48.8\pm8.3$ & $58.4\pm4.8$ & \cellcolor{second}$60.4\pm6.2$  & $52.4\pm8.53$ & $21.2\pm7.0$ & \cellcolor{best}\textbf{65.2$\pm$5.9} \\
 & 5 &  $21.2\pm7.6$  & $41.6\pm12.7$ & \cellcolor{second}66.4$\pm$8.3 & $56.8\pm7.6$  & $57.2\pm7.16$ & $27.6\pm7.9$ & \cellcolor{best}\textbf{66.8$\pm$4.8} \\
 & 6 &  $24.8\pm14.0$  & $59.2\pm12.4$ & $72.4\pm10.7$ & \cellcolor{second}$73.6\pm1.7$  & $65.2\pm9.86$ & $15.2\pm1.6$ & \cellcolor{best}\textbf{75.2$\pm$6.7} \\ \midrule
 Total & & 191.6 & 287.6 & \cellcolor{second}384.1 & 361.6 & 354.4 & 154.4 & \cellcolor{best}\textbf{416.4} \\
\bottomrule
\end{tabular}}
\end{table*}

\subsection{Results on Adroit tasks}
\label{app:experiemnts_androit}
We report the performance of our method and baseline approaches on the Adroit manipulation benchmarks in \Cref{tab: morph_shift_comparison_adroit}. 
As shown in \Cref{tab: morph_shift_comparison_adroit}, our proposed \algname consistently outperforms all baseline methods across both the Pen and Door tasks.
Overall, \algname achieves the highest total score of 402.24, yielding a substantial improvement over the second-best method (OTDF) with a total score of 330.41.
This demonstrates the strong robustness of our approach under severe morphology and kinematic mismatches in real-world application.
These results indicate that \algname generalizes well beyond navigation tasks and achieves strong performance on challenging manipulation benchmarks.
By leveraging localized dynamics-aware filtering, our method is able to retain source transitions that align well with the target dynamics, while prior methods relying on global assumption struggle under morphology and kinematic shifts.

\begin{table*}[ht]
\centering
\small
\caption{
Performance comparison on the Door and Pen environments under
\textbf{morphology shifts} and \textbf{kinematic shifts} across 5 seeds. We use \textbf{M} and \textbf{H} to denote the medium and hard levels of the dynamics shift.
}
\label{tab: morph_shift_comparison_adroit}
\resizebox{\textwidth}{!}{
\begin{tabular}{llcccccccc}
\toprule
Env & Type & Level & BOSA & IQL & DARA & IGDF & OTDF & DmC & \algname \\
\midrule

\multirow{4}{*}{Pen}
& kin-broken-joint & M 
& 42.56$\pm$11.31 
& 45.46$\pm$8.54 
& 46.17$\pm$6.26 
& 50.74$\pm$11.20 
& \cellcolor{second}52.14$\pm$6.51 
& 49.39$\pm$4.52 
& \cellcolor{best}\textbf{64.31$\pm$11.33} \\
& & H 
& 19.36$\pm$5.42 
& \cellcolor{second} 24.02$\pm$5.16 \cellcolor{second} 
& 22.41$\pm$4.73 
& 18.62$\pm$10.47 
& 20.95$\pm$7.23 
& 21.64$\pm$4.47 
& \cellcolor{best}\textbf{28.16$\pm$18.04} \\
& morph-shrink-finger & M 
& 10.33$\pm$2.79 
& 15.64$\pm$2.97 
& 15.33$\pm$2.39 
& 13.19$\pm$4.41 
& \cellcolor{second}18.02$\pm$5.64 
& 13.34$\pm$2.52 
& \cellcolor{best}\textbf{19.88$\pm$8.02} \\
& & H 
& 19.47$\pm$5.01 
& 19.40$\pm$5.05 
& 18.24$\pm$2.59 
& 19.05$\pm$3.87 
& \cellcolor{second}20.34$\pm$2.85 
& 18.78$\pm$3.41 
& \cellcolor{best}\textbf{33.06$\pm$11.48} \\

\midrule
\multirow{4}{*}{Door}
& kin-broken-joint & M 
& 20.65$\pm$10.72 
& \cellcolor{second}40.17$\pm$7.10 
& 36.82$\pm$4.63 
& 37.76$\pm$7.02 
& 39.82$\pm$4.75 
& 37.65$\pm$3.96 
& \cellcolor{best}\textbf{55.12$\pm$12.50} \\
& & H 
& 19.80$\pm$15.61 
& 52.37$\pm$5.52 
& 49.31$\pm$5.53 
& 52.61$\pm$4.83 
& \cellcolor{second}58.24$\pm$6.37 
& 51.07$\pm$4.33 
& \cellcolor{best}\textbf{59.81$\pm$7.23} \\

& morph-shrink-finger & M 
& 41.05$\pm$15.36 
& 54.85$\pm$7.79 
& 54.49$\pm$5.88 
& 54.73$\pm$5.28 
& 59.15$\pm$4.89 
& \cellcolor{second}60.59$\pm$5.23 
& \cellcolor{best}\textbf{71.48$\pm$8.80} \\

& & H 
& 38.92$\pm$18.85 
& \cellcolor{second}61.84$\pm$4.16 
& 58.91$\pm$3.05 
& 61.07$\pm$6.89 
& 61.74$\pm$3.24 
& 49.93$\pm$6.75 
& \cellcolor{best}\textbf{70.42$\pm$4.26} \\
\midrule
\multicolumn{3}{l}{Total} 
& 212.15 
& 313.75 
& 301.67 
& 307.78 
& \cellcolor{second}330.41 
& 302.39 
& \cellcolor{best}\textbf{402.24} \\
\bottomrule
\end{tabular}
}
\end{table*}

\subsection{Results under Medium-Expert setting}
\label{app:medium_expert}

The main experiments follow the standard ODRL setting, where both source and target datasets are collected by medium-level policies. 
We further examine whether \algname remains effective when the target dataset is replaced by an expert-level dataset. 
Specifically, we evaluate a medium-source and expert-target setting, where the source dataset remains medium-level while the target dataset is collected by an expert policy. 
This setting allows us to test whether the proposed filtering mechanism can still identify useful source transitions when the available target data comes from a stronger policy.

We evaluate this setting on HalfCheetah-Friction and Ant-Gravity across four shift levels. 
As shown in \Cref{tab:medium_expert_results}, \algname achieves the highest aggregate score of $328.60$, improving over the strongest baseline IGDF with $267.82$ by $22.7\%$. 
\algname obtains the best performance in 7 out of 8 settings. 
The gains are especially clear on Ant-Gravity, where \algname outperforms all baselines at every shift level. 
On HalfCheetah-Friction, \algname performs best in three out of four settings, while IGDF performs better on HalfCheetah-Friction-0.1. 
Overall, these results show that \algname remains effective when the target dataset is collected by an expert policy, suggesting that our localized filtering mechanism is not limited to the standard medium-level target-data setting.

\begin{table*}[ht]
\centering
\small
\caption{
Performance under the medium-source and expert-target setting.
The source dataset is collected by a medium-level behavior policy, while the target dataset is collected by an expert policy.
We report normalized target-domain scores over five seeds.
}
\label{tab:medium_expert_results}
\resizebox{\textwidth}{!}{
\begin{tabular}{lcccccc}
\toprule
Task & BOSA & IQL & DARA & IGDF & OTDF & \algname \\
\midrule
HalfCheetah-Friction-0.1 & $8.27\pm1.97$ & $22.09\pm13.30$ & $20.98\pm10.98$ & $\mathbf{32.63\pm6.37}$ & $9.29\pm0.46$ & $9.64\pm0.35$ \\
HalfCheetah-Friction-0.5 & $67.24\pm2.93$ & $68.58\pm1.05$ & $68.54\pm0.99$ & $67.16\pm1.50$ & $66.01\pm2.85$ & $\mathbf{68.81\pm0.30}$ \\
HalfCheetah-Friction-2.0 & $44.18\pm3.93$ & $45.35\pm1.46$ & $44.54\pm2.36$ & $44.07\pm1.65$ & $43.83\pm0.55$ & $\mathbf{45.97\pm0.48}$ \\
HalfCheetah-Friction-5.0 & $23.22\pm3.48$ & $17.79\pm5.20$ & $20.88\pm6.49$ & $25.30\pm6.28$ & $22.04\pm6.76$ & $\mathbf{32.55\pm3.08}$ \\
Ant-Gravity-0.1 & $17.26\pm4.07$ & $12.37\pm1.07$ & $12.61\pm0.79$ & $13.96\pm2.58$ & $16.24\pm3.25$ & $\mathbf{26.41\pm3.75}$ \\
Ant-Gravity-0.5 & $12.21\pm1.10$ & $10.18\pm1.04$ & $10.14\pm1.94$ & $9.13\pm2.00$ & $12.54\pm2.01$ & $\mathbf{24.42\pm4.36}$ \\
Ant-Gravity-2.0 & $42.50\pm2.24$ & $37.30\pm1.31$ & $37.65\pm2.64$ & $44.30\pm12.06$ & $37.55\pm4.44$ & $\mathbf{75.86\pm6.85}$ \\
Ant-Gravity-5.0 & $31.74\pm0.75$ & $31.27\pm0.57$ & $30.99\pm0.23$ & $31.27\pm0.71$ & $35.27\pm0.37$ & $\mathbf{44.94\pm3.81}$ \\
\midrule
Total & $246.62$ & $244.93$ & $246.33$ & $267.82$ & $242.77$ & $\mathbf{328.60}$ \\
\bottomrule
\end{tabular}}
\end{table*}

\subsection{Cluster-level sample coverage and classifier stability}
\label{app:cluster_coverage}

\algname estimates the cluster-level KL divergence by training a classifier within each cluster.
A natural concern is whether the limited target dataset provides enough samples for this estimation.
We therefore check the per-cluster sample counts and classifier training behavior in two representative tasks: HalfCheetah-Friction-0.5 and Ant-Gravity-0.5.
We use $K=30$ and report three representative clusters for each task.

As shown in \Cref{tab:cluster_coverage}, the cluster sizes are imbalanced, but the examined clusters still contain samples from both domains.
The target clusters contain 59 to 251 transitions, and the corresponding source clusters contain 1,431 to 82,304 transitions.
Thus, the classifier-based density-ratio estimation is well-defined in these clusters.

To reduce the effect of source-target imbalance within a cluster, we train each cluster classifier with minibatch resampling with replacement.
We set the maximum number of classifier-training epochs to 7000 and use early stopping with a patience of 15.
In practice, early stopping is triggered after 25 to 54 epochs, far before the maximum epoch budget.
The final training losses are also small, suggesting that the cluster classifiers can be optimized stably in these representative settings.

\begin{table}[ht]
\centering
\caption{
Cluster-level sample coverage and classifier training diagnostics.
We report representative clusters from HalfCheetah-Friction-0.5 and Ant-Gravity-0.5 with $K=30$.
}
\label{tab:cluster_coverage}
\resizebox{\textwidth}{!}{
\begin{tabular}{llrrrr}
\toprule
Task & Cluster index & Source samples & Target samples & Stopping epoch & Training loss \\
\midrule
HalfCheetah-Friction-0.5 & 10 & 2,151  & 251 & 42 & 0.018 \\
HalfCheetah-Friction-0.5 & 20 & 29,002 & 62  & 31 & 0.020 \\
HalfCheetah-Friction-0.5 & 29 & 82,304 & 59  & 46 & 0.034 \\
Ant-Gravity-0.5        & 10 & 2,362  & 188 & 54 & 0.066 \\
Ant-Gravity-0.5        & 20 & 1,431  & 88  & 25 & 0.170 \\
Ant-Gravity-0.5        & 29 & 32,329 & 196 & 38 & 0.070 \\
\bottomrule
\end{tabular}}
\end{table}

\subsection{Sensitivity to filtering ratios}
\label{app:filtering_ratio_sensitivity}

\algname uses fixed filtering ratios 
$(\xi_1,\xi_2,\xi_3)=(90\%,80\%,70\%)$ for low-, medium-, and high-KL clusters.
These ratios are fixed across all environments and are not tuned for each task.
We use this simple three-tier rule to keep enough source data while reducing the influence of clusters with larger estimated dynamics mismatch.

We study three more aggressive variants.
The first variant sets $(\xi_1,\xi_2,\xi_3)=(90\%,80\%,0\%)$, which completely removes source samples from high-KL clusters.
The second variant sets $(\xi_1,\xi_2,\xi_3)=(60\%,40\%,20\%)$, which reduces the retention ratios across all clusters.
The third variant sets $(\xi_1,\xi_2,\xi_3)=(30\%,20\%,10\%)$, which keeps only a small fraction of source samples from all clusters.
All other components are kept unchanged.

As shown in \Cref{tab:filtering_ratio_sensitivity}, aggressive filtering is not a reliable default. The hard-exclusion variant can perform slightly better in some Ant-Gravity settings, but it underperforms the default setting on HalfCheetah-Friction. The low-retention variant is worse in most settings, suggesting that keeping too little source data can hurt policy learning. These results suggest that the default soft filtering schedule provides a practical balance between mismatch filtering and state-space coverage, without requiring task-specific tuning of the filtering ratios.

For AntMaze, as shown in \Cref{tab:filtering_ratio_sensitivity_antmaze}, the default setting achieves the highest average score of 70.40 and performs best on three out of four maps. Although more aggressive filtering slightly improves performance on Diverse map 3, it generally underperforms the default setting. These results suggest that retaining sufficient source data is important for preserving state-space coverage in AntMaze.

\begin{table}[ht]
\centering
\small
\setlength{\tabcolsep}{4pt}
\renewcommand{\arraystretch}{1.05}
\caption{
Sensitivity to filtering ratios.
The default \algname uses fixed ratios $(\xi_1,\xi_2,\xi_3)=(90\%,80\%,70\%)$ across all environments. 
}
\label{tab:filtering_ratio_sensitivity}
\begin{tabular}{lcccc}
\toprule
Task & $(90\%,80\%,0\%)$ & $(60\%,40\%,20\%)$ & $(30\%,20\%,10\%)$ & \algname \\
\midrule
HalfCheetah-Friction-0.5 & $67.67\pm0.42$ & $67.30\pm0.45$ & $66.08\pm2.35$ & $\mathbf{68.51\pm0.44}$ \\
HalfCheetah-Friction-2.0 & $47.25\pm0.50$ & $46.74\pm0.92$ & $45.59\pm1.56$ & $\mathbf{48.19\pm0.75}$ \\
Ant-Gravity-0.5 & $\mathbf{21.90\pm1.19}$ & $18.64\pm1.79$ & $19.39\pm6.29$ & $18.61\pm2.79$ \\
Ant-Gravity-2.0 & $\mathbf{46.68\pm2.99}$ & $44.18\pm2.61$ & $45.09\pm3.25$ & $45.82\pm1.33$ \\
\bottomrule
\end{tabular}
\end{table}

\begin{table}[ht]
\centering
\small
\setlength{\tabcolsep}{4pt}
\renewcommand{\arraystretch}{1.05}
\caption{
Sensitivity to filtering ratios on AntMaze.
We report normalized target-domain scores over five seeds.
The default \algname uses fixed ratios $(\xi_1,\xi_2,\xi_3)=(90\%,80\%,70\%)$.
}
\label{tab:filtering_ratio_sensitivity_antmaze}
\begin{tabular}{lcccc}
\toprule
Task & $(90\%,80\%,0\%)$ & $(60\%,40\%,20\%)$ & $(30\%,20\%,10\%)$ & \algname \\
\midrule
Diverse map 2 & $62.80\pm8.63$ & $64.80\pm3.92$ & $65.20\pm4.12$ & $\mathbf{69.20\pm8.90}$ \\
Diverse map 3 & $69.60\pm8.80$ & $\mathbf{73.60\pm4.08}$ & $\mathbf{73.60\pm8.52}$ & $73.20\pm7.30$ \\
Play map 2 & $65.60\pm5.12$ & $57.20\pm6.88$ & $63.60\pm5.57$ & $\mathbf{67.60\pm6.70}$ \\
Play map 3 & $69.20\pm7.22$ & $68.40\pm7.09$ & $66.40\pm5.43$ & $\mathbf{71.60\pm3.80}$ \\
\midrule
Average & $66.80$ & $66.00$ & $67.20$ & $\mathbf{70.40}$ \\
\bottomrule
\end{tabular}
\end{table}

\subsection{Component ablations and L2-distance filtering}
\label{app:component_ablation}

We ablate three components of \algname. 
First, we set $\alpha=0$ to remove KL-based critic weighting. 
Second, we set $\lambda=0$ to remove the CVAE-based policy regularizer. 
Third, we replace the local discriminator-based filtering score with an L2-distance filtering score. 
All other components are unchanged.

As shown in \Cref{tab:component_ablation}, the full \algname achieves the best aggregate score of $316.42$. 
Removing critic weighting reduces the total score to $297.36$, and removing the CVAE regularizer further reduces it to $288.65$. 
The L2 filtering variant obtains $298.71$, which is competitive in some settings but worse than \algname overall. 
These results show that KL-based critic weighting, target-support regularization, and local discriminator-based filtering all contribute to the final performance.

\begin{table}[ht]
\centering
\small
\setlength{\tabcolsep}{5pt}
\renewcommand{\arraystretch}{1.05}
\caption{
Component ablations of \algname.
We compare the full method with variants that remove KL-based critic weighting ($\alpha=0$), remove CVAE policy regularization ($\lambda=0$), or replace local discriminator-based filtering with L2-distance filtering.
}
\label{tab:component_ablation}
\begin{tabular}{lcccc}
\toprule
Task & $\alpha=0$ & $\lambda=0$ & L2 filter & \algname \\
\midrule
HalfCheetah-Friction-0.1 & $9.40\pm0.82$ & $9.91\pm0.29$ & $\mathbf{10.02\pm0.78}$ & $9.97\pm0.47$ \\
HalfCheetah-Friction-0.5 & $66.17\pm0.58$ & $\mathbf{68.86\pm0.48}$ & $64.57\pm0.69$ & $68.51\pm0.44$ \\
HalfCheetah-Friction-2.0 & $46.12\pm0.40$ & $48.03\pm0.82$ & $45.68\pm0.50$ & $\mathbf{48.19\pm0.75}$ \\
HalfCheetah-Friction-5.0 & $51.25\pm5.11$ & $54.36\pm2.50$ & $53.07\pm6.14$ & $\mathbf{54.65\pm1.56}$ \\
Ant-Gravity-0.1 & $10.15\pm1.24$ & $10.06\pm0.23$ & $10.20\pm0.16$ & $\mathbf{22.38\pm1.28}$ \\
Ant-Gravity-0.5 & $\mathbf{19.35\pm2.64}$ & $12.58\pm1.83$ & $17.22\pm4.50$ & $18.61\pm2.79$ \\
Ant-Gravity-2.0 & $45.31\pm4.19$ & $45.78\pm1.94$ & $45.45\pm2.27$ & $\mathbf{45.82\pm1.33}$ \\
Ant-Gravity-5.0 & $49.61\pm7.01$ & $39.07\pm4.48$ & $\mathbf{52.50\pm6.59}$ & $48.29\pm3.68$ \\
\midrule
Total & $297.36$ & $288.65$ & $298.71$ & $\mathbf{316.42}$ \\
\bottomrule
\end{tabular}
\end{table}

\subsection{Comparison with CVAE-regularized baselines}
\label{app:cvae_baselines}

We further test whether the gain of \algname mainly comes from the CVAE-based policy regularizer.
To study this factor, we add the same CVAE regularizer to three baselines: IQL, DARA, and IGDF.
The original filtering or adaptation mechanism of each baseline is kept unchanged.

As shown in \Cref{tab:cvae_baselines}, adding CVAE regularization can improve some baselines in some settings, but the effect is not stable.
For example, IQL+CVAE improves over IQL on HalfCheetah-Friction-0.1, but performs worse on HalfCheetah-Friction-0.5.
Similarly, IGDF+CVAE improves over IGDF on Ant-Gravity-5.0, but performs worse on Ant-Gravity-2.0.
In contrast, \algname achieves the best aggregate score of $316.42$, outperforming the strongest CVAE-regularized baseline IQL+CVAE, which obtains $291.14$.
On Ant-Gravity, \algname also outperforms all CVAE-regularized baselines across all shift levels.
These results show that CVAE regularization alone does not explain the performance gain.
The improvement comes from combining target-support regularization with localized dynamics-aware filtering.

\begin{table*}[ht]
\centering
\caption{
Comparison with CVAE-regularized baselines.
We add the same CVAE-based policy regularizer used in \algname to IQL, DARA, and IGDF.
}
\label{tab:cvae_baselines}
\resizebox{\textwidth}{!}{
\begin{tabular}{lccccccc}
\toprule
Task & IQL & IQL+CVAE & DARA & DARA+CVAE & IGDF & IGDF+CVAE & \algname \\
\midrule
HalfCheetah-Friction-0.1 & $23.69\pm16.46$ & $\mathbf{43.63\pm7.56}$ & $26.39\pm11.35$ & $40.18\pm18.79$ & $20.56\pm11.22$ & $30.66\pm7.75$ & $9.97\pm0.47$ \\
HalfCheetah-Friction-0.5 & $68.26\pm1.42$ & $48.12\pm10.15$ & $68.41\pm3.19$ & $46.04\pm12.35$ & $66.30\pm1.79$ & $55.69\pm1.44$ & $\mathbf{68.51\pm0.44}$ \\
HalfCheetah-Friction-2.0 & $42.68\pm2.64$ & $\mathbf{48.34\pm0.97}$ & $46.17\pm1.72$ & $45.79\pm3.08$ & $46.22\pm1.27$ & $46.65\pm0.89$ & $48.19\pm0.75$ \\
HalfCheetah-Friction-5.0 & $48.97\pm4.68$ & $52.97\pm1.56$ & $45.31\pm9.83$ & $\mathbf{55.74\pm0.69}$ & $54.75\pm5.60$ & $55.70\pm8.17$ & $54.65\pm1.56$ \\
Ant-Gravity-0.1 & $11.03\pm1.24$ & $14.27\pm1.53$ & $12.53\pm1.11$ & $12.89\pm0.85$ & $13.55\pm2.13$ & $13.37\pm0.65$ & $\mathbf{22.38\pm1.28}$ \\
Ant-Gravity-0.5 & $10.53\pm1.06$ & $12.57\pm2.04$ & $9.08\pm0.88$ & $12.24\pm0.97$ & $11.34\pm1.89$ & $11.32\pm0.87$ & $\mathbf{18.61\pm2.79}$ \\
Ant-Gravity-2.0 & $38.26\pm3.09$ & $36.81\pm2.17$ & $38.54\pm2.19$ & $36.57\pm2.25$ & $41.59\pm4.63$ & $37.23\pm0.74$ & $\mathbf{45.82\pm1.33}$ \\
Ant-Gravity-5.0 & $31.51\pm1.97$ & $34.43\pm1.00$ & $32.51\pm1.71$ & $35.15\pm2.17$ & $34.44\pm2.77$ & $37.60\pm4.10$ & $\mathbf{48.29\pm3.68}$ \\
\midrule
Total & $274.93$ & $291.14$ & $278.94$ & $284.60$ & $288.75$ & $288.22$ & $\mathbf{316.42}$ \\
\bottomrule
\end{tabular}}
\end{table*}

\subsection{Sensitivity analysis on AntMaze and Adroit}
\label{app:additional_hyperparameter_sensitivity}

We further study the sensitivity of \algname to the number of clusters $K$ and the policy regularization weight $\lambda$ on AntMaze navigation and Adroit manipulation tasks. 
These environments differ substantially from MuJoCo locomotion: AntMaze has structured map geometry and sparse navigation objectives, while Adroit involves high-dimensional manipulation with morphology and kinematic shifts. 
Therefore, the best hyperparameters need not be identical across all task families.

For AntMaze, we evaluate smaller cluster numbers $K\in\{5,10,15\}$ and larger regularization weights $\lambda\in\{0.5,1.0,5.0\}$. 
As shown in \Cref{tab:antmaze_hyperparameter_sensitivity}, performance remains stable across the tested values of $K$, with all three choices achieving comparable average performance. 
The regularization weight has a larger effect: $\lambda=1.0$ provides the most reliable performance across the representative maps, while $\lambda=0.5$ is often too weak for the structured navigation setting.

For Adroit, we evaluate $K\in\{30,50\}$ and $\lambda\in\{0.1,1.0\}$. 
As shown in \Cref{tab:adroit_hyperparameter_sensitivity}, both cluster numbers lead to comparable performance, and $\lambda=1.0$ is slightly more stable on average. 
Overall, these results suggest that \algname only requires tuning within a compact task-dependent candidate set rather than a large continuous hyperparameter search.

\begin{table*}[ht]
\centering
\small
\caption{
Hyperparameter sensitivity of \algname on AntMaze. 
We report normalized target-domain scores under different choices of the number of clusters $K$ and the policy regularization weight $\lambda$.
}
\label{tab:antmaze_hyperparameter_sensitivity}
\resizebox{\textwidth}{!}{
\begin{tabular}{lcccccc}
\toprule
Task & $K=5$ & $K=10$ & $K=15$ & $\lambda=0.5$ & $\lambda=1.0$ & $\lambda=5.0$ \\
\midrule
Diverse map 2 & $64.80\pm3.03$ & $69.20\pm8.90$ & $63.20\pm2.28$ & $55.20\pm8.32$ & $69.20\pm8.90$ & $62.80\pm8.07$ \\
Diverse map 3 & $73.20\pm7.30$ & $65.20\pm7.56$ & $69.60\pm8.05$ & $54.80\pm5.40$ & $73.20\pm7.30$ & $73.20\pm7.56$ \\
Play map 2 & $69.60\pm6.23$ & $67.60\pm6.70$ & $64.40\pm5.18$ & $54.00\pm5.10$ & $67.60\pm6.70$ & $58.40\pm13.37$ \\
Play map 3 & $68.80\pm8.07$ & $71.60\pm3.80$ & $72.80\pm7.16$ & $63.20\pm5.40$ & $71.60\pm3.80$ & $65.20\pm10.06$ \\
\midrule
Average & $69.10$ & $68.40$ & $67.50$ & $56.80$ & $70.40$ & $64.90$ \\
\bottomrule
\end{tabular}}
\end{table*}

\begin{table}[ht]
\centering
\small
\caption{
Hyperparameter sensitivity of \algname on Adroit. 
We report normalized target-domain scores under different choices of $K$ and $\lambda$.
}
\label{tab:adroit_hyperparameter_sensitivity}
\begin{tabular}{lcccc}
\toprule
Task & $K=30$ & $K=50$ & $\lambda=0.1$ & $\lambda=1.0$ \\
\midrule
Pen morph-shrink (M) & $19.88\pm8.02$ & $18.77\pm4.22$ & $16.75\pm9.79$ & $19.88\pm8.02$ \\
Pen morph-shrink (H) & $23.89\pm13.84$ & $33.06\pm11.48$ & $33.06\pm11.48$ & $22.76\pm17.33$ \\
Door kin-broken (M) & $45.08\pm12.50$ & $55.12\pm12.50$ & $42.73\pm5.53$ & $55.12\pm12.50$ \\
Door kin-broken (H) & $61.98\pm10.07$ & $49.77\pm16.18$ & $55.67\pm8.30$ & $61.98\pm10.07$ \\
\midrule
Average & $37.71$ & $39.18$ & $37.05$ & $39.94$ \\
\bottomrule
\end{tabular}
\end{table}

\subsection{Time comparison of data filtering methods}\label{sec:experiments_runtime_comparison} 

To understand the computational cost of \algname compared with other filtering methods, we compare its wall-clock runtime with IGDF \citep{wen2024contrastive} and OTDF \citep{lyu2025cross}. 
We use the Ant-Gravity-0.5 environment and vary the number of source samples while keeping the target dataset fixed. 
We decompose the runtime into two stages. 
P1 denotes the pre-optimization stage, which includes the method-specific filtering or scoring procedure before policy optimization. 
For \algname, P1 includes K-means clustering, source-sample assignment, local discriminator training, KL estimation, and source-data filtering. 
Thus, the reported P1 runtime of \algname explicitly includes the cost of training the cluster-level domain classifiers. 
P2 denotes the downstream policy optimization stage after filtering or scoring.

For fairness, we modify the replay buffer so that it samples without replacement, meaning that it goes over the entire dataset in each epoch. 
We fix the batch size to 32 and the number of epochs to 7000. 
As shown in \Cref{tab:runtime}, IGDF has the smallest P1 runtime, while OTDF has the largest P1 runtime because it computes optimal transport costs between source and target samples. 
\algname is slower than IGDF in P1 due to local discriminator training, but is much faster than OTDF. 
In P2, the runtime is mainly determined by the downstream policy optimization pipeline and the additional regularization modules. 
Overall, IGDF is the fastest method but has weaker policy performance, while \algname provides a better balance between runtime and final performance.

\begin{table}[ht]
\centering
\small
\setlength{\tabcolsep}{4pt}
\renewcommand{\arraystretch}{1.15}
\caption{
Runtime comparison in Ant-Gravity-0.5 with different numbers of source samples.
P1 denotes the pre-optimization filtering or scoring stage, and P2 denotes the downstream policy optimization stage.
For \algname, P1 includes clustering, source-sample assignment, local discriminator training, KL estimation, and filtering.
}
\label{tab:runtime}
\begin{tabular}{c|cc|cc|cc}
\toprule
\# Sources & IGDF-P1 & IGDF-P2 & OTDF-P1 & OTDF-P2 & \algname-P1 & \algname-P2 \\
\midrule
1000000 & $41.49$s & $4570.07$s & $4597.82$s & $5988.57$s & $134.62$s & $6072.26$s \\
100000  & $38.71$s & $4878.31$s & $628.16$s  & $7671.25$s & $80.18$s  & $3955.51$s \\
10000   & $55.64$s & $4706.42$s & $126.39$s  & $6384.07$s & $72.92$s  & $3806.96$s \\
\bottomrule
\end{tabular}
\end{table}

\subsection{Robustness to Target Data}
\label{sec:target_data_robustness}

When target data are extremely sparse or collected by a suboptimal behavior policy, estimating localized dynamics discrepancies and learning effective policies becomes more challenging. To evaluate the performance of \algname under these conditions, we conduct additional experiments that progressively reduce the amount of target data and replace the original medium-quality target datasets with lower-quality alternatives.

We evaluate \algname on locomotion tasks with Ant-Gravity and HalfCheetah-Friction environments. For the limited target data setting, we retain $70\%$, $50\%$, $20\%$, and $10\%$ of the original target dataset. For the low-quality target data setting, we replace the medium target dataset with random-medium or random target data, using 5,000 transitions for each alternative dataset. The number of clusters $K$ and regularization weight $\lambda$ are tuned for these experiments following the hyperparameters in \Cref{hyperparameter}. We also adjust the diameter threshold $\delta$ to avoid removing all source data during the initial filtering stage when target data are limited.

\textbf{Limited target data.}
\Cref{tab:sparse_target_data} reports the results with progressively smaller target datasets. In most cases, \algname remains relatively stable when up to $50\%$ of the target data are removed, with more noticeable performance degradation under extreme sparsity ($10\%$ and $20\%$ of the original target data remain). These results suggest that localized discrepancy estimation remains effective under moderately limited target coverage. An exception is Ant-Gravity-0.5, where performance does not show a consistent decreasing trend. We attribute this behavior to variability in experimental settings rather than to a systematic benefit from reducing the number of target samples.

\begin{table}[ht]
    \centering
    \small
    \caption{Performance of \algname with different amounts of target data. Results are reported as mean $\pm$ standard deviation over five seeds.}
    \label{tab:sparse_target_data}
    \resizebox{\textwidth}{!}{
    \begin{tabular}{lccccc}
        \toprule
        Task & All target data & $70\%$ & $50\%$ & $20\%$ & $10\%$ \\
        \midrule
        HalfCheetah-Friction-0.1 & $\mathbf{9.97 \pm 0.47}$ & $9.64 \pm 0.85$ & $9.65 \pm 0.30$ & $9.60 \pm 0.20$ & $8.79 \pm 0.50$ \\
        HalfCheetah-Friction-0.5 & $\mathbf{68.51 \pm 0.44}$ & $65.50 \pm 0.62$ & $64.51 \pm 1.18$ & $60.16 \pm 1.75$ & $55.80 \pm 0.72$ \\
        HalfCheetah-Friction-2.0 & $\mathbf{48.19 \pm 0.75}$ & $47.05 \pm 0.29$ & $46.65 \pm 0.41$ & $44.52 \pm 0.63$ & $41.93 \pm 2.69$ \\
        HalfCheetah-Friction-5.0 & $\mathbf{54.65 \pm 1.56}$ & $47.43 \pm 5.28$ & $49.23 \pm 4.36$ & $42.94 \pm 4.51$ & $30.15 \pm 13.81$ \\
        Ant-Gravity-0.1 & $22.38 \pm 1.28$ & $13.77 \pm 10.39$ & $\mathbf{22.75 \pm 3.90}$ & $10.69 \pm 1.27$ & $8.85 \pm 0.17$ \\
        Ant-Gravity-0.5 & $18.61 \pm 2.79$ & $21.96 \pm 4.40$ & $21.83 \pm 2.42$ & $\mathbf{23.57 \pm 5.47}$ & $21.53 \pm 4.92$ \\
        Ant-Gravity-2.0 & $\mathbf{45.82 \pm 1.33}$ & $44.89 \pm 1.04$ & $42.75 \pm 2.50$ & $37.60 \pm 3.46$ & $33.60 \pm 0.84$ \\
        Ant-Gravity-5.0 & $\mathbf{48.29 \pm 3.68}$ & $42.60 \pm 6.72$ & $36.02 \pm 2.40$ & $34.82 \pm 0.41$ & $37.00 \pm 6.50$ \\
        \bottomrule
    \end{tabular}}
\end{table}

\textbf{Low-quality target data.}
We further evaluate \algname using target datasets of different quality levels. As shown in \Cref{tab:target_data_quality}, performance generally decreases as target data quality decreases, with larger drops when using random data. This is because low-quality target data provide limited state-action coverage and weaker supervision for policy learning. This limitation also applies to other off-dynamics methods that rely on target data for adaptation. Overall, these results show that \algname is sensitive to target data quality and benefits from higher-quality target data.

\begin{table}[ht]
    \centering
    \small
    \caption{Performance of \algname with target datasets of different quality levels. Results are reported as mean $\pm$ standard deviation over five seeds.}
    \label{tab:target_data_quality}
    \begin{tabular}{lccc}
        \toprule
        Task & Medium target & Random-medium target & Random target \\
        \midrule
        HalfCheetah-Friction-0.1 & $\mathbf{9.97 \pm 0.47}$ & $8.5 \pm 0.7$ & $4.0 \pm 0.8$ \\
        HalfCheetah-Friction-0.5 & $\mathbf{68.51 \pm 0.44}$ & $66.7 \pm 1.2$ & $62.0 \pm 2.3$ \\
        HalfCheetah-Friction-2.0 & $\mathbf{48.19 \pm 0.75}$ & $44.7 \pm 1.0$ & $40.2 \pm 2.0$ \\
        HalfCheetah-Friction-5.0 & $\mathbf{54.65 \pm 1.56}$ & $21.8 \pm 3.0$ & $11.0 \pm 2.0$ \\
        Ant-Gravity-0.1 & $\mathbf{22.38 \pm 1.28}$ & $10.4 \pm 0.1$ & $9.9 \pm 0.1$ \\
        Ant-Gravity-0.5 & $\mathbf{18.61 \pm 2.79}$ & $8.9 \pm 0.5$ & $7.0 \pm 0.4$ \\
        Ant-Gravity-2.0 & $\mathbf{45.82 \pm 1.33}$ & $34.8 \pm 1.1$ & $23.6 \pm 4.2$ \\
        Ant-Gravity-5.0 & $\mathbf{48.29 \pm 3.68}$ & $30.5 \pm 0.6$ & $2.1 \pm 0.7$ \\
        \bottomrule
    \end{tabular}
\end{table}

\clearpage
\newpage
\section*{NeurIPS Paper Checklist}

\begin{enumerate}

\item {\bf Claims}
    \item[] Question: Do the main claims made in the abstract and introduction accurately reflect the paper's contributions and scope?
    \item[] Answer: \answerYes{}.
    \item[] Justification: The claims are supported by the theory and experiments in the main text, with additional experimental results in the appendix.

\item {\bf Limitations}
    \item[] Question: Does the paper discuss the limitations of the work performed by the authors?
    \item[] Answer: \answerYes{}.
    \item[] Justification: We discuss limitations in the main text, including dependence on target data coverage, quality and source-target overlap.


\item {\bf Theory assumptions and proofs}
    \item[] Question: For each theoretical result, does the paper provide the full set of assumptions and a complete (and correct) proof?
    \item[] Answer: \answerYes{}.
    \item[] Justification: The assumptions and theoretical results are stated in the main text, and the complete proofs and supporting lemmas are provided in the appendix.


    \item {\bf Experimental result reproducibility}
    \item[] Question: Does the paper fully disclose all the information needed to reproduce the main experimental results of the paper to the extent that it affects the main claims and/or conclusions of the paper (regardless of whether the code and data are provided or not)?
    \item[] Answer: \answerYes{}.
    \item[] Justification: The main text describes the benchmarks, evaluation protocol, and main experimental setup, while the appendix provides datasets, hyperparameters, baselines, and implementation details.

\item {\bf Open access to data and code}
    \item[] Question: Does the paper provide open access to the data and code, with sufficient instructions to faithfully reproduce the main experimental results, as described in supplemental material?
    \item[] Answer: \answerYes{}.
    \item[] Justification: We use public benchmark datasets and provide a code package with instructions to reproduce the main experiments. Additional implementation details are provided in the appendix.

\item {\bf Experimental setting/details}
    \item[] Question: Does the paper specify all the training and test details (e.g., data splits, hyperparameters, how they were chosen, type of optimizer) necessary to understand the results?
    \item[] Answer: \answerYes{}.
    \item[] Justification: The main text describes the training and evaluation setup, and the appendix provides dataset details, hyperparameters, and baseline implementation details.

\item {\bf Experiment statistical significance}
    \item[] Question: Does the paper report error bars suitably and correctly defined or other appropriate information about the statistical significance of the experiments?
    \item[] Answer: \answerYes{}. %
    \item[] Justification: We report mean and standard deviation over five seeds for the main experimental results and most additional experiments.

\item {\bf Experiments compute resources}
    \item[] Question: For each experiment, does the paper provide sufficient information on the computer resources (type of compute workers, memory, time of execution) needed to reproduce the experiments?
    \item[] Answer: \answerYes{}. %
    \item[] Justification: We report the GPU, CPU, memory, and storage requirements in Appendix~D.7, and provide runtime comparisons in Appendix~E.10.
    
\item {\bf Code of ethics}
    \item[] Question: Does the research conducted in the paper conform, in every respect, with the NeurIPS Code of Ethics \url{https://neurips.cc/public/EthicsGuidelines}?
    \item[] Answer: \answerYes{}. %
    \item[] Justification: The work uses public offline RL benchmarks and simulation environments, and does not involve human subjects, private data, or high-risk deployment.

\item {\bf Broader impacts}
    \item[] Question: Does the paper discuss both potential positive societal impacts and negative societal impacts of the work performed?
    \item[] Answer: \answerNA{}.
    \item[] Justification: This is an algorithmic study evaluated on simulated offline RL benchmarks and does not introduce a deployed system or high-risk application.
    
\item {\bf Safeguards}
    \item[] Question: Does the paper describe safeguards that have been put in place for responsible release of data or models that have a high risk for misuse (e.g., pre-trained language models, image generators, or scraped datasets)?
    \item[] Answer: \answerNA{}.
    \item[] Justification: The paper does not release high-risk models, scraped datasets or generative systems that require special safeguards.

\item {\bf Licenses for existing assets}
    \item[] Question: Are the creators or original owners of assets (e.g., code, data, models), used in the paper, properly credited and are the license and terms of use explicitly mentioned and properly respected?
    \item[] Answer: \answerYes{}.
    \item[] Justification: The existing datasets, benchmarks, and baseline methods used in the paper are cited in the main text and appendix.

\item {\bf New assets}
    \item[] Question: Are new assets introduced in the paper well documented and is the documentation provided alongside the assets?
    \item[] Answer: \answerNA{}. %
    \item[] Justification: The paper does not introduce or release new datasets, models, or other new research assets.

\item {\bf Crowdsourcing and research with human subjects}
    \item[] Question: For crowdsourcing experiments and research with human subjects, does the paper include the full text of instructions given to participants and screenshots, if applicable, as well as details about compensation (if any)? 
    \item[] Answer: \answerNA{}. %
    \item[] Justification: The paper does not involve crowdsourcing or research with human subjects.

\item {\bf Institutional review board (IRB) approvals or equivalent for research with human subjects}
    \item[] Question: Does the paper describe potential risks incurred by study participants, whether such risks were disclosed to the subjects, and whether Institutional Review Board (IRB) approvals (or an equivalent approval/review based on the requirements of your country or institution) were obtained?
    \item[] Answer: \answerNA{}. %
    \item[] Justification: The paper does not involve human subjects, so IRB approval is not applicable.

\item {\bf Declaration of LLM usage}
    \item[] Question: Does the paper describe the usage of LLMs if it is an important, original, or non-standard component of the core methods in this research? Note that if the LLM is used only for writing, editing, or formatting purposes and does \emph{not} impact the core methodology, scientific rigor, or originality of the research, declaration is not required.
    \item[] Answer: \answerNA{}. %
    \item[] Justification: The core method, experiments, and theoretical analysis do not use LLMs as a methodological component.

\end{enumerate}

\end{document}